\documentclass[10pt]{article} 
\usepackage[preprint]{tmlr}

\usepackage{amsmath,amsfonts,bm}

\def\eqref#1{equation~\ref{#1}}

\def\1{\bm{1}}

\DeclareMathAlphabet{\mathsfit}{\encodingdefault}{\sfdefault}{m}{sl}
\SetMathAlphabet{\mathsfit}{bold}{\encodingdefault}{\sfdefault}{bx}{n}

\usepackage{hyperref}
\usepackage{url}
\usepackage[nolist]{acronym}

\usepackage{diagbox}
\usepackage{multicol}

\usepackage{microtype}
\usepackage{graphicx}
\usepackage{subcaption}
\usepackage{booktabs} 
\usepackage{xcolor}
\usepackage[nolist]{acronym}
\usepackage{array, multirow}
\usepackage{svg}
\usepackage{lipsum}
\usepackage{amsmath}
\usepackage{pifont}
\usepackage{centernot}

\usepackage{amsthm}

\usepackage{algorithm}
\usepackage{algpseudocode}

\theoremstyle{definition}
\newtheorem{definition}{Definition}[section]
\newtheorem{proposition}{Proposition}[section]
\newtheorem{corollary}{Corollary}[section]

\newcommand\independent{\mathpalette\independenT{\perp}}
\def\independenT#1#2{\mathrel{\rlap{$#1#2$}\mkern2mu{#1#2}}}

\newcommand{\notindependent}{\mathrel{\centernot{\independent}}}

\title{Concept Drift from a Causal Perspective}

\author{\name Eduardo V. L. Barboza \email eduardo.lima-barboza.1@ens.etsmtl.ca \\
      \addr LIVIA, École de Technologie Supérieure
      \AND
      \name Jean Paul Barddal \email paul.jean@pucpr.br \\
      \addr Graduate Program in Informatics (PPGIa), Pontifícia Universidade Católica do Paraná (PUCPR)
      \AND
      \name Robert Sabourin \email robert.sabourin@etsmtl.ca\\
      \addr LIVIA, École de Technologie Supérieure 
      \AND
      \name Rafael M. O. Cruz \email rafael.menelau-cruz@etsmtl.ca\\
      \addr LIVIA, École de Technologie Supérieure 
      }

\def\month{MM}  
\def\year{YYYY} 
\def\openreview{\url{https://openreview.net/forum?id=XXXX}} 

\begin{document}

\maketitle

\begin{acronym}[TDMA]
    \acro{ML}{Machine Learning}
    \acro{IRT}{Item response theory}
    \acro{iid}{independent and identically distributed}
    \acro{DAG}{Directed Acyclic Graph}
    \acro{DAGs}{Directed Acyclic Graphs}
    \acro{IoT}{Internet of Things}
    \acro{SCM}{Structural Causal Model}
    \acro{SCMs}{Structural Causal Models}
    \acro{LTMs}{Large Tabular Models}
    \acro{LLMs}{Large Language Models}
    \acro{GReaT}{Generation of Realistic Tabular data}
    \acro{CaDrift}{\emph{Causal Drift Generator}}
    \acro{EWMA}{exponentially weighted moving average}
    \acro{OWDSG}{Open World Data Stream Generator with Concept Non-stationarity}
    \acro{ACF}{Autocorrelation Function}
    \acro{MMD}{Maximum Mean Discrepancy}
    \acro{HT}{Hoeffding Tree}
    \acro{ARF}{Adaptive Random Forest}
    \acro{DDM}{Drift Detection Method}
    \acro{ADWIN}{Adaptive Windowing}
    \acro{MDDM}{McDiarmid Drift Detection Method}
    \acro{UDD}{Uncertainty Drift Detector}
    \acro{PU-index}{Prediction Uncertainty Index}
    \acro{CD-NOD}{Causal discovery from nonstationary/heterogeneous data}
    \acro{LIN}{Latent Intervened Non-stationary learning}
    \acro{ICP}{Invariant Causal Prediction}
    \acro{IRM}{Invariant Risk Minimization}
    \acro{REx}{Risk Extrapolation}
    \acro{ICM}{Independent Causal Mechanism}
    \acro{SMS}{Sparse Mechanism Shift}
\end{acronym}

\begin{abstract}
Concept drift is a common phenomenon in real-world data streams, in which changes in the data-generating distribution can degrade predictive model performance. 
Most existing definitions characterize drift as changes in the joint distribution $P(\mathbf{x}, y)$, without distinguishing which component of the data-generating process has changed. 
In this work, we introduce a causal perspective on concept drift based on Structural Causal Models (SCMs). 
We propose a taxonomy that categorizes drift events by their causal origin, including changes in exogenous variables, endogenous mechanisms, confounders, and target-generating processes. 
Building on this framework, we develop an SCM-based data stream generator that simulates controlled mechanism-level drift events. 
Our experiments empirically characterize the distributional effects of each drift type and show that drifts with different causal origins induce distinct patterns of distribution shift and predictive behavior.
Furthermore, by integrating causal discovery methods, we use our framework to construct data streams grounded in real-world dependency structures, enabling more realistic and informative evaluation scenarios. 
We also demonstrate that leveraging the generated data can improve downstream performance.
These results highlight the importance of accounting for causal structure when studying and evaluating adaptive learning methods, and establish a foundation for causally-aware evaluation in non-stationary environments.
\end{abstract}

\section{Introduction}
\label{sec:intro}

Concept drift is a pervasive challenge in real-world data streams, where the statistical properties of data evolve over time \citep{lu2019, hinder2024a}. 
Such changes can significantly degrade the performance of predictive models deployed in dynamic environments, including census analysis \citep{Chakrabarty2018}, fraud detection \citep{Hernandez2024}, and social media analysis \citep{yogi2024}. 
As a result, a large body of research has focused on detecting, characterizing, and adapting to concept drift in streaming settings \citep{barboza2025, paim2025, Kurian2024}.

Most existing definitions characterize concept drift as any change in the joint distribution $P(\mathbf{x},y)$ \citep{gama2014survey, lu2019, hinder2024a}. 
While this probabilistic view captures a broad family of distribution shifts, it does not reveal which components of the data-generating process have changed.
Consequently, drift events stemming from fundamentally different causes may appear indistinguishable at the level of $P(\mathbf{x},y)$, even though their implications for learning and adaptation can be substantially different.

This limitation is closely related to challenges studied in out-of-distribution (OOD) generalization \citep{arjovsky2020invariantriskminimization}.
A growing body of work argues that robust generalization to distribution shifts requires identifying invariant causal mechanisms rather than relying on spurious correlations that may change across environments \citep{arjovsky2020invariantriskminimization, scholkopf2021towardCRL, eastwood2022QRM}. 
A classic example is that of an image classifier trained to recognize cows, but that inadvertently relies on the presence of green pastures \citep{arjovsky2020invariantriskminimization}. 
Such a model may fail on images of cows standing on beaches, not because the concept has changed, but because the causal structure that generated the data differs across environments.

We can build this intuition about mechanism changes using the graph in Figure \ref{fig:graph-treatment}. 
Consider that $T$ is a treatment variable corresponding to a drug dosage, and the outcome $Y$ is the effect of the dosage on a patient. 
The confounder $C$ (e.g., patient demographics) creates spurious correlations by influencing both the dosage $T$ and the outcome $Y$. 
With continued use, a patient may develop physiological resistance to the drug \citep{English2010}, which characterizes a change in the cause--effect relationship between $T$ and $Y$.

\begin{figure}
    \centering
    \includegraphics[width=0.2\linewidth]{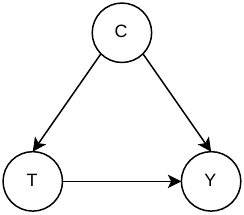}
    \caption{A graph exemplifying a treatment-outcome relationship with a confounder variable.}
    \label{fig:graph-treatment}
\end{figure}

These observations suggest that concept drift is fundamentally a \emph{causal} phenomenon: it arises from changes in the structural mechanisms that generate the data. 
Within the framework of \acp{SCM} \citep{peters2017elements}, data are generated by a set of structural equations defining how each variable depends on its causes.
From this perspective, drift corresponds to changes in one or more components of this system, such as the distribution of exogenous variables, the functional mechanisms relating variables, or the structure of the causal graph itself.
Different types of mechanism changes may induce similar distributional shifts in $P(\mathbf{x},y)$, yet they may require fundamentally different adaptation strategies.
Existing approaches categorize drift through statistical discrepancies or model-dependent signals without distinguishing between changes in the underlying data-generating mechanism.
As a result, they provide limited insight into what type of drift occurred and how learning algorithms should respond.

In this work, we propose a causal taxonomy of concept drift grounded in \acp{SCM}. 
We categorize drift events according to their causal origin, distinguishing between exogenous drift, confounder drift, endogenous drift, target drift, and structural drift. 
Each category corresponds to a specific type of change in the underlying data-generating mechanisms.
Such categorization sets a foundation for research in concept drift through the causal perspective.
To the best of our knowledge, this is the first work to explicitly define different sources of concept drift from the causal lens.

Building on this taxonomy, we present an \ac{SCM}-based data stream generator that enables controlled simulation of mechanism-level drift events, which we call \ac{CaDrift}.
Unlike existing synthetic generators, which typically simulate drift by arbitrarily modifying decision boundaries or feature centroids \citep{street2001, bifet2009a, komorniczak2025}, our approach models drift as changes in structural mechanisms of the data-generating process, and simulates time dependence.


Finally, we empirically characterize the distributional effects of each drift type using \ac{CaDrift}.
Our experiments show that drift events with different causal origins induce distinct patterns of change in marginal and conditional distributions, as well as differing impacts on predictive performance.
Furthermore, we show that leveraging \ac{CaDrift} for data augmentation improves downstream performance on real-world datasets.
These results highlight the importance of reasoning about causal mechanisms when studying concept drift and evaluating adaptive learning methods.
We argue that a causal perspective enables more principled analysis of distributional changes and can inform the development of more robust learning algorithms under non-stationarity \citep{scholkopf2021towardCRL}.
\ac{CaDrift}'s source code is available in our GitHub repository\footnote{https://github.com/eduardovlb/CaDrift.}.

Our main contributions are as follows:
\begin{itemize}
    \item We introduce a causal taxonomy of concept drift grounded in structural causal models.
    \item We develop a synthetic data stream generator capable of simulating mechanism-level drift events and temporal dependencies.
    \item We empirically validate the proposed taxonomy, showing that different drift types induce qualitatively different effects on marginal and conditional distributions, as well as on downstream model performance.
    \item We demonstrate the practical utility of our framework by applying it to data augmentation, showing improved performance in streaming scenarios.
\end{itemize}
\section{Background}
\label{sec:theory}

\textbf{Notation.} Let $\mathbf{x} = (X_1,\dots,X_d) \in \mathcal{X} \subseteq \mathbb{R}^d$ denote a $d$-dimensional feature vector, and let $y \in \mathcal{Y}$ denote the target variable (or label). A data stream is defined as an ordered sequence of instances 

\[S=(\mathbf{x}_1,y_1),(\mathbf{x}_2,y_2), \dots,\]

where $(\mathbf{x}_t, y_t)$ corresponds to the observation arriving at time step $t$.
We denote by $P^{(t)}(\mathbf{x},y)$ the joint data-generating distribution at time $t$, with corresponding marginals $P^{(t)}(\mathbf{x})$ and conditionals $P^{(t)}(y \mid \mathbf{x})$.

\textbf{Concept Drift.} \emph{Concept drift} \citep{gama2014survey} is a prevalent phenomenon in data stream mining, characterized by changes in the data distribution over time. Formally, drift occurs when $P^{(t)}(\mathbf{x},y) \neq P^{(t+\delta)}(\mathbf{x},y)$, for some $\delta>0$ \citep{lu2019, hinder2024a}. As a convention, we denote data distributions at different moments in time as $P'(\mathbf{x},y)$.

\emph{Concept drift} is usually divided into two types \citep{gama2014survey, lu2019}: \textit{real concept drift}, also known as \textit{distributional shift}; and \textit{virtual concept drift}, or \textit{covariate shift}.
A \textit{distributional shift} happens when $P(y\mid \mathbf{x})\neq P'(y\mid \mathbf{x})$ \citep{lu2019}. 
This implies that the relationship between features and label evolves over time. 
Consequently, a model trained under distribution $P(\mathbf{x},y)$ may become suboptimal under $P'(\mathbf{x},y)$, even if the feature distribution remains unchanged. 
In this case, the predictive function itself must adapt.

\emph{Covariate Shift} happens when $P(\mathbf{x}) \neq P'(\mathbf{x})$, i.e., the data distribution changes in the feature space but the posterior probability $P(y \mid \mathbf{x})$ remains unaffected \citep{lu2019}. 
As an example, think of an object detection model that has been trained to detect cars. 
If this model were trained using data only from sunny days, it would never see cars in rainy or snowy conditions. 
However, the ``true'' concept definition of what a car is remains unchanged regardless of weather conditions.
Other types of \emph{concept drift} are usually subtypes of either \emph{distributional} or \emph{covariate shift}.

\textbf{Structural Causal Models.} \acp{SCM} \cite{peters2017elements} are defined as follows:

\begin{definition}[Structural Causal Model]\label{def:scm}
    A \emph{structural causal model} (SCM) $\mathcal{M}$ over a set of variables
    $V = \{X_1, \dots, X_d\}$ consists of a collection of structural assignments
    \begin{equation}
        X_i := f_i(\mathrm{pa}_i, U_i), \quad i = 1, \dots, d,
    \end{equation}
    where $\mathrm{pa}_i \subseteq V \setminus \{X_i\}$ denotes the set of parents of
    $X_i$ in the associated causal graph, and $U_i$ are jointly independent exogenous noise
    variables such that $U_i \independent U_j, \forall i\neq j$.
\end{definition}

\acp{SCM} allow us to model complex cause--effect relationships between variables and the target, guided by deterministic effect mapping functions $f_i$. 

\textbf{Interventions.} With \acp{SCM}, we can reason about the effects of interventions on variables in the causal graph. 
Broadly, interventions can be categorized into two types: \emph{hard interventions} and \emph{soft} (or \emph{structural}) interventions. 
Hard interventions correspond to forcibly assigning a value to a variable, thereby removing the influence of its parents and effectively cutting all incoming edges in the causal graph.
Considering the original graph structure in Figure \ref{subfig:graph-no-int}, a hard intervention is exemplified in Figure \ref{subfig:graph-hard-int}, where we apply a hard intervention to the feature $X_2$.
This is formalized using Pearl’s do-notation \citep{pearl2009causality}: $P(y \mid \text{do}(X_2 = x))$, where the do-operator indicates that the variable $X_2$ is set to the value $x$, independently of its usual causal mechanisms.

\begin{figure}[!b]
    \centering
    \subfloat[Original graph structure.]{\label{subfig:graph-no-int}
        \includegraphics[width=0.3\linewidth]{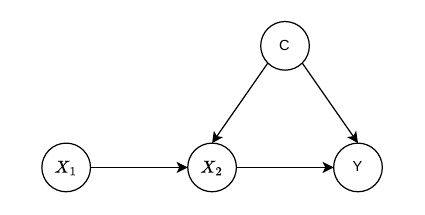}
    }\hspace{3mm}
    \subfloat[Example of a hard intervention on the node $X_2$.]{\label{subfig:graph-hard-int}
        \includegraphics[width=0.3\linewidth]{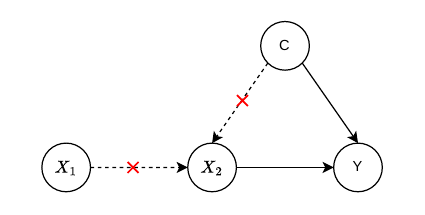}
    }\hspace{3mm}
    \subfloat[Example of a soft intervention on the node $X_2$.]{\label{subfig:graph-soft-int}
        \includegraphics[width=0.3\linewidth]{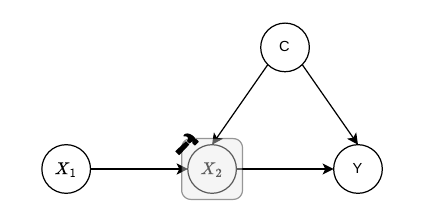}
    }
    \caption{Graphs exemplifying hard and soft interventions on the node $X_2$}
    \label{fig:example-intervention}
\end{figure}

With soft interventions, the causal edges are not cut off from the graph. 
Instead, we modify the structural assignment of a variable, i.e., we replace $f_i$ with a different function $f_i'$, while keeping the same set of parents, as illustrated in Figure \ref{subfig:graph-soft-int}. 
This results in a new data-generating process where the mechanisms relating causes to effects have changed, but the underlying causal graph remains intact.

The framework of \acp{SCM} allows us to reinterpret concept drift through a causal lens. 
In this view, different types of drift correspond to changes in specific components of the causal model. 
For instance, covariate shift can be associated with (soft) interventions on upstream variables (i.e., causes of $\mathbf{x}$), which modify $P(\mathbf{x})$ while leaving the conditional mechanism $P(y \mid \mathbf{x})$ unchanged. 
On the other hand, distributional shift can be understood as interventions on the structural mechanism generating $y$, directly affecting $P(y \mid \mathbf{x})$. 
We explore this further in the next section.

\section{Concept Drift from a Causal Perspective}
\label{sec:drift-causal}

In this work, we consider that concept drift may originate from either observable or unobservable (latent) factors. 
We define \emph{explicit concept drift} as changes affecting observable variables, since these directly alter the cause--effect relationships among the observed components of the data-generating process.

In contrast, \emph{latent concept drift} originates from unobserved variables that are not accessible to the learner. Although these changes occur in latent factors, they can still influence observable associations indirectly. For example, price fluctuations or user preferences may be driven by numerous unmeasured factors, making it impractical to fully specify all relevant variables in real-world environments.

Both \emph{explicit} and \emph{latent drift} can induce downstream changes in the causal graph, altering the joint distribution $P(\mathbf{x},y)$ in different manners depending on which node has drifted.
Within a Structural Causal Model (\ac{SCM}), concept drift occurs when either the distribution of exogenous variables or the mechanisms governing endogenous variables change. Formally, this corresponds to $\mathcal{M}^{(t)} \neq \mathcal{M}^{(t+\delta)}$, for any $\delta>0$.

This formulation reflects the idea that, when the underlying data-generating process evolves, the corresponding \ac{SCM} must also change. We categorize different types of concept drift within the \ac{SCM} framework according to which variables are affected by the drift. The categories can be found in Figure~\ref{fig:causal-taxonomy}. 
Our explanations highlight how different forms of drift emerge through alterations to the cause‑‑effect relationships encoded in the original graph.


\begin{figure}[!b]
    \centering
    \includegraphics[width=0.8\linewidth]{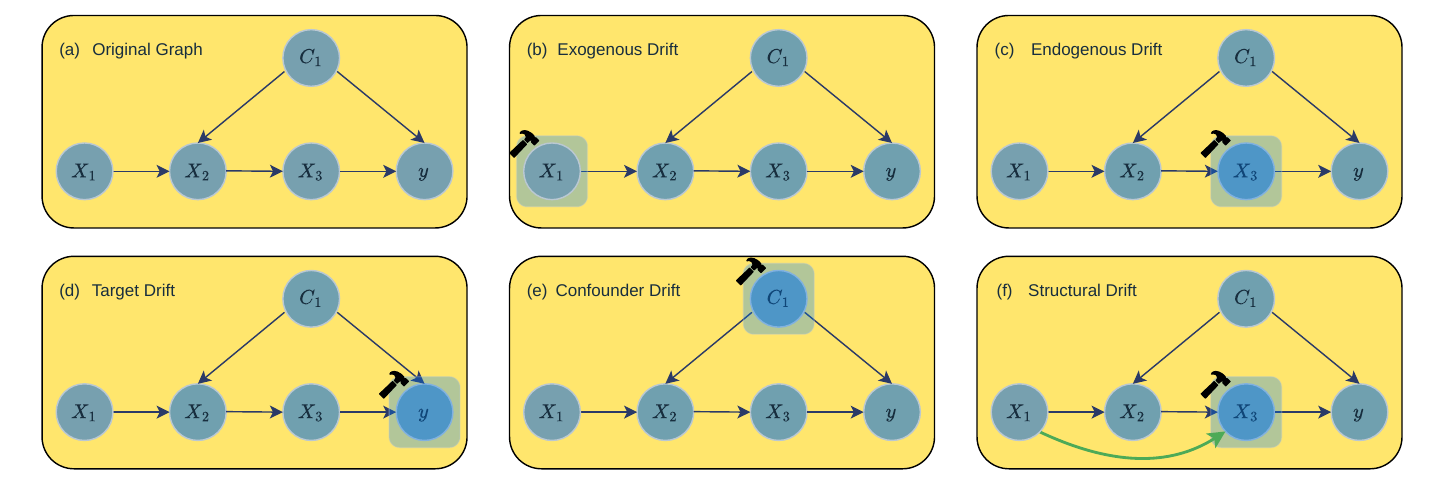}
    \caption{Causal taxonomy of concept drift.}
    \label{fig:causal-taxonomy}
\end{figure}



\emph{Exogenous drift} occurs when the distribution of exogenous variables changes over time, i.e., $P(U) \neq P'(U)$, where $U$ denotes an exogenous variable in the \ac{SCM}. 
In practice, exogenous variables are typically unobserved. 
Therefore, we focus on drift events that affect observable root nodes, whose values are directly determined by exogenous inputs. 
Modeling drift at root nodes allows us to represent changes in upstream environmental factors while keeping the structural mechanisms between variables unchanged.

This type of drift resembles \emph{virtual concept drift}, or \emph{covariate shift}. 
The graph representation of \emph{exogenous drift} is illustrated in Figure~\ref{fig:causal-taxonomy}b, where the mechanism sampling the root node $X_1$ changes.
Importantly, the cause--effect relationships between variables remain intact.

\emph{Endogenous drift} occurs when the mechanism of internal, or endogenous variables, changes, or in other words, there is a change in how one variable affects another. 
To define this mathematically, we should consider the effect function of a node $X_i$: $f^{(t)}_i(pa_i)$, such that $pa_i$ are the parents of the node $X_i$.
\emph{Endogenous drift} takes place when the function $f_i(\cdot)$ changes. 
Hence, $f^{(t)}_i(pa_i) \neq f^{(t+\delta)}_i(pa_i)$. 
This is illustrated in Figure \ref{fig:causal-taxonomy}c.
This type of drift produces a downstream effect in the causal graph, which produces effects on the marginal distribution of the $X_i$-th node descendants.

Essentially, when $P(X_i\mid pa_i)$ changes, the perturbation naturally propagates to the descendants. 
By the law of total probability, both \emph{exogenous} and \emph{endogenous drifts} can change the marginal distribution:

\begin{equation}
    P(X_j) = \sum_{pa_j}P(X_j\mid pa_j)P(pa_j),\ \  \forall X_j \in \text{descendants}(X_i),
\end{equation}

\noindent because either the conditional mechanism $P(X_i \mid pa_i)$ changes (\emph{endogenous drift}) or the upstream marginal $P(pa_i)$ changes (\emph{exogenous drift}), and both propagate through the factorization of the joint distribution.

\emph{Endogenous drift} potentially changes $P(y \mid a)$, where $a$ is the set of ancestor nodes of $X_i$, such that $X_i$ is the drifted node. 
In contrast, the structural mechanism of the target itself remains invariant, assuming that $f_y(pa_y)$ is unchanged. 
In other words, the conditional distribution of the target given all of its direct causes, denoted as $P(y\mid pa_y)$, is kept intact.
For a descendant $X_j$ of $X_i$ (with $X_j \neq y$), $P(y\mid X_j)$ may or may not change, depending on whether $X_j$ blocks every directed path from $X_i$ to $y$.
Concretely, if $X_j$ is a cut vertex separating $X_i$ to $y$ (all paths from $X_i$ to $y$ pass through $X_j$), then $P(y\mid X_j)$ remains invariant.
Otherwise, it may change.
This leads us to the following proposition:

\begin{proposition}[Propagation of Endogenous Drift]\label{prop:endogenous}
    Consider an SCM $\mathcal{M}$ and suppose an \emph{endogenous drift} at node $X_i$, such that the structural mechanism $f_i$ changes while all other mechanisms remain invariant. Assume in particular that the target mechanism $f_y$ is unchanged. Then: 
    \begin{enumerate}
        \item The structural conditional distribution $P(y \mid pa_y)$ remains invariant.
        \item For any ancestor $a \in \text{ancestors}(X_i)$, the conditional distribution $P(y \mid a)$ may change across concepts due to the altered mediation through $X_i$.
        \item Let $X_j$ be a descendant of $X_i$ such that $X_j$ \emph{d-separates} $X_i$ from $y$ -- that is, every path between $X_i$ and $y$ is blocked by conditioning on $X_j$. 
        Then $P(y \mid X_j) = P'(y \mid X_j)$. 
        Conversely, if $X_j$ does not d-separate $X_i$ from $y$, then $P(y\mid X_j)$ may change across concepts.
    \end{enumerate}
\end{proposition}

The proof can be found in Appendix \ref{app:proofs}.




\emph{Target drift} (Figure~\ref{fig:causal-taxonomy}d) is a special case of \emph{endogenous drift} in which the affected node is the target variable itself, i.e., $f_y(pa_y) \neq f_y'(pa_y)$ -- hence, we discard the second assumption in Proposition~\ref{prop:endogenous}.
This causes the posterior probability conditioned on the target node changes for every feature, i.e., $P(y\mid \mathbf{x}) \neq P'(y\mid \mathbf{x})$, such that $\mathbf{x}=\{X_1,\dots,X_d\}$, and $d$ is the number of nodes in the graph that are not the target variable. 
If the target $y$ is not a leaf, it also changes $P(\mathbf{x})$.
\emph{Target drift} produces more drastic changes to the decision process, as the change takes place directly on how direct parents of the target variable influence the outcome:

\begin{corollary}[Target Drift]
    Suppose an \emph{endogenous drift} occurs at the target node $y$, changing its structural mechanism $f_y$. Let $\mathbf{x}$ be the set of all other observable variables. If $y$ is a leaf node (i.e., it has no descendants), then the conditional distribution $P(y \mid \mathbf{x})$ changes across concepts, while the marginal distribution of the features $P(\mathbf{x})$ remains invariant. Conversely, if $y$ is not a leaf node, $P(\mathbf{x})$ may also change, as the drift propagates to any descendants of $y$ contained within $\mathbf{x}$.
\end{corollary}

\emph{Confounder drift} occurs when the drift affects a variable that lies on a backdoor path \citep{pearl2009causality} to the target node. 
In Figure~\ref{fig:causal-taxonomy}e, the mechanism sampling the confounder $C_1$ changes, which might induce a spurious change in the observed association between the node $X_3$ and the target $y$. 
Since $C_1$ induces a spurious (non-causal) association between $X_2$ and $y$, changes in its mechanism alter the observed correlations without modifying the direct causal mechanism of the target, given its true causes.

When such confounders are unobserved, no learner can explicitly condition on an appropriate adjustment set to block the backdoor path \citep{pearl2009causality}. 
Consequently, models may adapt to shifts in spurious correlations rather than to genuine changes in the target-generating mechanism. 
Even when confounders are observed, most stream learners optimize predictive performance without causal constraints, and may therefore rely on associations induced by the backdoor path. 

A practical example arises in medical diagnosis. 
Suppose $X$ represents a treatment and $y$ patient recovery, while an unobserved confounder $C$ corresponds to disease severity. 
Patients with more severe conditions are both more likely to receive the treatment and less likely to recover. 
If the distribution of disease severity changes over time, the observed association between treatment and recovery may change even though the causal effect of the treatment remains unchanged.

From a causal standpoint, confounder drift alters observed associations while leaving the causal effect of true parents on the target invariant.
However, when the confounder is unobserved, this invariance is not statistically identifiable from observational data alone \citep{pearl2009causality}.

\begin{proposition}[Effects of Confounder Drift]\label{prop:confounder-drift}
    Consider an SCM where an exogenous confounder $C$ and a variable $X$ are direct parents of the target $y$, forming a backdoor path $X \leftarrow C \rightarrow y$. 
    Let $Z = pa_y \setminus \{X, C\}$ denote the set of all other parents of $y$. 
    Suppose a drift occurs such that the marginal distribution of the confounder changes ($P(C)\neq P'(C)$), while all other structural mechanisms in the SCM remain invariant. 
    Hence:
    \begin{enumerate}
        \item The observable conditional distribution $P(y \mid X)$ may change across concepts due to the altered spurious association.
        \item The conditional interventional distribution, denoted by $P(y\mid \text{do}(X), c)$, remains strictly invariant.
    \end{enumerate}
\end{proposition}

The proof referring to the propagation of \emph{confounder drift} can also be found in Appendix \ref{app:proofs}.

\textbf{Remark.} While the conditional causal effect remains invariant under \emph{confoudner drift}, it is worth noting that the marginal causal effect (or average treatment effect) does change under confounder drift. 
The interventional distribution is defined by Pearl's backdoor adjustment \citep{pearl2009causality}:

\begin{equation}
    P(y\mid \text{do}(X))=\sum_c P(y\mid X, C=c)P(C=c).
\end{equation}

Because $P(C)$ shifts to $P'(C)$, the outcome of an intervention $P(y\mid \text{do}(X))$ changes, even though the specific mechanism generating $y$ remains fundamentally unchanged.

\emph{Structural Drift} happens when edges in the causal graph change, i.e., the set of parents that cause a node changes. 
It may involve adding or removing causal edges between variables. 
Mathematically, we write this as $pa^{(t)}_i \neq pa^{(t+\delta)}_i$. 

For example, a new tax policy may begin to affect consumer demand for a product. 
Initially, demand $D$ may depend only on the product price $P$, i.e., $P \rightarrow D$. 
After the policy change, the tax variable $T$ may also directly affect demand, yielding the updated structure $P \rightarrow D \leftarrow T$. 
In this case, the parent set of $D$ changes from $pa_D=\{P\}$ to $pa_D=\{P, T\}$, which characterizes a structural drift event.
Figure \ref{fig:causal-taxonomy}f shows an example in which the node $X_1$ becomes a cause of $X_3$.

Most drift types in our taxonomy can be interpreted as soft interventions on the underlying \ac{SCM}, in which structural mechanisms change while the graph structure remains fixed. 
Structural drift, in contrast, corresponds to interventions that modify the graph itself by altering parent sets.

In Table \ref{tab:drift-types}, we position the drift types defined through the causal perspective in contrast to classic and well-known groups of \emph{concept drift}, from the probabilistic perspective. 

\begin{table}[htb]
    \centering
    \caption{Contrasting classical concept drift taxonomy with concept drift from a causal perspective.}
    \begin{tabular}{c c c}
    \toprule
        Classical concept drift &  & Causal perspective \\
        \midrule
        \multicolumn{3}{c}{$P(X) \neq P'(X)$} \\
        Virtual concept drift & \multirow{2}{*}{$\leftrightarrow$} & \multirow{2}{*}{Exogenous Drift} \\
        Covariate shift &  & \\
        \midrule 
        \multicolumn{3}{c}{$P(y\mid X) \neq P'(y\mid X)$} \\
        \multirow{2}{*}{Real concept drift} & $\leftrightarrow$ & Endogenous drift \\
        \multirow{2}{*}{Distributional shift} & $\leftrightarrow$ & Target drift \\
         & $\leftrightarrow$ & Confounder drift \\
         & $\leftrightarrow$ & Structural drift \\
         \bottomrule
    \end{tabular}
    \label{tab:drift-types}
\end{table}


\textbf{Rate of change.} In addition to drift categorization in terms of its intricate factors, the literature also categorizes drifts according to their rate of change, i.e., drifts may happen abruptly, incrementally, or gradually \citep{lu2019}. 

\emph{Abrupt drift} corresponds to an instantaneous change in the underlying \ac{SCM}, i.e., $\delta = 1$. 
\emph{Incremental} and \emph{gradual drifts} involve a transition period ($\delta > 1$), during which the data-generating process evolves over time.

From a causal perspective, \emph{incremental drift} corresponds to a smooth, continuous change in the parameters of one or more structural equations, such that the \ac{SCM} $\mathcal{M}$ is slightly modified at each time step. 
In contrast, under \emph{gradual drift}, two distinct \acp{SCM} coexist: a source model $\mathcal{M}_o$ and a target model $\mathcal{M}_n$ corresponding to the new concept. 
During the transition period $\delta$, data samples are generated by either $\mathcal{M}_o$ or $\mathcal{M}_n$.

Concept drift might also present recurrence, such as the change of seasons during the year.
Under the \ac{SCM} framework, recurrence is simulated by retrieving a previously observed state of the data-generating process, i.e., by reinstating an earlier \ac{SCM} $\mathcal{M}_{t'}$.
This corresponds to the reactivation of past causal mechanisms, thereby making previously valid cause-and-effect relationships relevant again.

\section{Generating Data Streams with Time-dependent Structural Causal Models}

Building on the \ac{SCM} framework, we simulate time-dependence across generated data samples by using autoregressive (AR) noise, sinusoidal seasonality, and an \ac{EWMA} \citep{roberts1959}.
All of these components induce continuous non-stationarity on synthetic data samples. With them, we introduce \ac{CaDrift}, which generates synthetic time-dependent data streams capable of simulating the causal drift events defined in Section~\ref{sec:drift-causal}.

\begin{definition}[Time-dependent Structural Causal Model]\label{def:td-scm}
    A \emph{time-dependent SCM} $\mathcal{M}$ over a set of variables
    $V = \{X_1, \dots, X_d\}$ is defined by the structural assignments
    \begin{equation}
        X_i := f_i(\mathrm{pa}_i, U_i^{(t)}), \quad i = 1, \dots, d,
    \label{eq:timedependentSCM}
    \end{equation}
    where the noise variables follow an autoregressive process:
    \begin{equation}
        U_i^{(t)} = \rho U_i^{(t-1)} + \epsilon_i^{(t)}, \quad 
        \epsilon_i^{(t)} \sim \mathcal{N}(0, \sigma^2), \quad \rho \in [0,1].
    \end{equation}
    Here, $\rho$ controls the temporal smoothness of the noise, determining the degree of dependence between consecutive samples. Higher values of $\rho$ induce stronger temporal correlation, resulting in smoother transitions over time. The impact of this parameter is further analyzed in the experimental section.
\end{definition}

In practice, this process induces a temporal causal dependence on two subsequent data samples, as shown in Figure \ref{fig:td-scm} -- hence, generated data samples are non-iid. 

\begin{figure}
    \centering
    \includegraphics[width=0.3\linewidth]{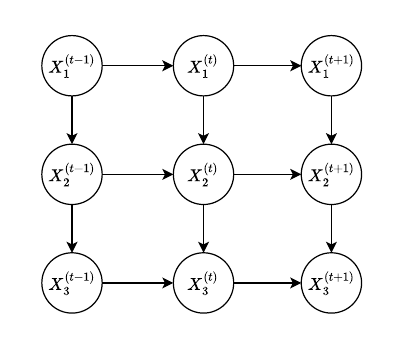}
    \caption{Scheme of a time-dependent SCM.}
    \label{fig:td-scm}
\end{figure}

Next, let us define \ac{EWMA}:

\begin{equation}
    z_t = (1-\alpha)z_{t-1}+\alpha X_t,
    \label{eq:ewma}
\end{equation}

\noindent such that $z_t$ denotes the current average, $\alpha \in [0,1]$ is the smoothing parameter, and $X_t$ the current observation at time $t$. We use \ac{EWMA} at the root nodes to induce temporal correlation in the feature space.
In order to introduce the autoregressive noise $U_i^{(t)}$ on the \ac{SCM} definition, we define time-dependent \ac{SCM}:

Lastly, we incorporate seasonality into the \ac{SCM} framework to capture periodic patterns commonly observed in real-world data streams. 
Seasonality refers to recurring fluctuations that occur at regular intervals, such as daily or yearly cycles.

To model this behavior, we augment each structural assignment with an additive periodic component. 
Specifically, for variables exhibiting seasonality, we define:

\begin{equation}
    X_i^{(t)} := f_i(\mathrm{pa}_i^{(t)}, U_i^{(t)}) + s_i^{(t)},
\end{equation}

\noindent where $s_i^{(t)}$ represents the seasonal component, defined as:

\begin{equation}\label{eq:seasonality}
    s_i^{(t)} = A_i \sin\left( \frac{2\pi t}{T_i} + \phi_i \right).
\end{equation}

\noindent Here, $A_i$ denotes the amplitude of the seasonal effect, $T_i$ the period, and $\phi_i$ an optional phase shift \citep{hyndman2018forecasting}.

By combining autoregressive noise, EWMA smoothing, and seasonal components, the resulting time-dependent \ac{SCM} is capable of generating non-stationary data streams exhibiting both incremental non-stationarity and recurring temporal patterns, in addition to the causal drift events presented in Section~\ref{sec:drift-causal}. 
This enables the simulation of realistic scenarios where both distributional and structural properties evolve over time.

Algorithm \ref{alg:cadrift}, in Appendix \ref{app:pseudocode} shows the pseudocode for CaDrift's synthetic generation process.
As input, \ac{CaDrift} receives a graph $\mathcal{G}$, which can be any \ac{DAG}.
Furthermore, for drift events, it receives a drift schedule $\mathcal{D}$ containing the type of drift, as defined in Section \ref{sec:drift-causal}, the time point $t$ at which it happens, and the drift length $\delta\geq 1$.

\section{Related Work}

\textbf{Concept Drift Adaptation.} Standard approaches typically define concept drift as any change in the joint distribution $P(\mathbf{x}, y)$ \citep{lu2019, hinder2024a}. 
In practice, this broad definition encompasses shifts in the marginal distribution $P(\mathbf{x})$, the conditional distribution $P(y \mid \mathbf{x})$, or both, each of which may impact predictive performance differently.
A large body of work focuses on \emph{error-driven adaptation}, where drift is inferred through changes in model performance.
These methods monitor fluctuations in the prediction error over time \citep{agrahari2022}, triggering adaptation mechanisms such as incremental retraining or ensemble updates \citep{gomes2017, paim2025, barboza2025}.
For instance, the \ac{DDM} \citep{gama2004} monitors the online classification error and raises alarms based on statistically derived thresholds, distinguishing between warning and drift levels.
Similarly, the \ac{ADWIN} detector \citep{Bifet2007} maintains a variable-length sliding window and continuously tests for significant differences between the averages of two sub-windows, providing guarantees on false positive rates.
The \ac{MDDM} \citep{pesaranghader2018mcdiarmid} extends this idea by assigning higher weights to more recent instances and leveraging McDiarmid's inequality to detect changes.
Despite their effectiveness, error-driven methods present notable limitations.
Because they rely on labeled data, their responsiveness is inherently delayed in scenarios where labels are scarce or arrive with latency.

These limitations have motivated the development of \emph{unsupervised} drift detection approaches, which operate directly on the input data distribution.
For example, DD-SCC and DD-KRC \citep{Agrahari2024} measure changes in feature relationships using the Spearman Correlation Coefficient (SCC) and the Kendall Rank Correlation (KRC), respectively.
These methods assume that concept drift manifests as changes in the dependency structure among features, captured by correlation statistics computed over sliding windows.
However, such correlation-based methods may be vulnerable to changes in the underlying causal relationship, such as \emph{confounder drift}, which can trigger false alarms when spurious correlations shift.
Another line of work explores \emph{uncertainty-based} detectors.
Methods such as \ac{UDD} \citep{baier2021} estimate predictive uncertainty from neural networks and monitor its evolution over time as a proxy for drift.
PUDD \citep{Lu_Lu_Liu_Zhang_2025} introduces the \ac{PU-index}, which monitors the classifier's predictive uncertainty to identify concept drift before it manifests as significant drops in accuracy.
However, like error-driven approaches, uncertainty-based methods rely on model-dependent signals that may not fully capture changes in the underlying data-generating process.

Most drift detectors operate by comparing statistics, such as error rates or feature distributions, across consecutive windows.
While effective at identifying distributional changes, these approaches treat the data-generating process as a black box.
\citet{gowerwinter2026windowdilemmaconceptdrift} show that perceived drift may arise from window partitioning artifacts rather than genuine changes in the underlying process.
Moreover, existing detectors primarily indicate \emph{when} drift occurs, but provide limited insight into \emph{where} and \emph{how} the data-generating mechanism has changed.
Taken together, these methods characterize drift through statistical discrepancies or model-dependent signals, without explicitly modeling the causal mechanisms that generate the data.
Our causal taxonomy provides the foundation to address this limitation by explicitly defining drift as mechanistic changes in the underlying system.

\textbf{Causality and Concept Drift.} 
Recent work has begun to study concept drift through a causal lens, leveraging distributional changes to gain insight into underlying data-generating mechanisms.
For example, \citet{komnick2025causalexplanationconceptdrift} propose a framework for post-hoc explanation of drift events, attributing observed performance changes to shifts in specific causal mechanisms and providing actionable insights to users, and \cite{yang2025detecting} considers changes in cause--effect relationships to detect concept drift.
Distributional changes have also been exploited for causal discovery, motivating several extensions of causal discovery algorithms to non-stationary environments.
For instance, \citet{zhang2017cdnod} introduce \ac{CD-NOD}, which exploits distributional changes to recover the underlying \ac{DAG}, under the assumption that such changes correspond to shifts in structural mechanisms.
Similarly, \ac{LIN} considers non-stationary settings with latent interventions, where changes in the data-generating process are not directly observed but can be inferred from distributional variation \citep{liu23lin}.
Moreover, J-PCMCI+ \cite{gunther23ajpcmci} jointly analyzes multiple related time series by introducing context variables that explain differences across environments, while assuming invariant causal mechanisms across datasets.
PCMCI$_\Omega$ \cite{gao2023pcmciomega} also extends PCMCI \cite{runge2019pcmci} to semi-stationary time series, assuming that causal mechanisms vary periodically.
In contrast, JIT-LiNGAM \cite{fujiwara23ajitlingam} learns local linear causal models in the neighborhood of each observation, allowing the recovered graph to adapt continuously to nonlinear and non-stationary dynamics.
Causal-RuLSIF \cite{gao2025causaldiscoverydrivenchangepoint} exploits changes in causal mechanisms through density-ratio estimation to identify causal directions under distribution shifts.
While these approaches use distributional changes to explain, detect drift, or recover causal structure, they do not provide a systematic characterization of different types of drift.
In contrast, our work introduces a causal taxonomy of concept drift, explicitly modeling drift as interventions on \acp{SCM}, enabling controlled analysis of their effects on distributions and learning performance, setting a foundation for linking the fields of causal inference and concept drift.

\textbf{Causal domain generalization.} 
Causal approaches to domain generalization assume that data are generated by an underlying \ac{SCM}, and that certain mechanisms, in particular the target-generating mechanism, remain invariant across environments \citep{arjovsky2020invariantriskminimization, yao2025unifyingcrl}.
Methods such as \ac{ICP} \citep{peters2016icp} and \ac{IRM} \citep{arjovsky2020invariantriskminimization} aim to identify predictors that remain stable under distribution shifts by exploiting these invariances.
Subsequent work has relaxed strict invariance assumptions through risk extrapolation, quantile-based risk minimization, and representation learning approaches \citep{krueger21a_rex, eastwood2022QRM, rosenfeld2021risksinvariantriskminimization, scholkopf2021towardCRL}.
Within our taxonomy, such assumptions correspond to settings where the target-generating mechanism is preserved, such as \emph{exogenous} or \emph{confounder drift}, and, in some cases, \emph{endogenous} and \emph{structural drift}.
In these regimes, invariance-based strategies are well-justified and can yield robust generalization.
In contrast, \emph{target drift} explicitly violates the invariance of the predictive mechanism, highlighting that the effectiveness of domain generalization methods depends critically on the underlying type of causal drift -- an aspect not explicitly captured in existing frameworks.

To address shifts that violate strict invariance, recent work builds on the \ac{ICM} principle and the \ac{SMS} hypothesis \citep{scholkopf2021towardCRL, Chen_Neurips2024}. 
The \ac{SMS} hypothesis posits that distribution shifts typically affect only a small subset of structural mechanisms, while others remain stable, suggesting that robust generalization requires localized adaptation rather than global retraining \citep{bengio2019metatransfer}.
However, most domain generalization methods assume access to discrete, labeled environments during training, whereas real-world concept drift usually happens continuously without explicit boundaries between environments.
By formalizing drift as mechanism-level interventions, our taxonomy and the \ac{CaDrift} framework set the foundation for controlled evaluation of sparse adaptation and causal transfer strategies in continuous streaming settings.

\textbf{Synthetic Data Stream Generators.} 
Synthetic data generators are widely used to evaluate learning algorithms under controlled concept drift scenarios.
Classical generators, such as SEA \citep{street2001}, Hyperplane \citep{hulten2001}, and RandomRBF \citep{bifet2009a}, simulate drift by modifying ideal decision boundaries, feature distributions, or class priors over time.
While these tools enable the evaluation of predictive accuracy under \emph{concept drift}, they rely purely on probabilistic or geometric manipulations.
Consequently, they lack the structural semantics necessary to simulate targeted causal interventions.
In contrast, synthetic data generation based on \acp{SCM} has been widely adopted in causal inference and discovery, where data are generated from user-specified directed acyclic graphs and structural equations.
Examples include simulation frameworks used in libraries such as DoWhy \citep{sharma2020dowhy}, as well as benchmarking platforms like CauseMe \citep{runge2019causeme}, which enable controlled evaluation across diverse causal structures and functional mechanisms.
However, these approaches are typically limited to static settings and do not explicitly model temporal dynamics or evolving data streams.
CauKer \citep{xie2026cauker} is an \ac{SCM}-based generator that incorporates seasonality for training classification time-series foundation models.
However, it lacks a formal taxonomy to distinguish among specific causal origins of drift.
\section{Experiments}

We evaluate whether the proposed \ac{SCM}-based framework (\ac{CaDrift}) generates data streams that reflect well-defined and distinguishable forms of concept drift. 
In particular, we aim to validate three key properties: 
\textbf{(i)} that different causal drift types induce distinct distributional changes in the data, 
\textbf{(ii)} that these changes can be identified through marginal and conditional analyses, and 
\textbf{(iii)} that these changes have a measurable impact on predictive performance.
To this end, we analyze the statistical properties of the generated data streams and relate them to classifier behavior under drift.

After assessing the distributional impact of each causal drift event, we assess the distributional impact when multiple causal drift types happen simultaneously (compound drift) in Section~\ref{subsec:compound-drift}.
After that, we evaluate the serial correlation of samples generated by \ac{CaDrift} in Section~\ref{subsec:autocorr-analysis}.
Finally, in Section \ref{subsec:case-study-elec2} we perform a case study in which we synthesize the real-world ELEC2 dataset with \ac{CaDrift}, and use it for data augmentation in stream learners.

\textbf{Dataset Generation Protocol.} Synthetic datasets are generated using \ac{CaDrift}. We use the handcrafted \acp{DAG} presented in Figure~\ref{fig:graph-mmd}, which gives the basis for us to evaluate the effect of each drift type defined in Section \ref{sec:drift-causal}.
Each node $X_i$ of the \acp{DAG} is associated with a structural equation:
\begin{equation}
X_i^{(t)} := f_i^{(t)}(\text{pa}_i^{(t)}) + U_i^{(t)} + s_i^{(t)}.
\end{equation}

We perform experiments considering linear and nonlinear mechanisms for $f_i$.
To generate the target $y$ for classification tasks, we use a prototype-based mapper, in which prototypes are randomly sampled according to the distribution of $pa_y$, such as in other \ac{SCM} generators \citep{Hollmann2025}.
Each class $y \in \mathcal{Y}$ is assigned to one or more prototypes. 
For each data sample, the classes are assigned to their nearest prototypes based on the Euclidean distance computed on $pa_y$.
This design enables flexible, non-linear class boundaries while preserving dependence on the parent variables.
Further details on prototype generation are provided in Appendix \ref{app:detail-dag}, and drift event details performed for the experiments in this section can be found in Appendix \ref{app:drift-events}.
Unless otherwise stated, drift events are abrupt.

\begin{figure}
    \centering

    \subfloat[DAG \#1.]{\label{subfig:graph-1}
        \includegraphics[width=0.3\linewidth]{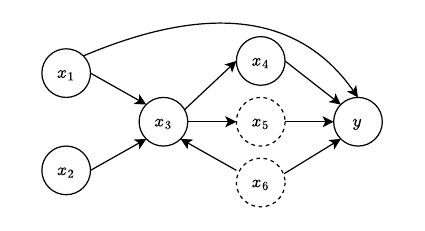}
    }
    \subfloat[DAG \#2.]{\label{subfig:graph-2}
        \includegraphics[width=0.3\linewidth]{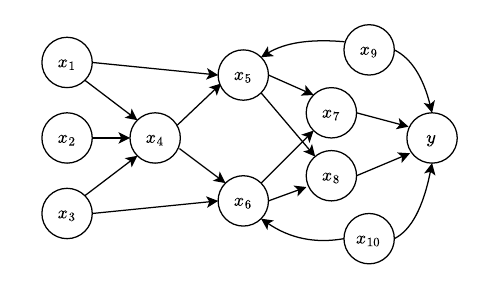}
    }    
    \subfloat[DAG \#3.]{\label{subfig:graph-3}
        \includegraphics[width=0.3\linewidth]{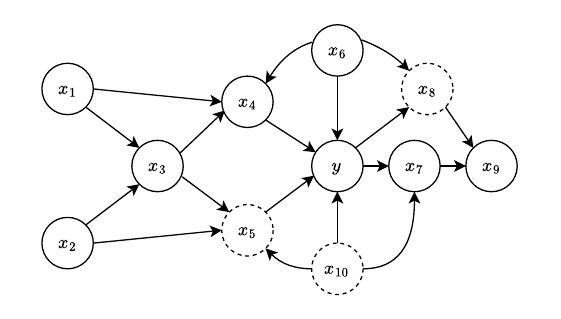}
    }
    \caption{DAGs used to generate data samples for distribution analysis. Dashed nodes represent latent (unobservable) variables.}
    \label{fig:graph-mmd}
\end{figure}

\textbf{Distributional impact analysis.} 
To assess whether different drift types induce distinguishable changes in the data distribution, we analyze the evolution of marginal distributions over time by computing the \acf{MMD} \citep{gretton12ammd} using sliding windows of 200 samples. 
For each window, we compare the distribution of samples from the initial (reference) concept with that in the current window.

For this analysis, we generate streams of 20,000 samples with a drift event introduced every 2,000 samples. 
This configuration provides sufficient observations per concept to obtain stable empirical estimates of \ac{MMD}, while keeping the computational cost tractable given the quadratic complexity of kernel-based \ac{MMD} estimation.

\textbf{Conditional distribution analysis.} 
Since marginal analysis alone cannot capture changes in the predictive relationship between features and the target, we additionally analyze shifts in the conditional distribution $P(y \mid \mathbf{x})$.
To do so, we estimate predictive posteriors using a multinomial logistic regression classifier implemented in scikit-learn with the $lbfgs$ solver and L2 regularization to approximate the conditionals $P(y\mid \mathbf{x})$. 
For each sliding window, we train a model on the current window and compute the conditional KL divergence \citep{lee2024, Kurian2024} between its predicted posterior probabilities and those produced by a model trained on the initial concept.
The KL divergence is averaged over a fixed reference set to ensure comparability across windows.

This procedure captures changes in the predictive relationship between features and the target, allowing us to distinguish drift types that are indistinguishable under marginal analysis (e.g., target drift).

\textbf{Performance evaluation.} 
To evaluate the practical implications of the induced drift, we measure the predictive performance of the \ac{HT} classifier \citep{domingos2000} under different drift scenarios in a test-then-train manner, considering a prompt label availability after the test, as well as when label availability is delayed by 100 samples, as in previous works \citep{gomes2017}.
We also couple the \ac{HT} with the classic \ac{DDM} detector \citep{gama2004}.

For taxonomy validation, we restrict the analysis to stationary concepts without intra-concept temporal dependence (i.e., excluding autoregressive noise, \ac{EWMA}, and seasonality). 
This ensures that observed distributional changes arise exclusively from structural interventions associated with each drift type, isolating the causal impact of the induced drift.
Plots with the non-stationary components enabled can be found in Appendix \ref{app:causal-drift-additional-exp}.


\subsection{The impact of causal drift events on distribution and performance}
\label{subsec:exp-dist}

We analyze the effect of drift events using three complementary measures: marginal divergence (\ac{MMD}), conditional divergence (KL), and prequential accuracy.
The evolution of \ac{MMD}, conditional KL divergence, and prequential accuracy over time for each \ac{DAG} is presented in Figures~\ref{fig:mmd-kl-acc-full} and \ref{fig:mmd-kl-acc-full-nonlinear} for linear and nonlinear mappers, respectively.
We use polynomial and sigmoid functions for the nonlinear mappers.
Results are averaged over 10 datasets generated by each \ac{DAG} per drift type.
Drift events are random adjustments to the weights of the structural equations that map each node, and the graph state is reverted to the original concept before each new drift event, allowing direct comparison with the first concept in the stream.

\begin{figure}[hp]
    \centering
    \subfloat[Exogenous drift.]{\label{subfig:mmd-kl-acc-exogenous}
        \includegraphics[width=0.8\linewidth]{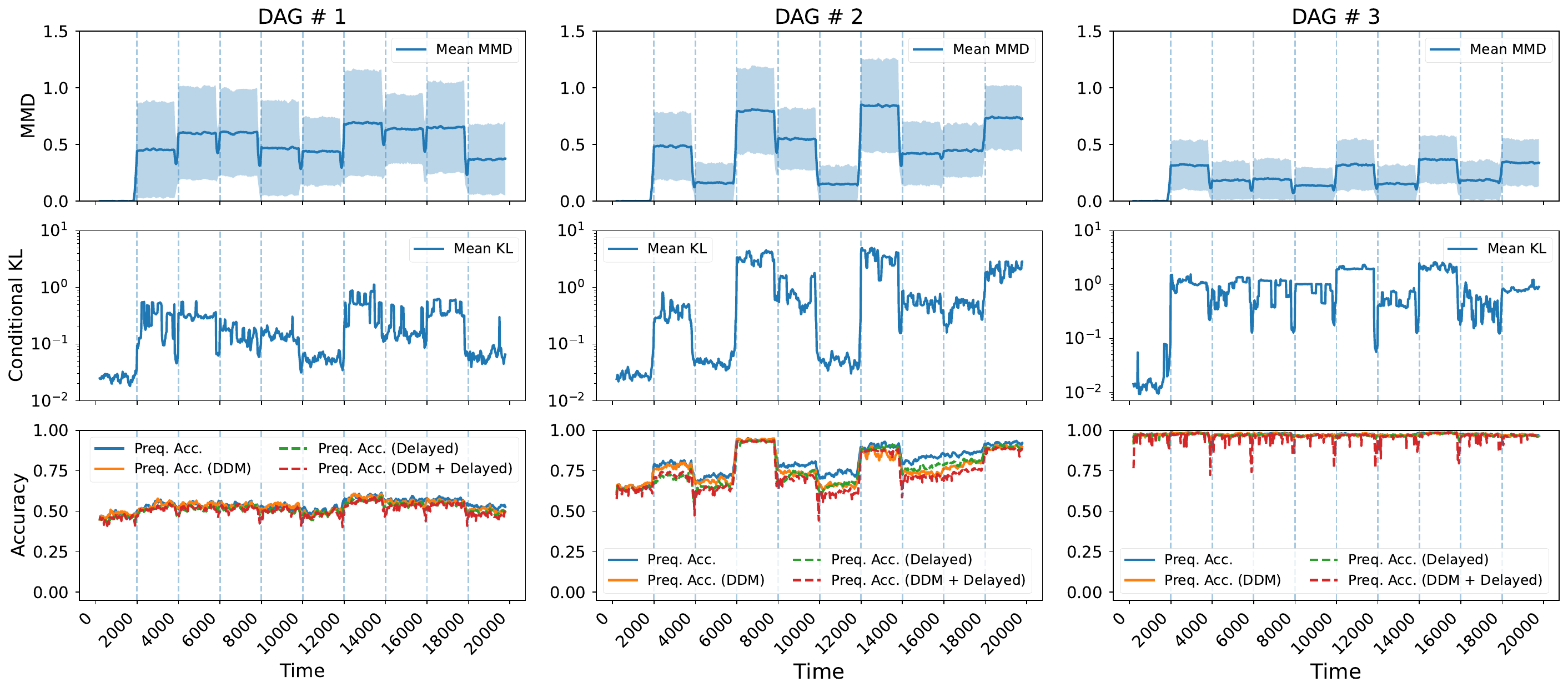}
    }    
    
    \subfloat[Endogenous drift.]{\label{subfig:mmd-kl-acc-endogenous}
        \includegraphics[width=0.8\linewidth]{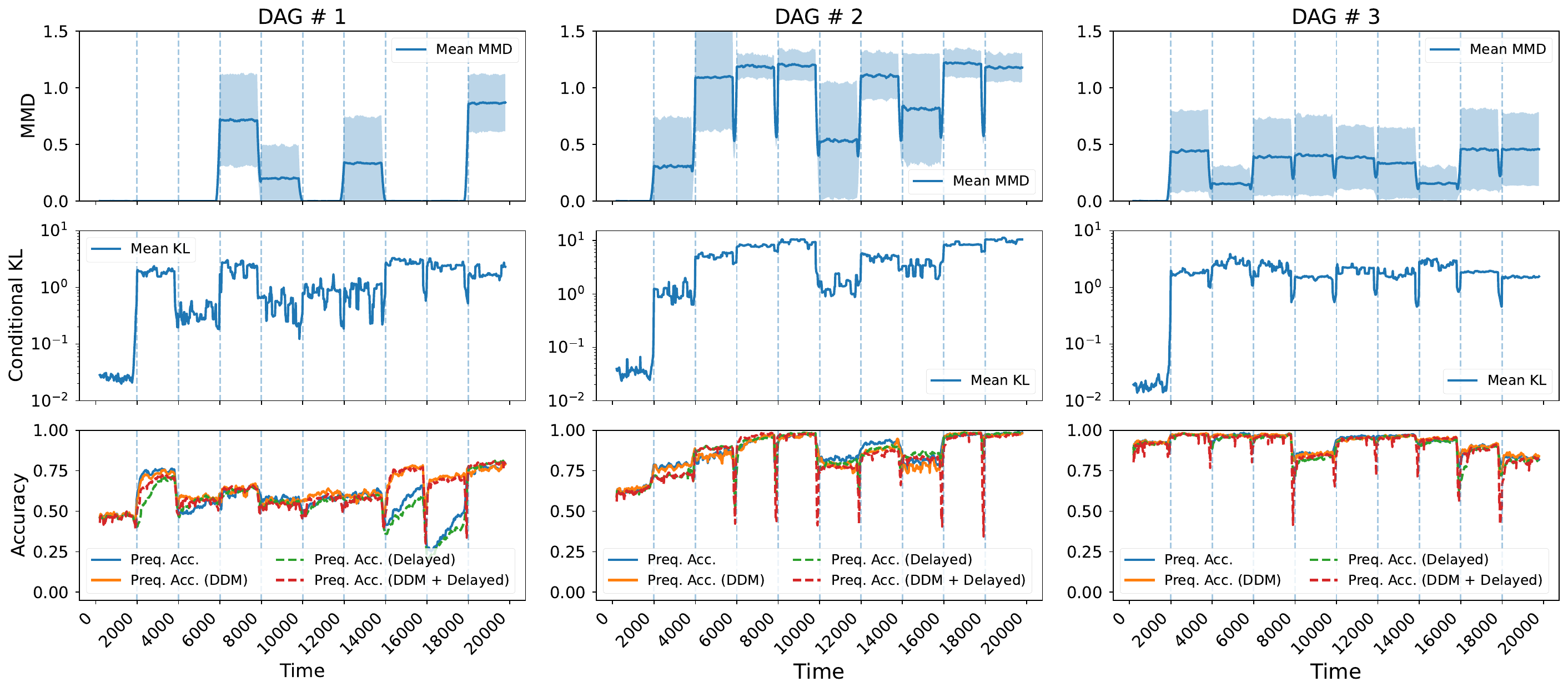}
    }
    
    \subfloat[Confounder drift.]{\label{subfig:mmd-kl-acc-confounder}
        \includegraphics[width=0.8\linewidth]{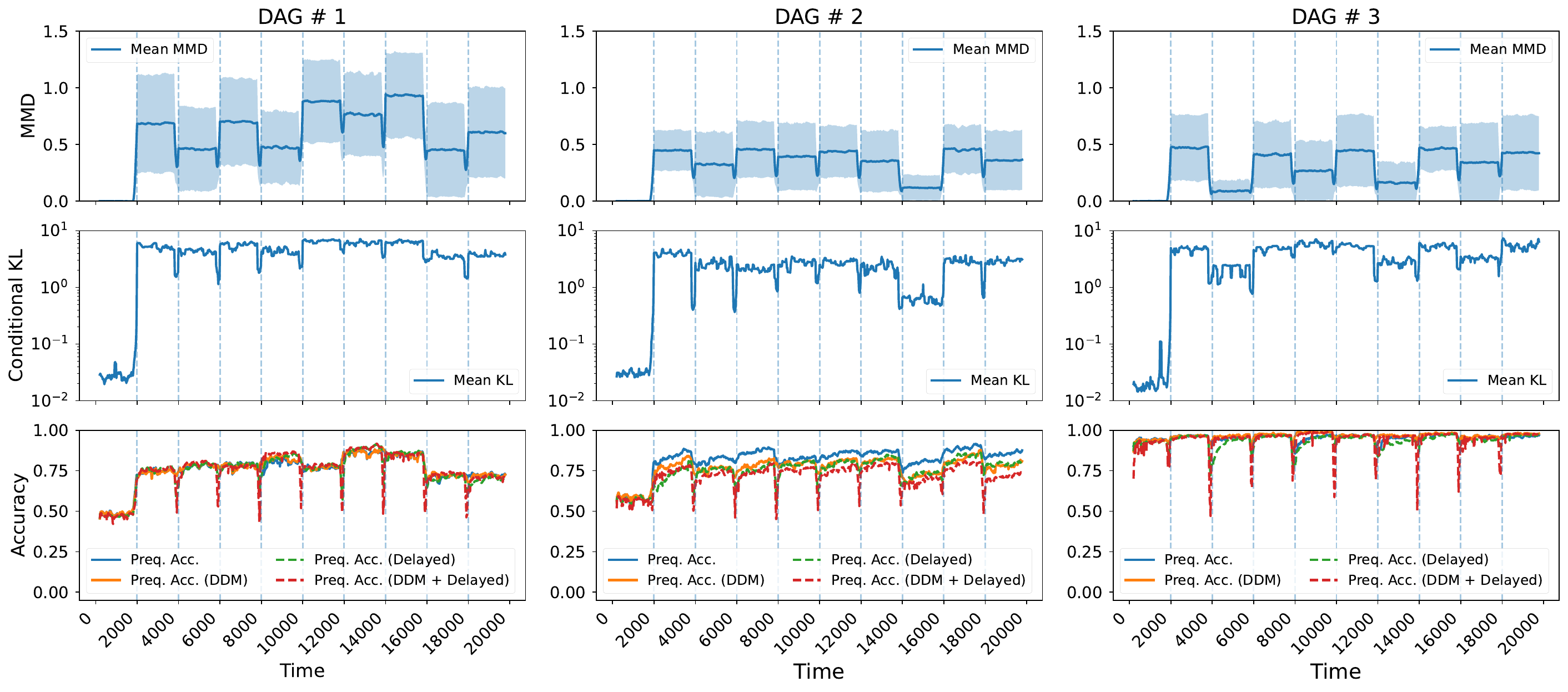}
    }
    
    \caption{MMD, KL divergence, and prequential accuracy over time in each DAG on causal drift types.}
    \label{fig:mmd-kl-acc-full}
\end{figure}

\begin{figure}[ht]
\ContinuedFloat
    \centering
    
    \subfloat[Target drift.]{\label{subfig:mmd-kl-acc-target}
        \includegraphics[width=0.8\linewidth]{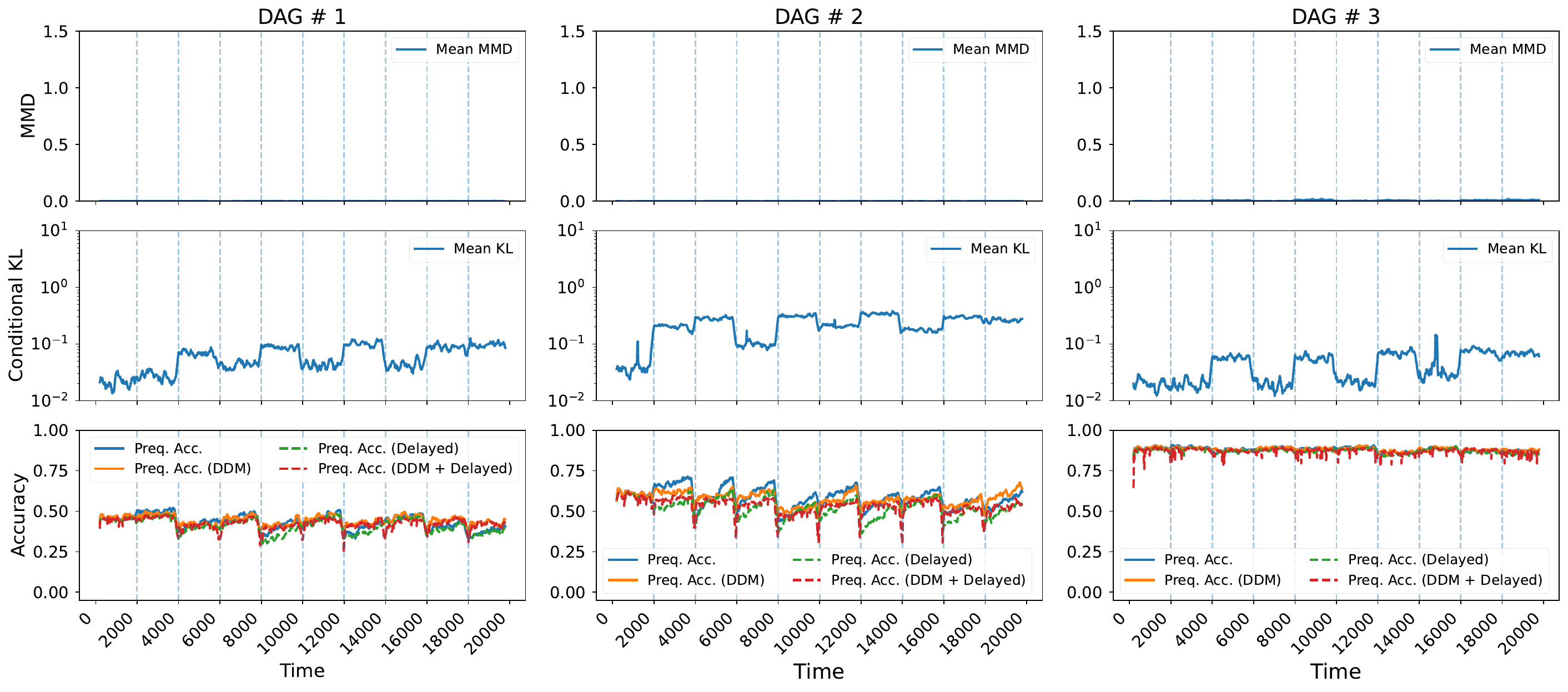}
    }

    \subfloat[Structural drift.]{\label{subfig:mmd-kl-acc-structural}
        \includegraphics[width=0.8\linewidth]{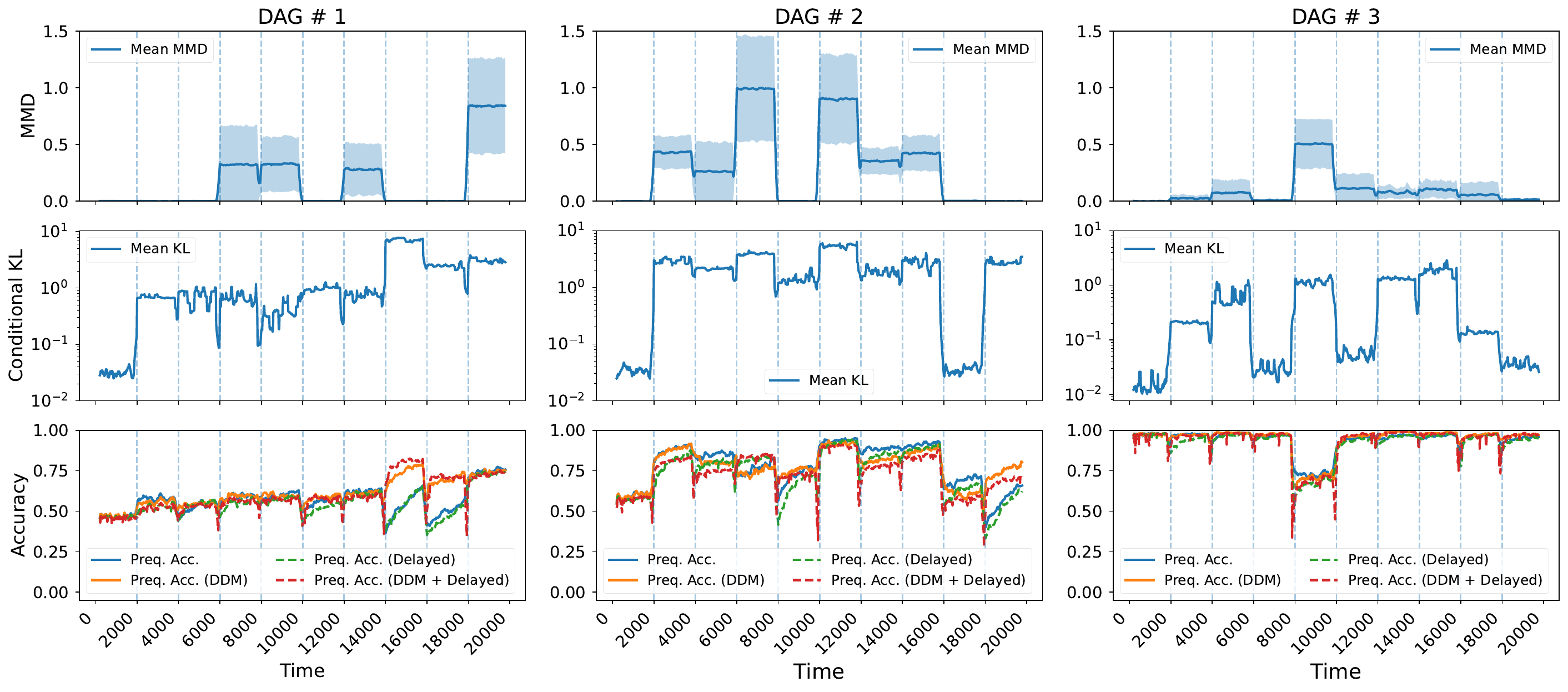}
    }
    
    \caption{MMD, KL divergence, and prequential accuracy over time in each DAG on causal drift  (continued).}
    \label{fig:mmd-kl-acc-full-2}
\end{figure}

\begin{figure}[hp]
    \centering
    \subfloat[Exogenous drift.]{\label{subfig:mmd-kl-acc-exogenous-nonlinear}
        \includegraphics[width=0.8\linewidth]{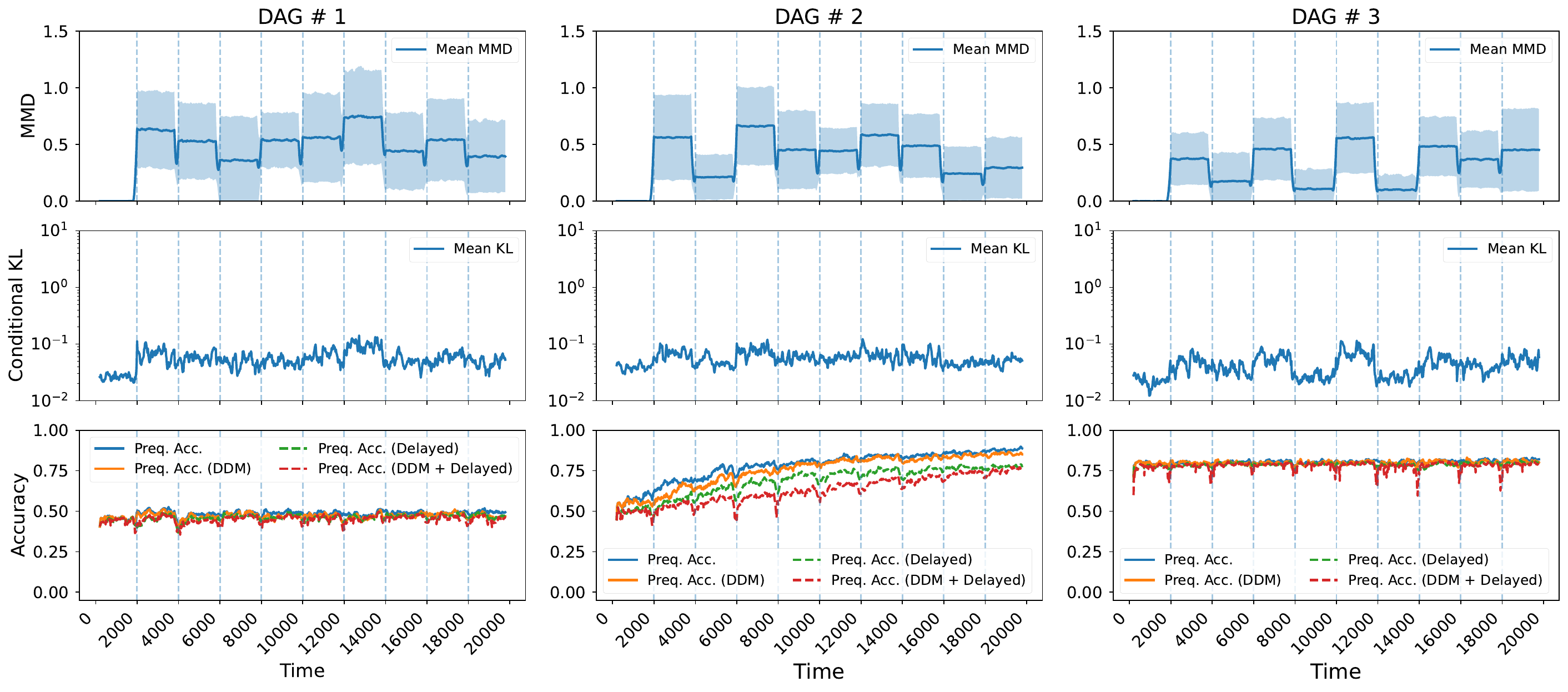}
    }    
    
    \subfloat[Endogenous drift.]{\label{subfig:mmd-kl-acc-endogenous-nonlinear}
        \includegraphics[width=0.8\linewidth]{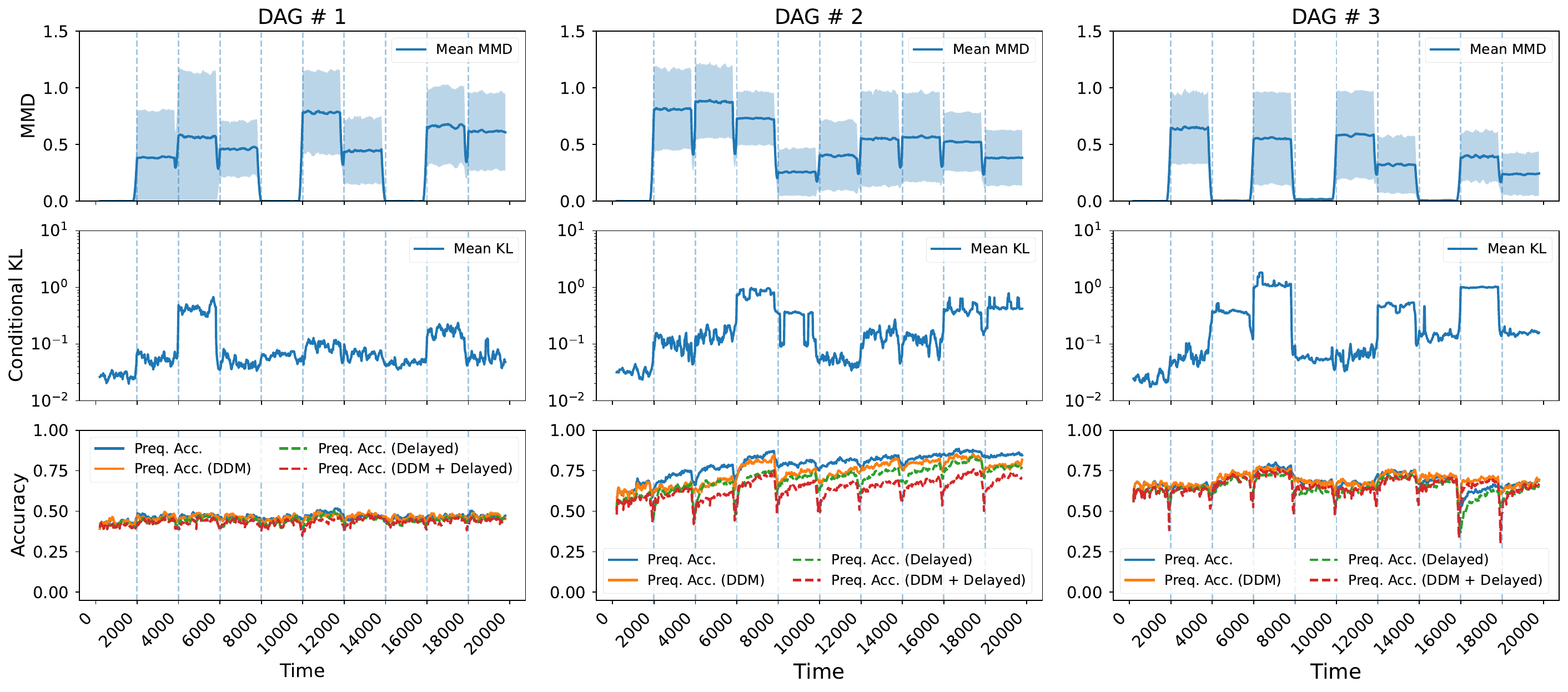}
    }
    
    \subfloat[Confounder drift.]{\label{subfig:mmd-kl-acc-confounder-nonlinear}
        \includegraphics[width=0.8\linewidth]{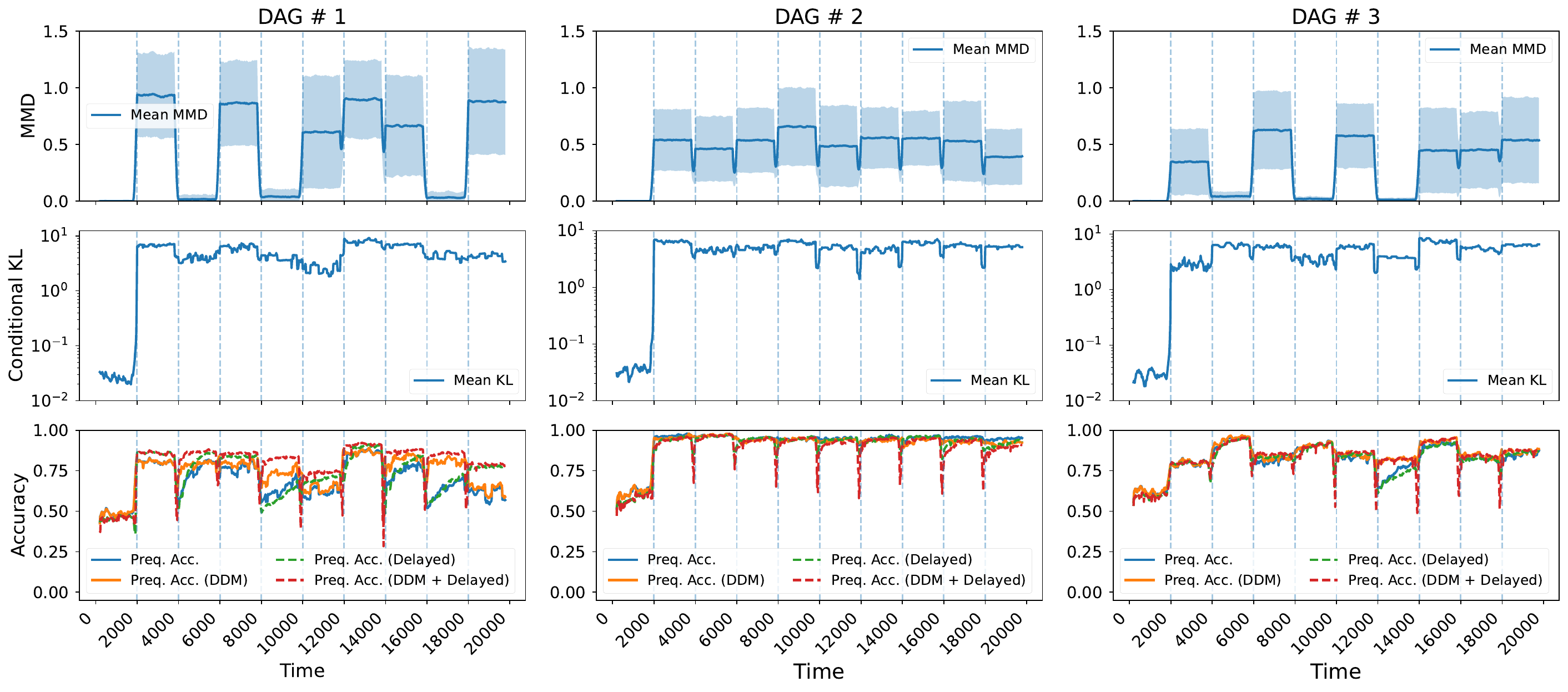}
    }
    
    \caption{MMD, KL divergence, and prequential accuracy over time in each DAG on causal drift types -- Nonlinear mappers.}
    \label{fig:mmd-kl-acc-full-nonlinear}
\end{figure}

\begin{figure}[ht]
\ContinuedFloat
    \centering
    
    \subfloat[Target drift.]{\label{subfig:mmd-kl-acc-target-nonlinear}
        \includegraphics[width=0.8\linewidth]{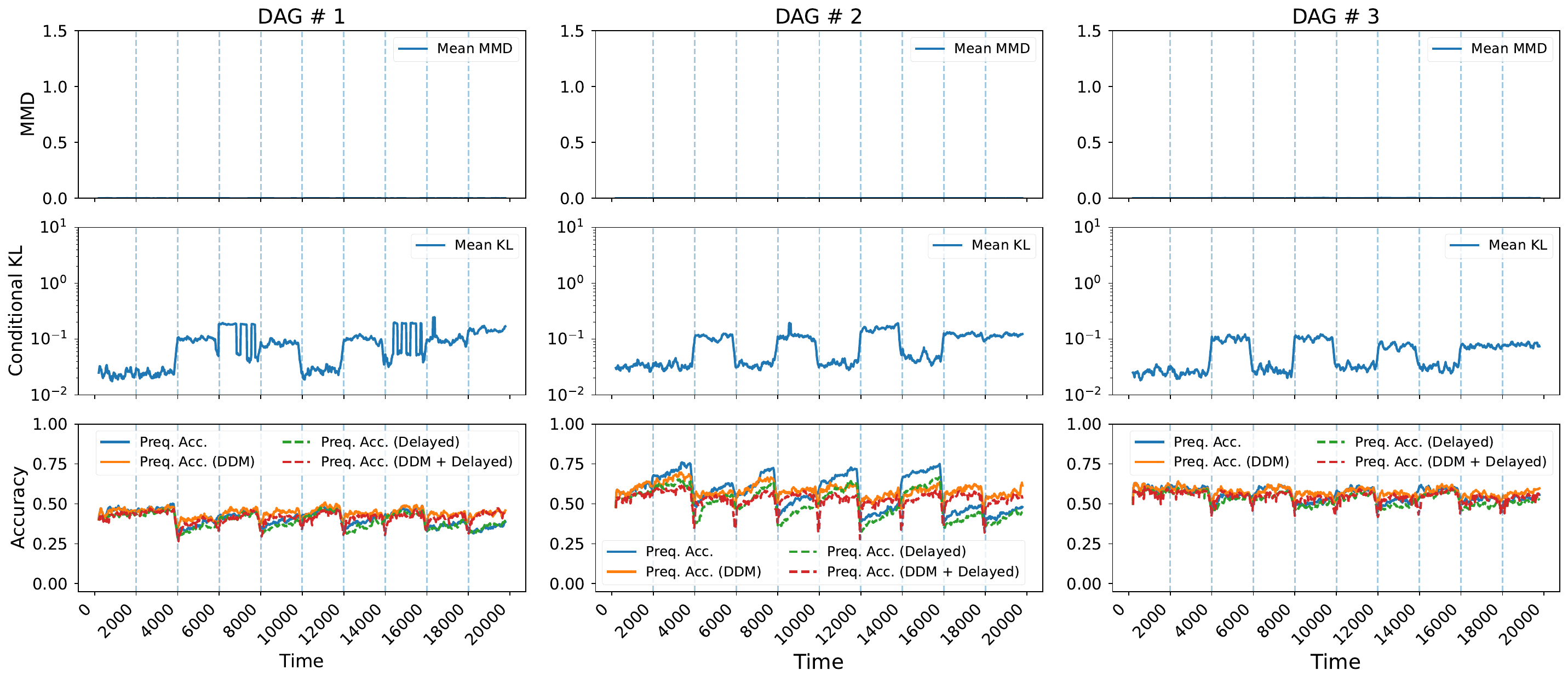}
    }

    \subfloat[Structural drift.]{\label{subfig:mmd-kl-acc-structural-nonlinear}
        \includegraphics[width=0.8\linewidth]{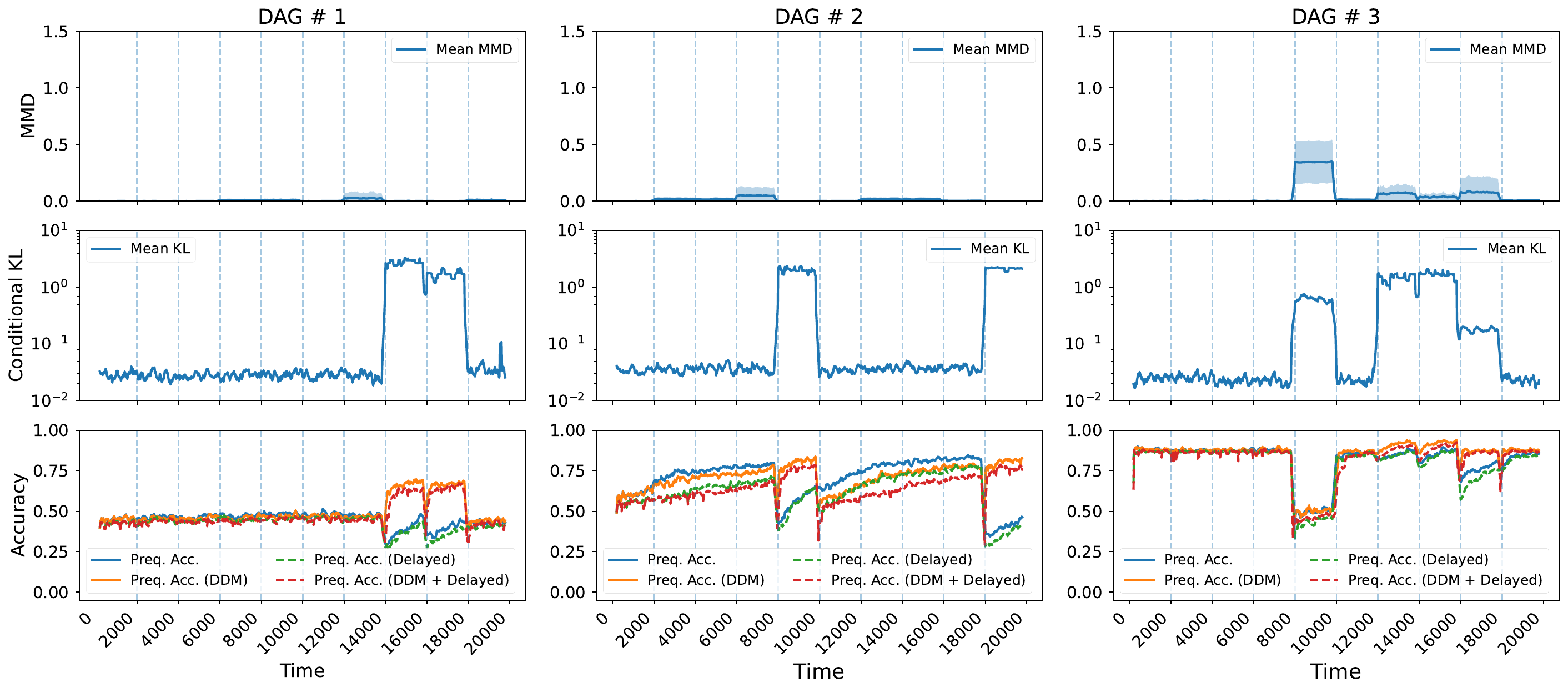}
    }
    
    \caption{MMD, KL divergence, and prequential accuracy over time in each DAG on causal drift -- Nonlinear mappers  (continued).}
    \label{fig:mmd-kl-acc-full-nonlinear-2}
\end{figure}

\emph{Exogenous drift} (Figures~\ref{subfig:mmd-kl-acc-exogenous} and \ref{subfig:mmd-kl-acc-exogenous-nonlinear}), on both linear and nonlinear equations, impacts the marginal distribution measured by \ac{MMD}, as well as the KL divergence, even though the cause--effect relationships between nodes remain intact. 
We observe an impact on prequential accuracy due to \emph{exogenous drift} induced in \ac{DAG} \#2, whereas for \acp{DAG} \#1 and \#3 it usually remains stable, except for minor fluctuations when delayed feedback is applied.

Differences in accuracy performance are best perceived when we apply \emph{endogenous drift} (Figures~\ref{subfig:mmd-kl-acc-endogenous} and \ref{subfig:mmd-kl-acc-endogenous-nonlinear}). 
In this type of drift, we notice drops in accuracy followed by recovery, as the \ac{HT} classifier incorporates sufficient samples from the new concept.
Using \ac{DDM} to actively adapt to \emph{endogenous drift} gives different behaviors.
We hypothesize that the effectiveness of active adaptation via \ac{DDM} is determined by the degree of distributional overlap between concepts.
In some instances of \emph{endogenous drift}, the change creates an inconsistency between the concepts, where the optimal decision boundary for the new concept directly contradicts the previously learned one.
In these cases, adaptation through detectors such as \ac{DDM} is essential to clear the model's ``memory'' of the obsolete concept.
In contrast, other events may simply push the data into a previously unobserved region of the feature space where the underlying causal mechanism remains locally consistent with the global rule.
In such a scenario, the model faces a sample-complexity issue rather than a structural failure -- therefore, incremental learning from new samples is more effective.
This argument is potentialized when we observe \ac{DDM}'s under datasets with nonlinear mappers, where the relationships between the variables are more complex (Figure~\ref{subfig:mmd-kl-acc-endogenous-nonlinear}): \ac{DDM} does not seem to provide any advantage over the HT -- the triggers in concept drift actually seem to decrease the immediate accuracy more than if not using \ac{DDM}.

\emph{Confounder drift}, in Figures~\ref{subfig:mmd-kl-acc-confounder} and \ref{subfig:mmd-kl-acc-confounder-nonlinear}, also produces measurable marginal shifts. 
Its effect resembles that of exogenous drift with respect to the observed covariates, but its impact on the target variable is revealed only when analyzing conditional distributions and model performance. 
Drops in accuracy are observed, even though the inner and target nodes remained invariant -- as expected from the theoretical insights explained in Section \ref{sec:drift-causal}.
Drops in accuracy are more pronounced when relationships between variables are nonlinear (Figure~\ref{subfig:mmd-kl-acc-confounder-nonlinear}).
These observations highlight that marginal measures alone are insufficient to characterize the causal nature of drift, as \emph{confounder drift} alters statistical associations without modifying the underlying generating mechanisms.
Consequently, performance degradation under this type of drift is primarily a symptom of a model having learned statistical correlations rather than the invariant causal mechanism of the target variable.

As expected, \emph{target drift} (Figure~\ref{subfig:mmd-kl-acc-target} and \ref{subfig:mmd-kl-acc-target-nonlinear}) does not modify covariate distributions on either linear or nonlinear mappers, since the intervention affects only the conditional mechanism of the target variable. 
The effect of this drift type is perceived in the conditional KL divergence and in the accuracy performance -- an increase in KL divergence is associated with a decrease in accuracy on datasets generated by \acp{DAG} \#1 and \#2.
Interestingly, \emph{target drift} on data generated by \ac{DAG} \#3 produced no substantial accuracy degradation on linear mappers.
We can notice some drops on the data generated with nonlinear mappers, but not as substantial. 
As a matter of fact, \ac{DAG} \#3 seems to produce less challenging environments for the \ac{HT} classifier.
We attribute this to the presence of causal descendants of the target, which act as noisy proxies of $y$ and may be exploited by the \ac{HT} classifier during splitting, thus explaining why \emph{target drift} did not lead to degradation in accuracy.
This suggests that the presence of observable descendants of the target can partially mask \emph{target drift}, reducing its impact on predictive performance.

Under \emph{structural drift} (Figures~\ref{subfig:mmd-kl-acc-structural} and \ref{subfig:mmd-kl-acc-structural-nonlinear}), we observe varying effects on both \ac{MMD} and KL divergence: in some cases, both remain low despite noticeable drops in accuracy.
In \ac{DAG} \#3, for instance, a sharp accuracy decline between samples 8,000 and 10,000 on both setups (linear and nonlinear) coincides with the removal of the edge $y \rightarrow X_7$. This supports the hypothesis that observable causal descendants can simplify the learning problem and, when present, mask the effects of \emph{target drift}.
Additional experiments on incremental interventions in inner nodes are provided in Appendix~\ref{app:incremental-endogenous-analysis}.

We have also performed experiments where we enable the components to induce time dependence in Appendix~\ref {app:causal-drift-additional-exp}, where we notice that the distributional signatures and model performance are the same as reported in this section, with some noise on the observed MMD and KL divergence curves.

Overall, the results reveal that each drift type induces a characteristic and interpretable signature across marginal divergence, conditional divergence, and predictive performance, empirically validating the proposed causal taxonomy. 
More importantly, no single metric is sufficient to fully characterize drift: marginal measures such as \ac{MMD} fail to capture purely conditional changes, while conditional divergence alone does not reflect shifts in the input distribution, and predictive performance conflates both effects with model adaptation dynamics. 
This highlights the necessity of a joint analysis to correctly diagnose the nature of drift.
In Table~\ref{tab:signatures}, we summarize the observed signatures in \ac{MMD}, KL divergence, and accuracy performance of each drift type. 

From a causal perspective, these findings demonstrate that the impact of drift on learning systems depends not only on the magnitude of distributional change, but on where the intervention occurs in the data-generating process. 
As a result, two drift events with similar statistical signatures may require fundamentally different adaptation strategies. 
This reinforces the central premise of our framework: understanding and modeling concept drift through a causal lens is essential for disentangling its effects and designing robust adaptive learning systems.


\begin{table}[htb]
    \centering
    \caption{Observed signatures of each causal drift type.}
    \begin{tabular}{l c c c}
    \toprule
        Drift Type & MMD & KL & Accuracy Impact \\
        \midrule
        Exogenous & $\uparrow$ & $\uparrow$ & low to moderate \\
        Endogenous & $\uparrow$ & $\uparrow$ & moderate to high \\
        Confounder & $\uparrow$ & $\uparrow$ & moderate \\
        Target & none & $\uparrow$ & moderate to high \\
        Structural & variable & variable & variable \\
        \bottomrule
    \end{tabular}
    \label{tab:signatures}
\end{table}

\subsection{Compound Drift}
\label{subsec:compound-drift}

Although different causal drift types produce distinct distributional signatures depending on which node suffered the intervention, it is common in real-world scenarios that many features change at the same time.
For instance, in energy demand forecasting, a shift in weather patterns (exogenous drift) may co-occur with a new pricing regulation that alters how consumers respond to temperature variation (target drift).
In such cases, from a causal perspective, concept drift is a composition of multiple causal interventions.

In this section, we investigate whether the distributional signatures measured by MMD and KL divergence retain their identifiers.
Cases of compound drift when events are combined with \emph{target drift} are presented in Figure~\ref{fig:mmd-kl-acc-compound-target}, and some cases when combined with \emph{endogenous drift} in Figure~\ref{fig:mmd-kl-acc-compound-endogenous}.

\begin{figure}[hp]
    \centering
    \subfloat[Exogenous and target drift.]{\label{subfig:mmd-kl-acc-compound-exogenous-target}
        \includegraphics[width=0.8\linewidth]{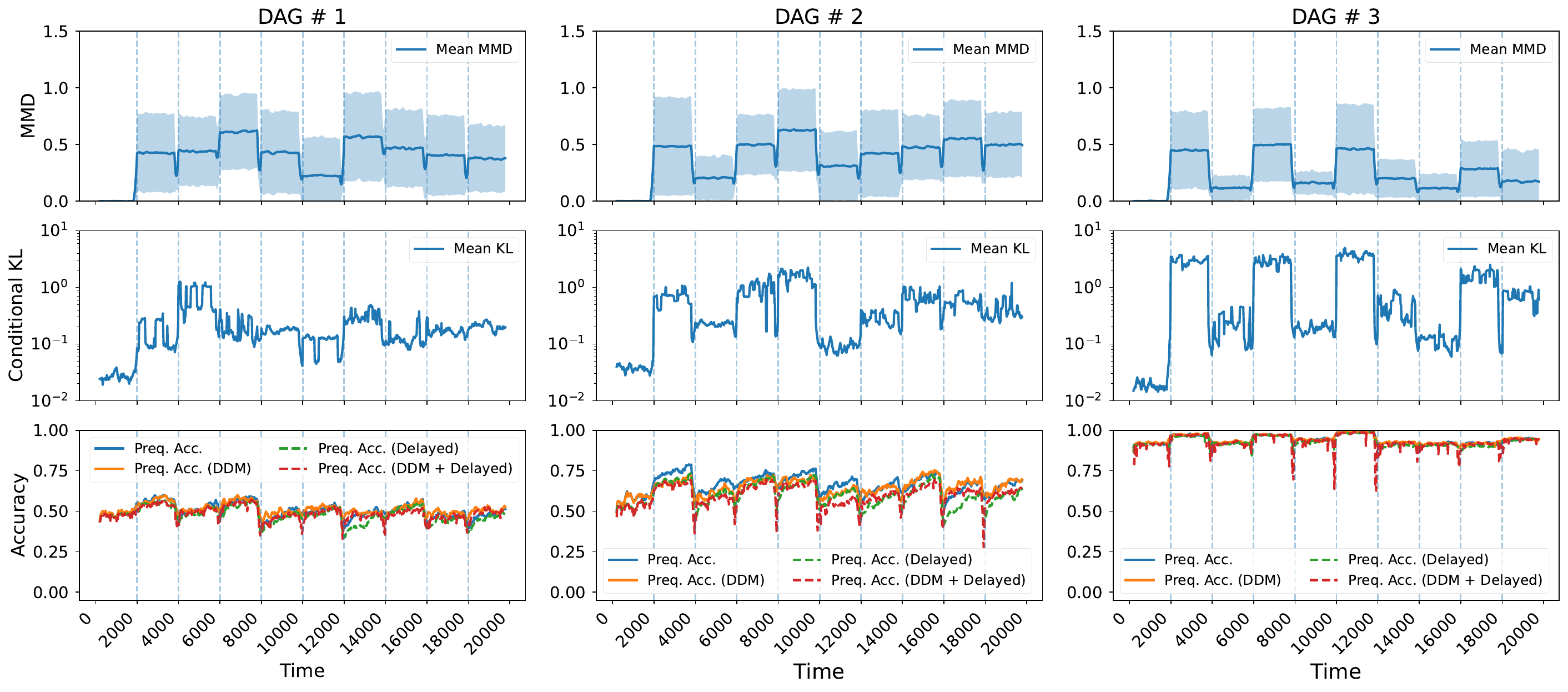}
    }    
    
    \subfloat[Endogenous and target drift.]{\label{subfig:mmd-kl-acc-compound-endogenous-target}
        \includegraphics[width=0.8\linewidth]{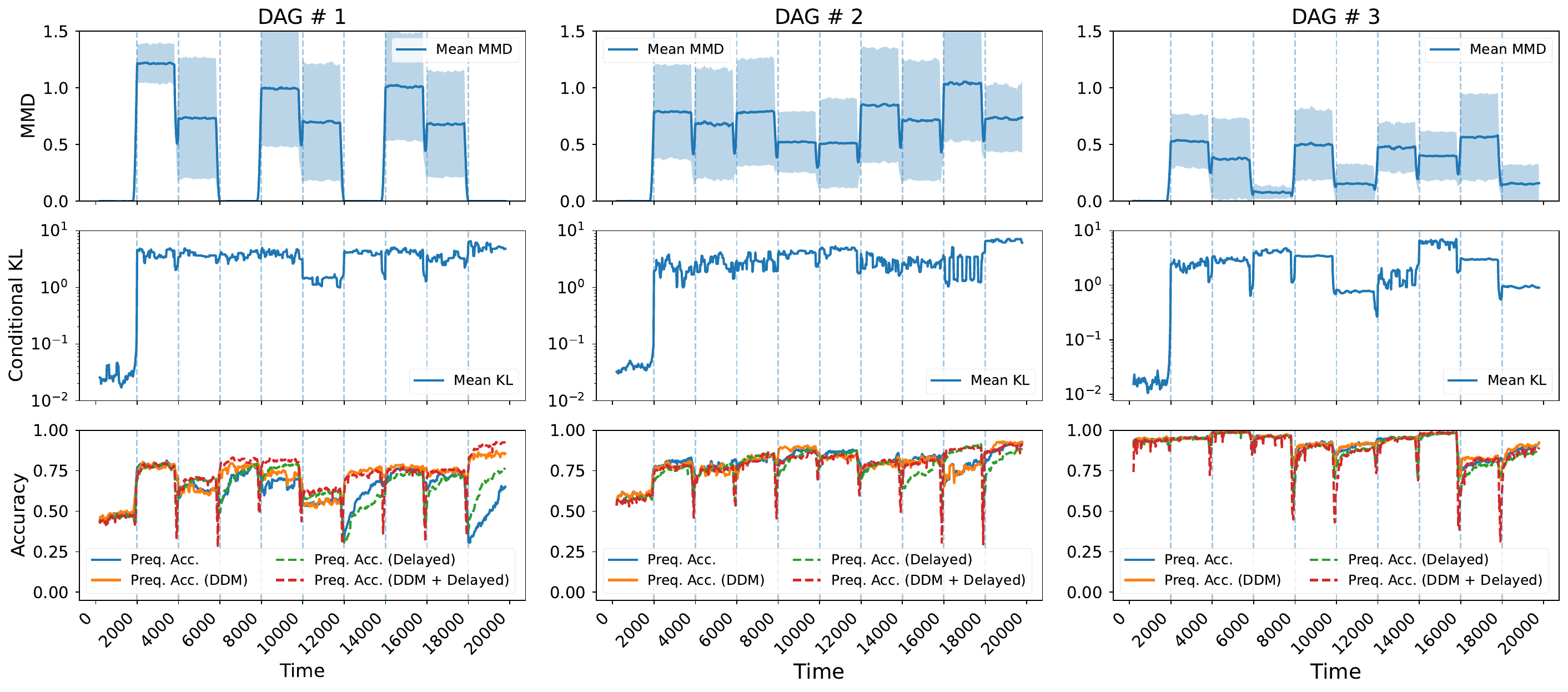}
    }
    
    \subfloat[Confounder and target drift.]{\label{subfig:mmd-kl-acc-compound-confounder-target}
        \includegraphics[width=0.8\linewidth]{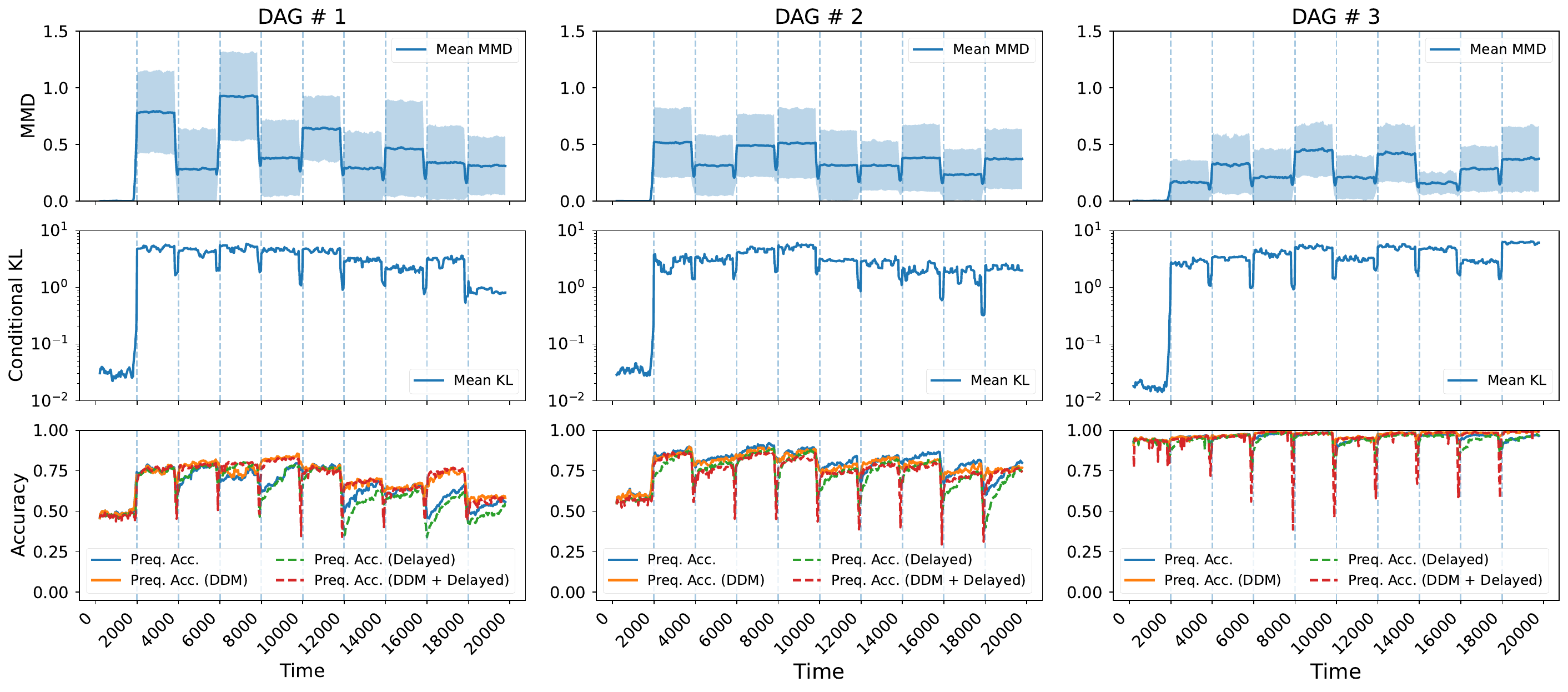}
    }
    
    \caption{MMD, KL divergence, and prequential accuracy over time for compound drift in each DAG on causal drift events for all drift types along with target drift.}
    \label{fig:mmd-kl-acc-compound-target}
\end{figure}

\begin{figure}[ht]
\ContinuedFloat
    \centering
    
    \subfloat[Structural and target drift.]{\label{subfig:mmd-kl-acc-compound-structural-target}
        \includegraphics[width=0.8\linewidth]{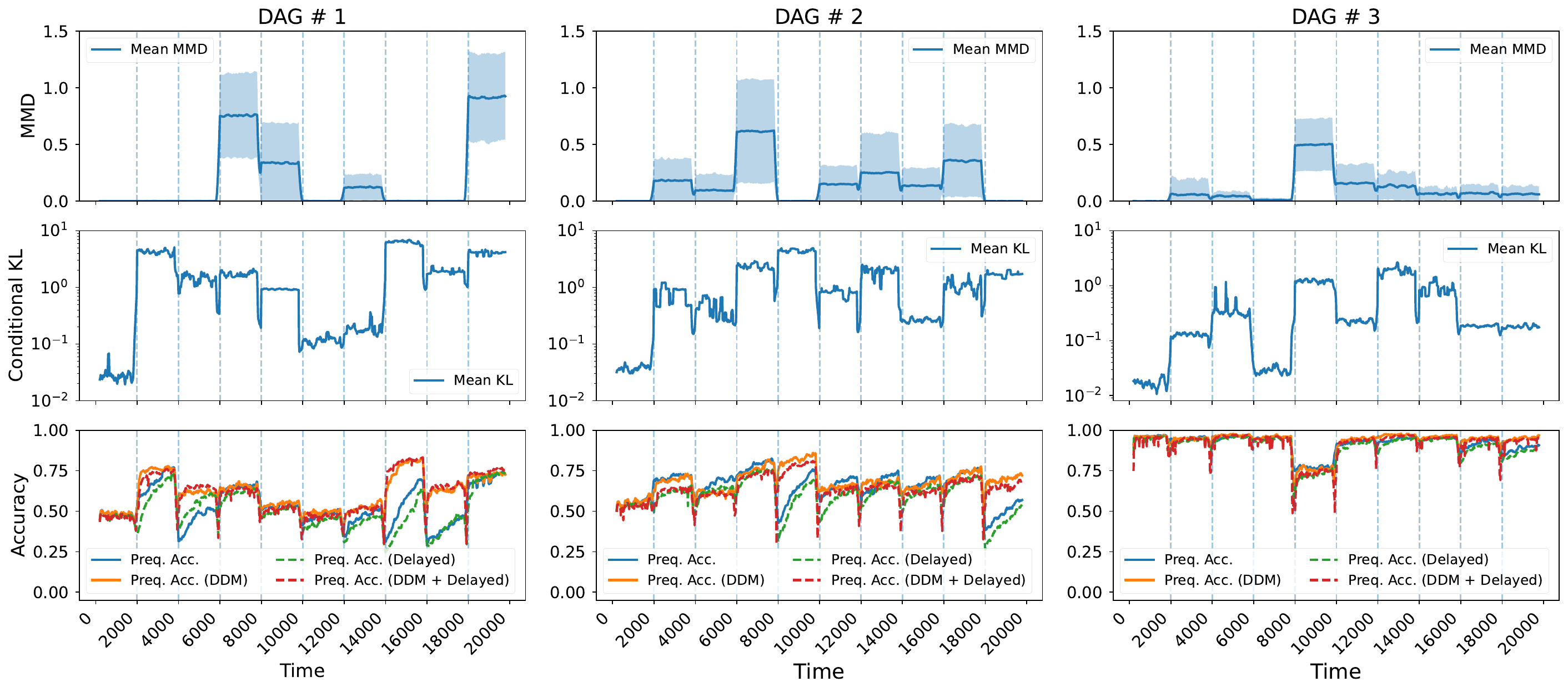}
    }
    
    \caption{MMD, KL divergence, and prequential accuracy over time for compound drift in each DAG on causal drift events for all drift types along with target drift (continued).}
    \label{fig:mmd-kl-acc-compound-target-2}
\end{figure}

\begin{figure}[htb]
    \centering

    \subfloat[Exogenous and Endogenous drift.]{\label{subfig:mmd-kl-acc-compound-exogenous-endogenous}
        \includegraphics[width=0.8\linewidth]{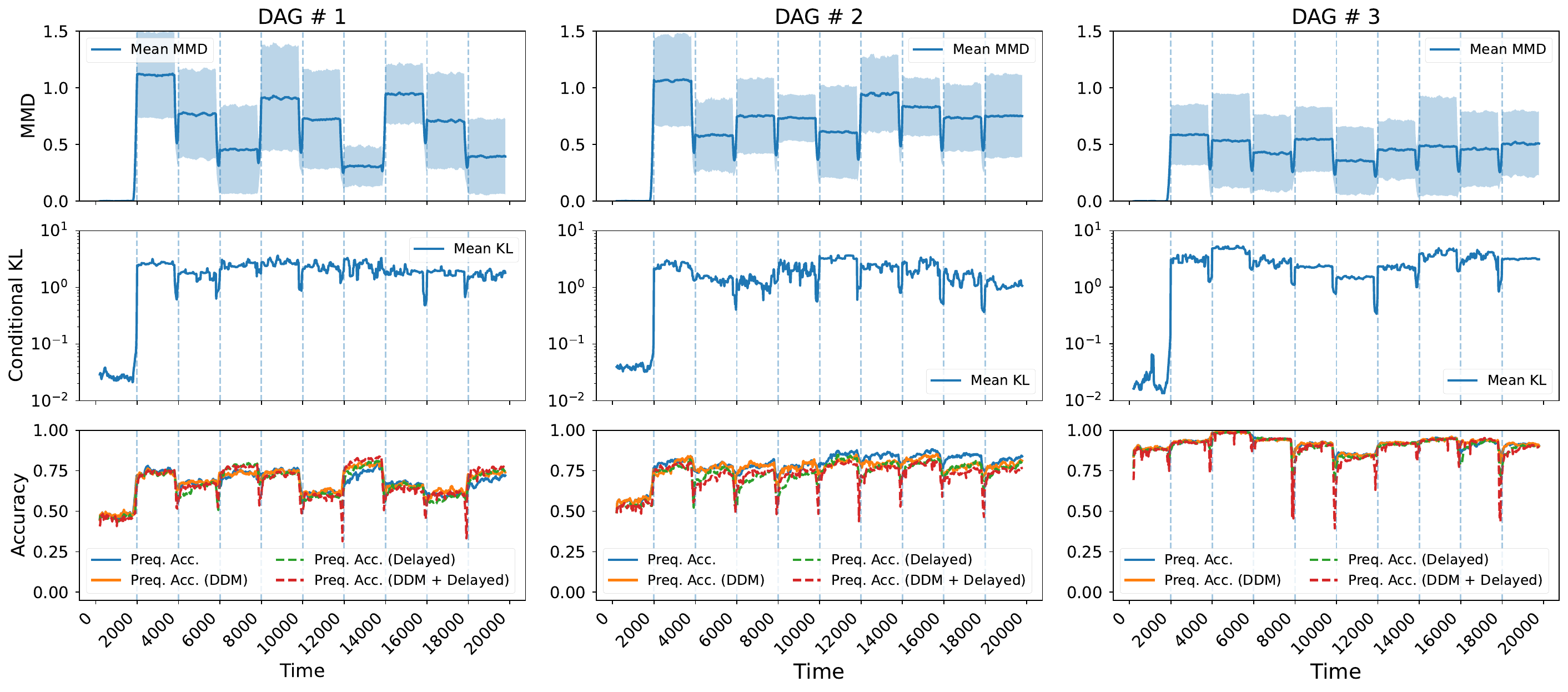}
    }

    \subfloat[Confounder and Endogenous drift.]{\label{subfig:mmd-kl-acc-compound-confounder-endogenous}
        \includegraphics[width=0.8\linewidth]{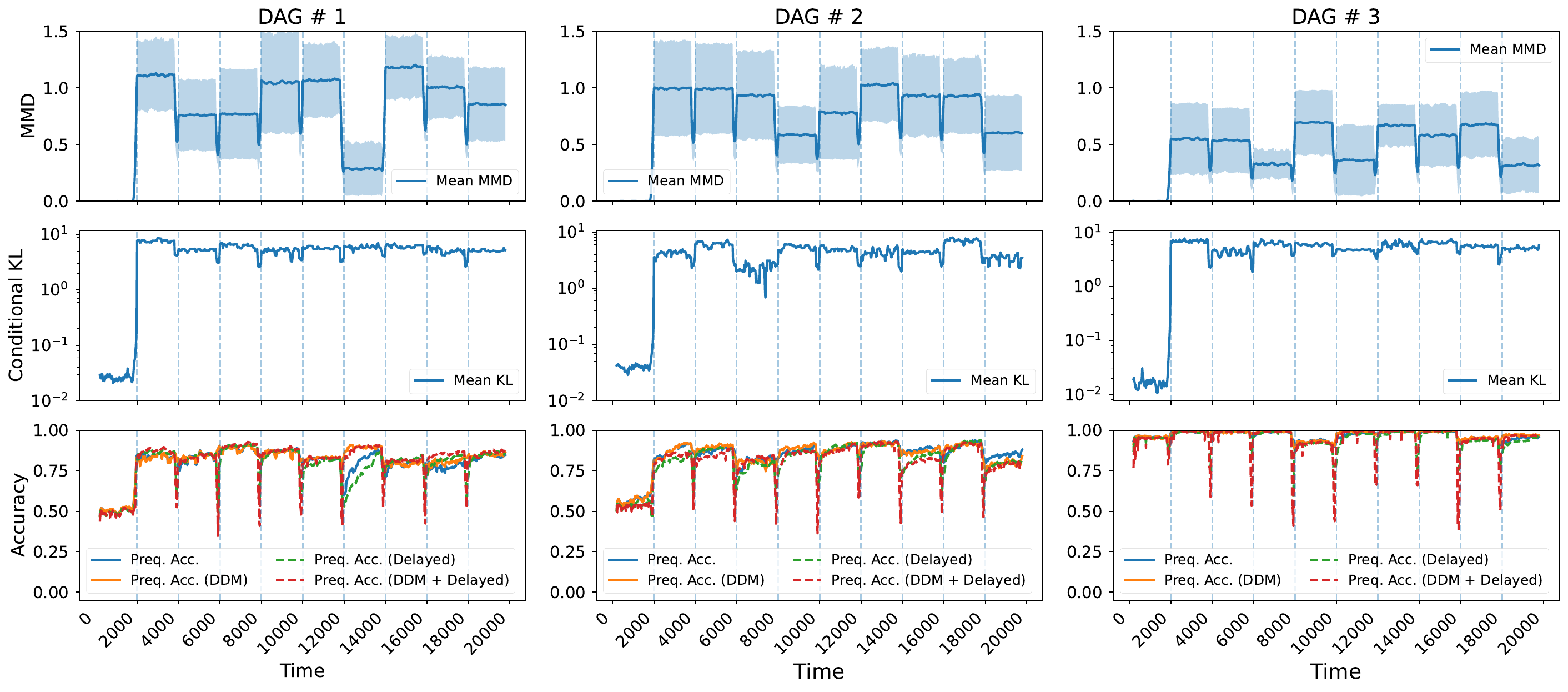}
    }
    
    \caption{MMD, KL divergence, and prequential accuracy over time for compound drift in each DAG on causal drift events for all drift types along with endogenous drift (continued).}
    \label{fig:mmd-kl-acc-compound-endogenous}
\end{figure}

As observed in Section \ref{subsec:exp-dist}, \emph{target drift} has no influence on the MMD over time, whereas drift types such as \emph{exogenous, endogenous,} and \emph{confounder} presented a direct effect on MMD.
We observe the same MMD signatures when \emph{exogenous drift} happens along with \emph{target drift} (Figure~\ref{subfig:mmd-kl-acc-compound-exogenous-target}), but the drops in accuracy are more substantial due to the shift in the target mechanism.

The same happens under \emph{endogenous} and \emph{confounder drifts} (Figures~\ref{subfig:mmd-kl-acc-compound-endogenous-target} and \ref{subfig:mmd-kl-acc-compound-confounder-target}) -- MMD is impacted along with the KL divergence, and performance degrades.
In these drift types, the loss in accuracy performance appears more substantial when compared to when only one drift type occurs -- thus, compound target drift increases post-drift performance drop.

Finally, the behavior is replicated under \emph{structural drift}, in Figure \ref{subfig:mmd-kl-acc-compound-structural-target}.
Before, not all \emph{structural drift} events led to changes in accuracy performance, whereas when combined with \emph{target drift}, fluctuations in accuracy performance are now visible.
We also observe the same pattern noticed previously in DAG \# 3, where the HT achieves near-perfect accuracy with minor fluctuations after most drift events, except for when the edge $y \rightarrow X_7$ is removed.

Although other combinations of compound drift are possible, these experiments already provide an indication of how the signatures of atomic causal drift types behave under composition. 
In general, when different interventions are combined with \emph{target drift}, the distributional signatures associated with the non-target component are preserved under compound drift: \emph{exogenous, endogenous, confounder, and structural drift} continue to produce the characteristic changes in MMD and KL divergence observed under their respective atomic interventions. 
At the same time, the simultaneous occurrence of \emph{target drift} introduces an additional change in the target mechanism, which can amplify the degradation in predictive performance. 
Thus, compound causal drift does not necessarily obscure the signatures of its constituent drift types; rather, the effects of multiple interventions can be stacked, with some components primarily affecting the observed data distribution and others directly altering the predictive relationship. 
These results suggest that causal drift signatures may remain informative under compound drift, although disentangling the individual causal sources of a compound event remains a more challenging problem and requires further investigation.

\subsection{Autocorrelation Analysis}
\label{subsec:autocorr-analysis}

To analyze the temporal dependence induced by the proposed components, we generate synthetic streams from \ac{DAG} \#2 under four configurations: no temporal dependence, autoregressive (AR), EWMA, and seasonality.
For each configuration, we generate 10,000 samples and compute the autocorrelation function (ACF) \citep{box2015time} for the generated features averaged over 10 datasets, shown in Figure~\ref{fig:acf-plots}.

These plots illustrate the presence of serial correlation in the generated streams and how temporal dependence propagates through the causal structure. 
Because the target variable is generated from its parents through structural equations, temporal correlations present in upstream variables naturally propagate to the labels.

As expected, larger values of $\rho$ increase the autocorrelation of both features and the target variable at early lags, reflecting stronger dependence between consecutive samples. 
The autocorrelation then gradually decreases with increasing lag, consistent with autoregressive processes in which the influence of past observations decays over time.

The \ac{EWMA} parameter $\alpha$ also plays an important role in shaping the temporal dynamics. 
Small values of $\alpha$ (e.g., $\alpha=0.05$) produce longer memory in the smoothed series, resulting in persistent autocorrelation across a larger number of lags. 
Conversely, larger $\alpha$ values reduce this persistence by assigning greater weight to recent observations, thereby diminishing the time dependence.


\begin{figure}[htb]
    \centering
    \includegraphics[width=0.3\linewidth]{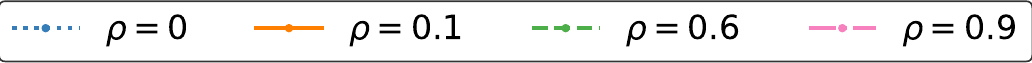}
    \includegraphics[width=.8\linewidth]{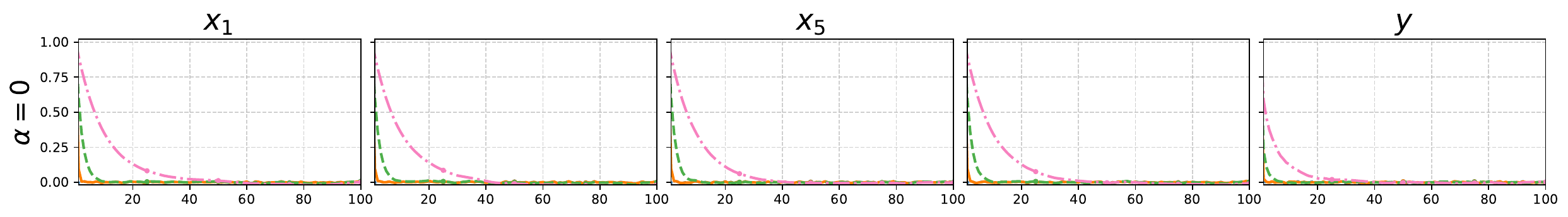}
    \includegraphics[width=.8\linewidth]{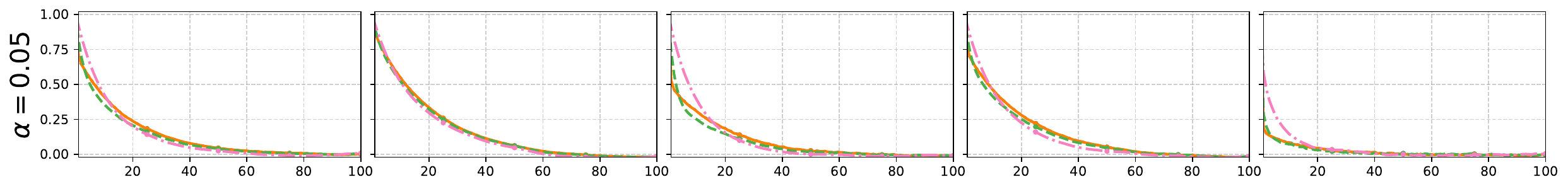}
    \includegraphics[width=.8\linewidth]{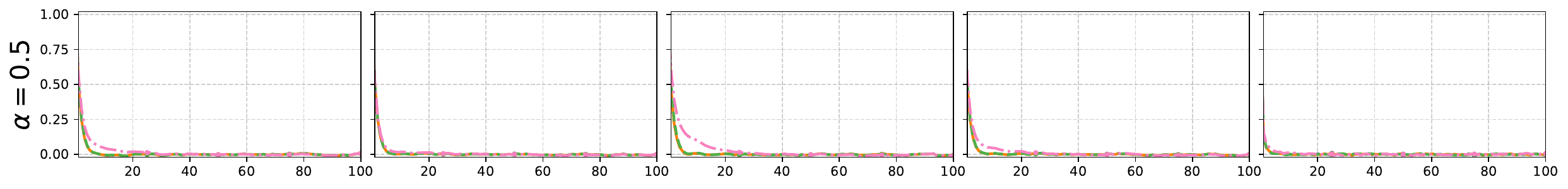}
    \includegraphics[width=.8\linewidth]{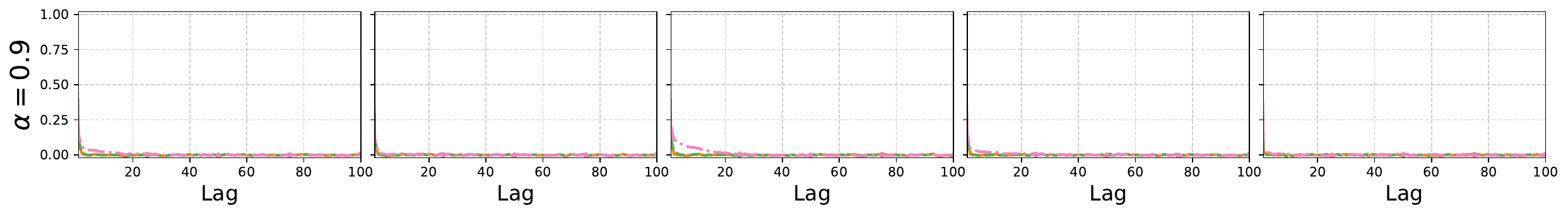}
    
    \caption{The impact of the $\alpha$ and $\rho$ variables on the lagged autocorrelation function in DAG \#2. Each row refers to a different value for $\alpha$, and each column a different feature ($X_1,X_4,X_5,X_8$ and $y$). The y-axis refers to the autocorrelation, and the x-axis to the lag.}
    \label{fig:acf-plots}
\end{figure}

Figure~\ref{fig:acf-seasonality} shows the \ac{ACF} plots when seasonality is introduced in the generation process.
The hyperparameters of the seasonality component (Equation \ref{eq:seasonality}) for these plots are set to $A_i=0.2$, $T_i=50$, and $\phi=0$.
The resulting autocorrelation patterns exhibit periodic peaks characteristic of seasonal time series.

Because \ac{EWMA} smoothing is applied to the root nodes of the \ac{DAG}, the resulting seasonal patterns propagate through the causal structure while being progressively attenuated along the causal chain, due to the combined effect of intermediate transformations and noise terms. 
In practice, larger values of $\alpha$ tend to smooth the seasonal signal more aggressively, reducing the strength of the induced temporal dependence in downstream variables.

Overall, the combination of autoregressive noise, EWMA smoothing, and seasonal components allows the proposed framework to generate a wide variety of temporally dependent data streams, capturing several forms of non-stationarity commonly observed in real-world streaming scenarios.
This makes \ac{CaDrift} a suitable framework for evaluating models under non-stationary streaming data.

\begin{figure}[htb]
    \centering
    \includegraphics[width=0.3\linewidth]{figures/legend_rho.pdf}
    \includegraphics[width=.8\linewidth]{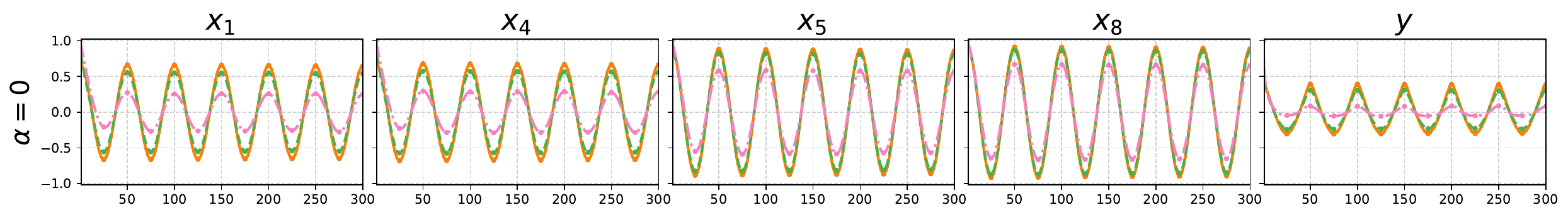}
    \includegraphics[width=.8\linewidth]{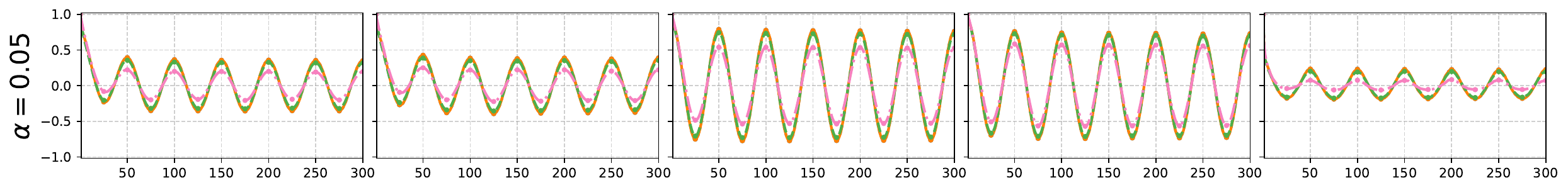}
    \includegraphics[width=.8\linewidth]{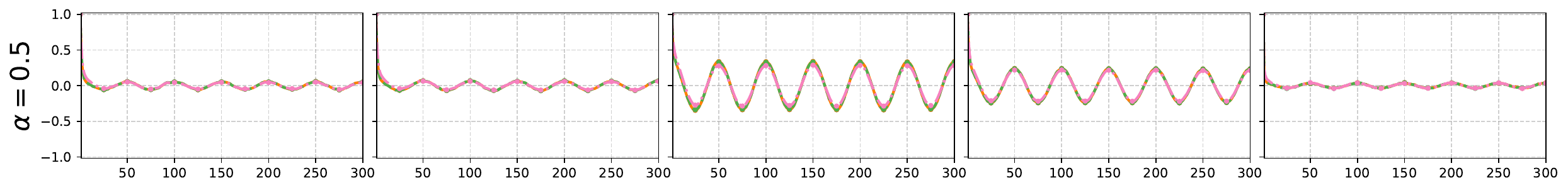}
    \includegraphics[width=.8\linewidth]{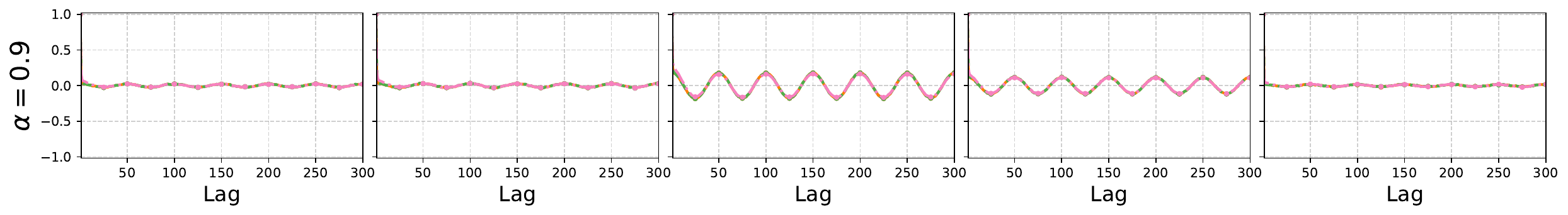}
    
    \caption{ACF plot on features with seasonality in DAG \#2. The y-axis refers to the autocorrelation, and the x-axis to the lag.}
    \label{fig:acf-seasonality}
\end{figure}

To further investigate temporal dependence in data generated by \ac{CaDrift}, we apply the Ljung–Box (LB) test \citep{ljung1978measure}, with detailed results for \ac{DAG} \#2 reported in Appendix~\ref{app:lb-test}.
The LB test rejects the null hypothesis of no autocorrelation when AR, EWMA, or seasonal components are included, confirming the presence of temporal dependence. 
In contrast, when these components are absent, the test does not reject the null, indicating no evidence of autocorrelation and consistency with i.i.d. samples.

\subsection{Case Study: Synthesizing the Electricity Market dataset}
\label{subsec:case-study-elec2}

In this section, we investigate whether causally synthesized data can improve the performance of online learners on the ELEC2 dataset.
The Electricity Market dataset (ELEC2) \citep{harries1999} is widely used to evaluate data stream learning methods, in which the task is to predict whether the price will increase or decrease given demand- and supply-related features. 
The dataset contains time-ordered instances collected at 30-minute intervals, exhibiting temporal dependencies.
Following common practice in data stream mining \citet{losing2016}, we ignore the features $date$ and $nswprice$.

To infer the \ac{DAG} for the ELEC2 dataset, we employ the CD-NOD algorithm \citep{zhang2017cdnod}, a constraint-based causal discovery method.
The DAG inferred from the ELEC2 dataset can be found in Figure \ref{fig:elec-dag-pc}. 
We use the first half of the dataset (22,656 samples) to learn the \ac{DAG} structure via the CD-NOD algorithm.
To detect seasonality in the features, we use the Fast Fourier Transform (FFT) \citep{musbah2019fft}.

\begin{figure}[htb]
    \centering
    \includegraphics[width=0.45\linewidth]{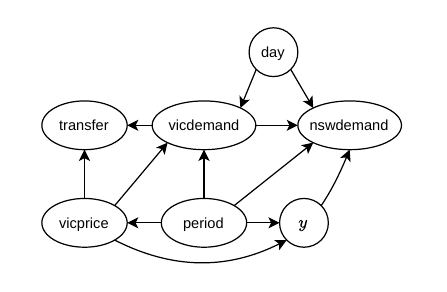}
    
    \caption{DAG inferred from the ELEC2 dataset through the CD-NOD algorithm.}
    \label{fig:elec-dag-pc}
\end{figure}

For each inner node $X_i$ in the \ac{DAG}, we train a feedforward neural network (NN) that maps its parent variables to $X_i$, thereby approximating the causal mechanisms in the data.
When the node is continuous (all endogenous nodes in the ELEC2 \ac{DAG}), we fit an NN regressor. Otherwise (the target node $y$), we fit a classifier.
Further, these learned mappers are used to sample instances in the generated stream. 
Details of the neural network hyperparameters used to map the features are provided in Appendix \ref{app:hyperparameters}.
For the root nodes ($day$ and $period$), we used different functions to simulate the behavior observed in the real-world ELEC2 dataset. 
These details, as well as additional experimental details in this section, are available in Appendix \ref{app:detail-exp-elec2}.
We provide ACF plots for the synthesized data compared to the original ELEC2 dataset in Appendix \ref{app:acf-elec2-synth-real}, where we see that the root nodes closely follow the ACF plots of the real-world ELEC2.
However, for many features, the data synthesized by \ac{CaDrift} do not follow the same seasonal trend -- the FFT only detected seasonality for the $period$ feature, but not for other features that present seasonality.

We use two competitive online learners: \textbf{IncA-DES} \citep{barboza2025}, an online dynamic ensemble selection method, and \textbf{\ac{ARF}} \citep{gomes2017}, an online ensemble. In addition to the online learners, we also include \textbf{TabPFNv2.5} \citep{grinsztajn2026tabpfn25advancingstateart}, a transformer-based tabular classifier.
For TabPFN, we adopt the data-stream setup described by \citet{lourenço2025}, which maintains a sliding context window.
The learners' hyperparameters are reported in Appendix \ref{app:learners-hyperparameters}.
To isolate the effect of causal data augmentation, all experimental conditions are kept identical across both settings, with the only difference being the inclusion of $n$ additional synthetic samples generated by \ac{CaDrift} before starting the evaluation.

In Figure~\ref{fig:preq-acc-elec2-downstream}, we present the windowed prequential accuracy \citep{gama2014survey} without data augmentation and with \ac{CaDrift}-based augmentation for different values of $n$.
IncA-DES (Figure~\ref{subfig:incades-elec}), \ac{ARF} (Figure~\ref{subfig:arf-elec}), and TabPFN (Figure~\ref{subfig:tabpfn-elec}) exhibit improved predictive performance when trained with synthetic data when no delay is applied.
For \ac{ARF}, the prequential accuracy curves eventually converge, as the \ac{HT} base learners are progressively updated with sufficient real data from the ELEC2 stream.
The same happens with TabPFN, as the context window receives more samples from the real ELEC2 dataset.
In contrast, the improvements for IncA-DES persist for longer periods.

However, when we apply delayed feedback to the stream, the performance gain persists for fewer data samples, as we observe in Figure~\ref{fig:preq-acc-elec2-downstream-delay}.
Since the gain in accuracy is more short-term under delayed feedback, we plot a smaller subset of the stream.
For TabPFN, the early gain in performance is not observed -- data augmentation only shows its benefits after more supervised knowledge is observed by the model.
This result suggests that the generator may synthesize samples based on an outdated concept, causing the model to overfit to obsolete cause--effect relationships.
Figure~\ref{fig:preq-acc-elec2-diff} supports this observation: although a large volume of synthetic data boosts early accuracy (within the first 100--200 data points) for ARF and IncA-DES, it eventually hinders adaptation, causing performance to degrade faster than the non-augmented baseline.
Again, TabPFN shows no early accuracy gain after data augmentation under delayed feedback -- benefits come after supervised knowledge.
Consequently, causal augmentation is most beneficial for immediate recovery but requires continuous real-world supervision to prevent long-term reliance on outdated distributions.


\begin{figure}[htb]
    \centering
    \subfloat[IncA-DES (immediate feedback).]{\label{subfig:incades-elec}
    \includegraphics[width=0.45\linewidth]{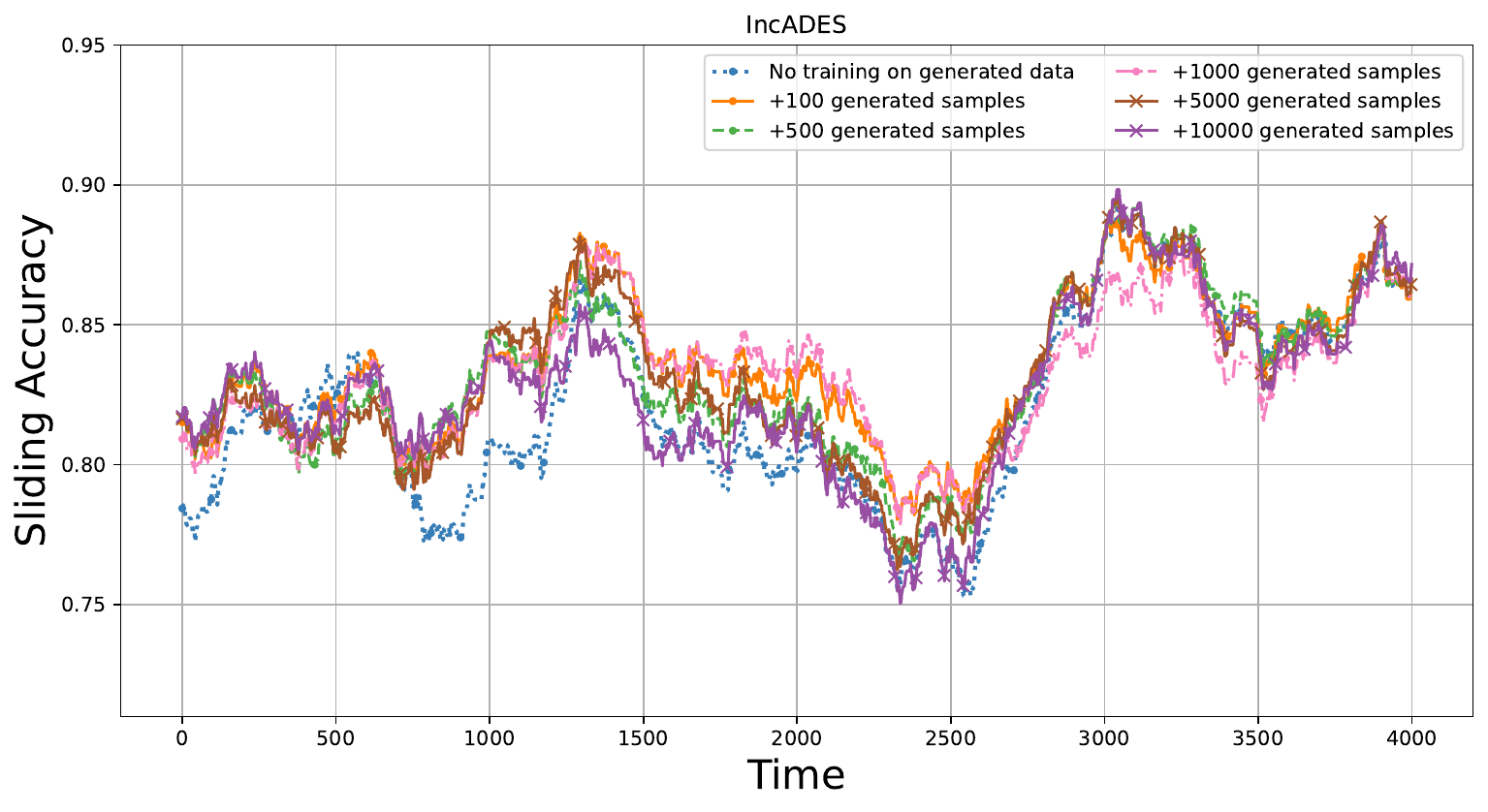}}
    \subfloat[ARF (immediate feedback).]{\label{subfig:arf-elec}
    \includegraphics[width=0.45\linewidth]{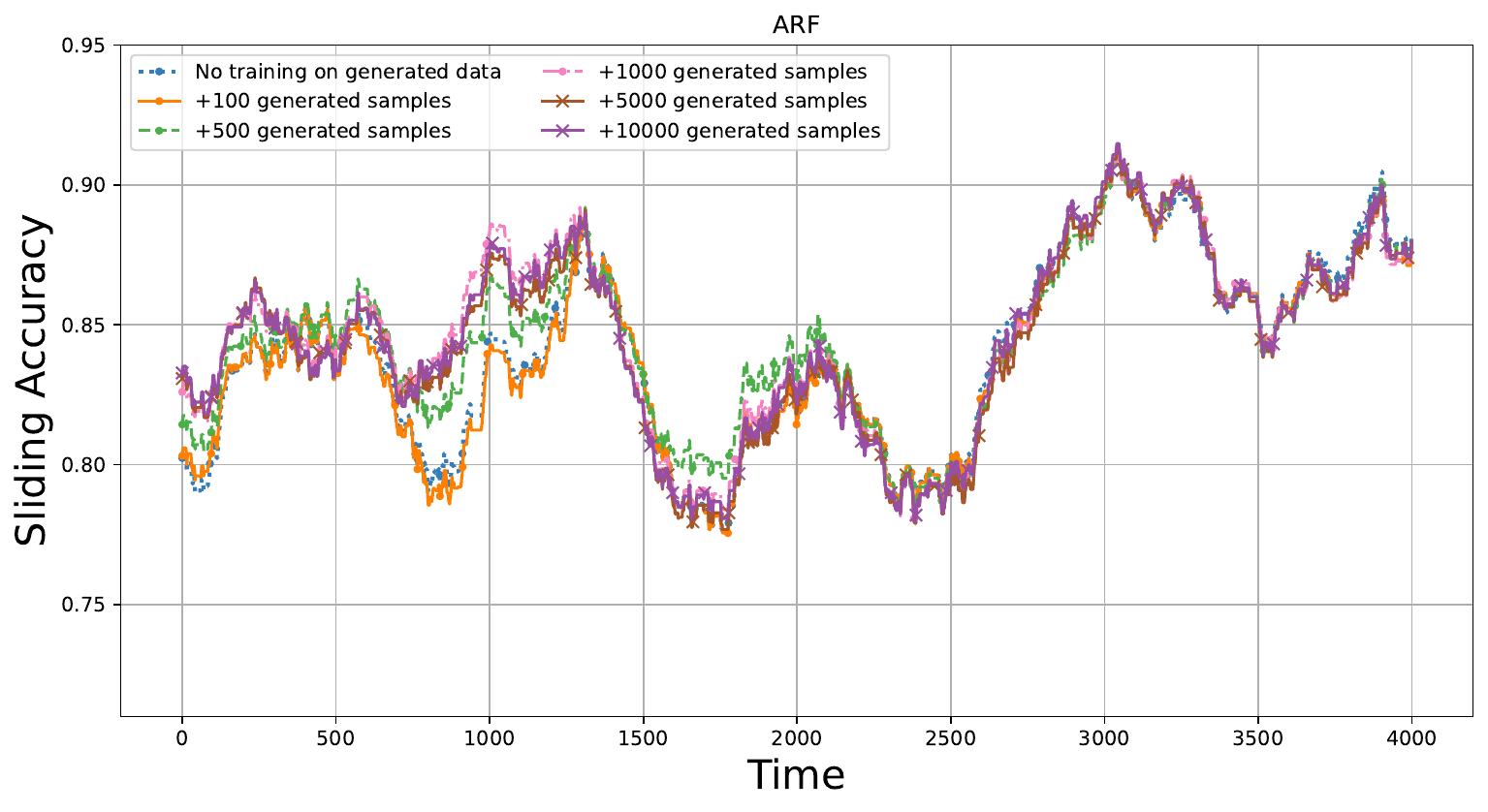}}

    \subfloat[TabPFN (immediate feedback).]{\label{subfig:tabpfn-elec}
    \includegraphics[width=0.45\linewidth]{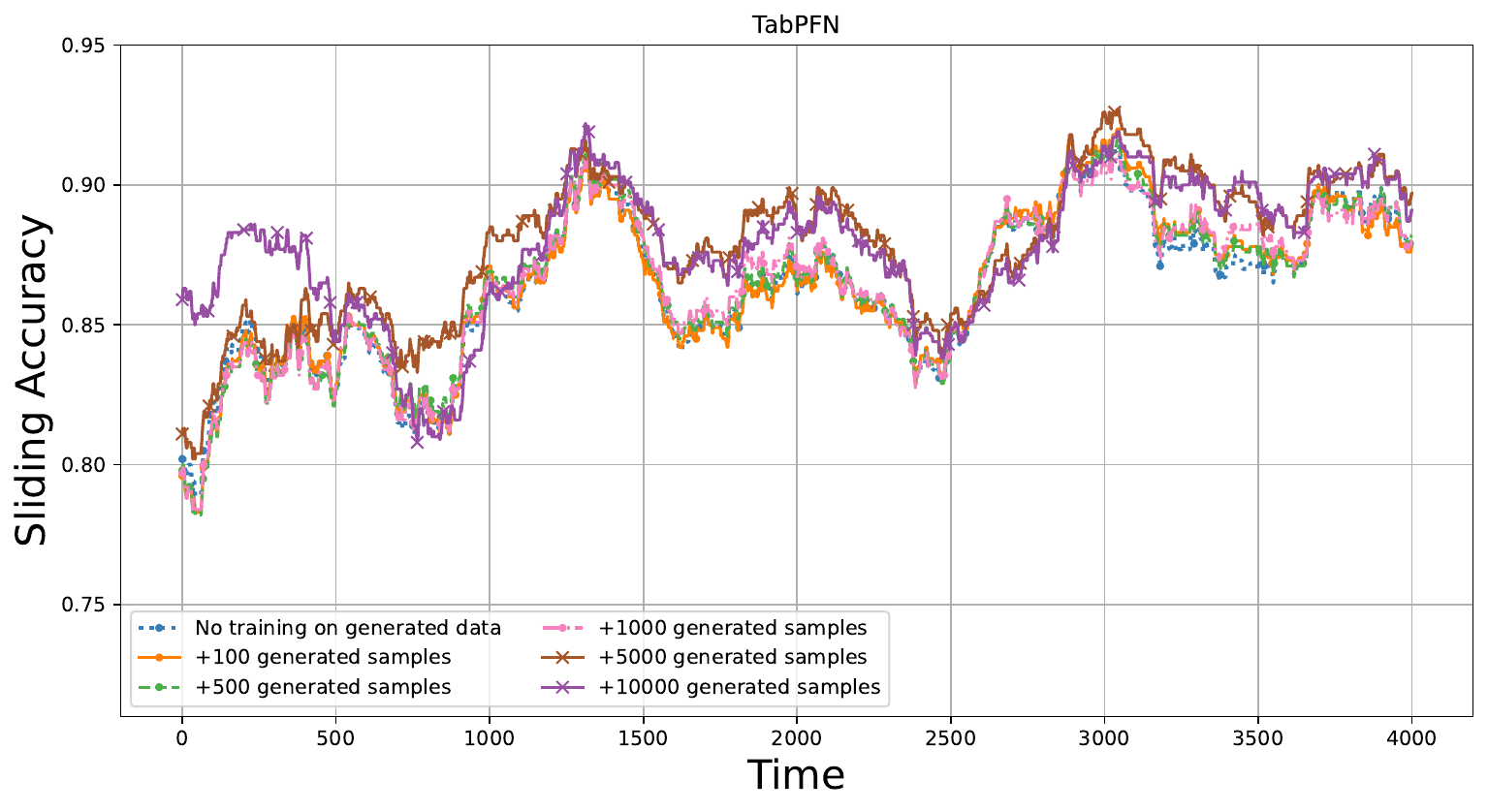}}
    
    \caption{The impact of samples generated by CaDrift on IncA-DES, ARF, and TabPFN (immediate feedback).}
    \label{fig:preq-acc-elec2-downstream}
\end{figure}

\begin{figure}[htb]
    \centering
    \subfloat[IncA-DES (delay).]{\label{subfig:incades-elec-delay}
    \includegraphics[width=0.45\linewidth]{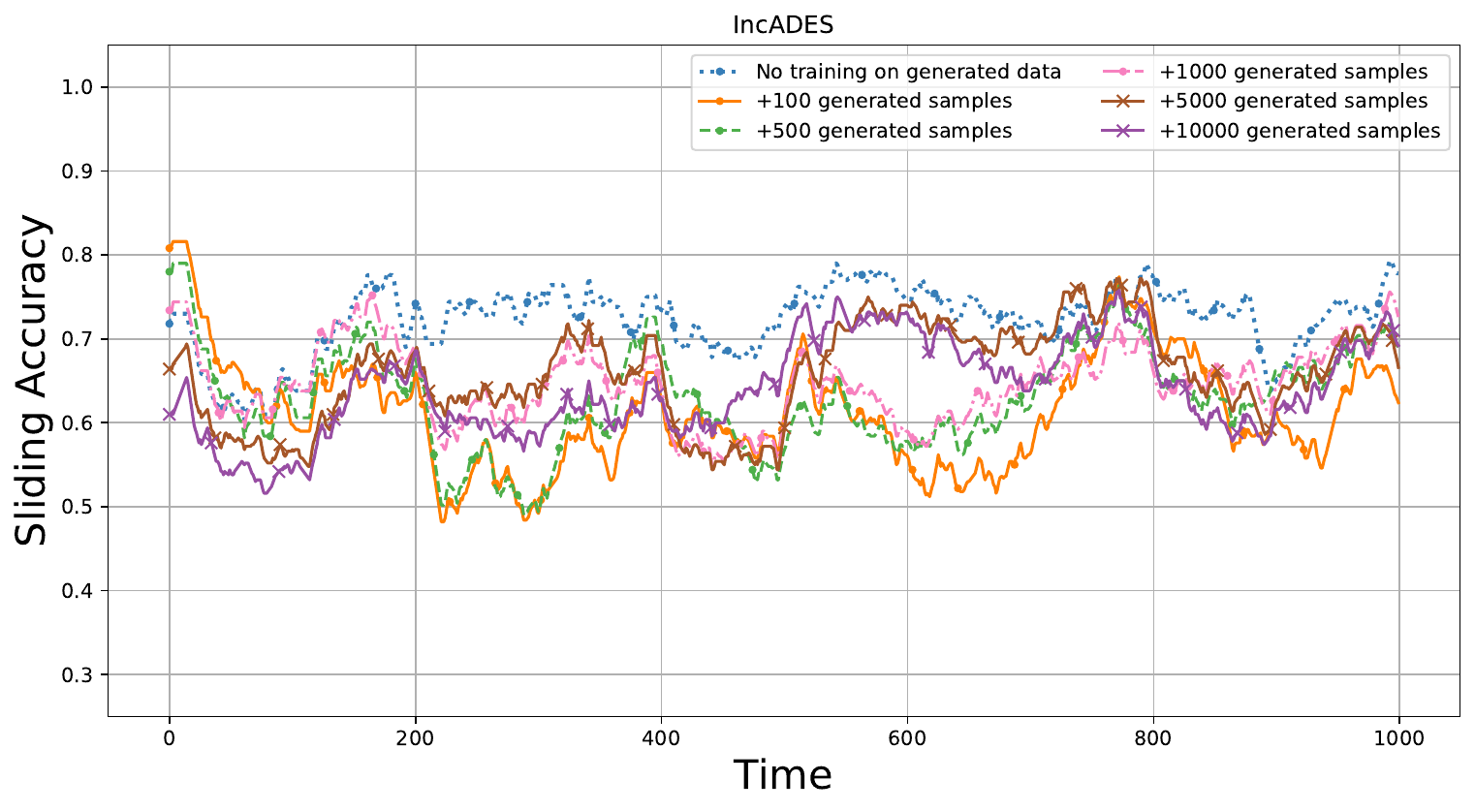}}
    \subfloat[ARF (delay).]{\label{subfig:arf-elec-delay}
    \includegraphics[width=0.45\linewidth]{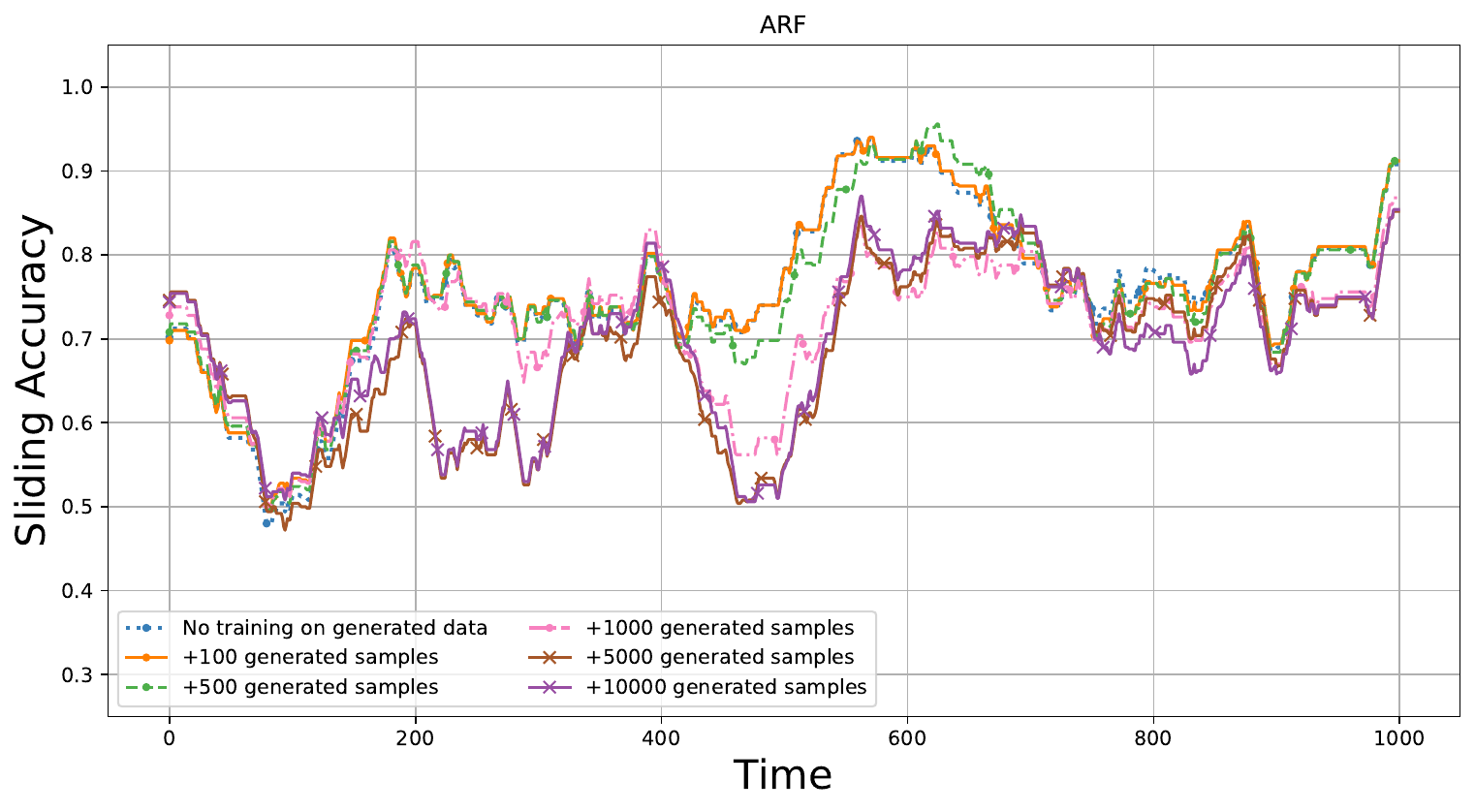}}

    \subfloat[TabPFN (delay).]{\label{subfig:tabpfn-elec-delay}
    \includegraphics[width=0.45\linewidth]{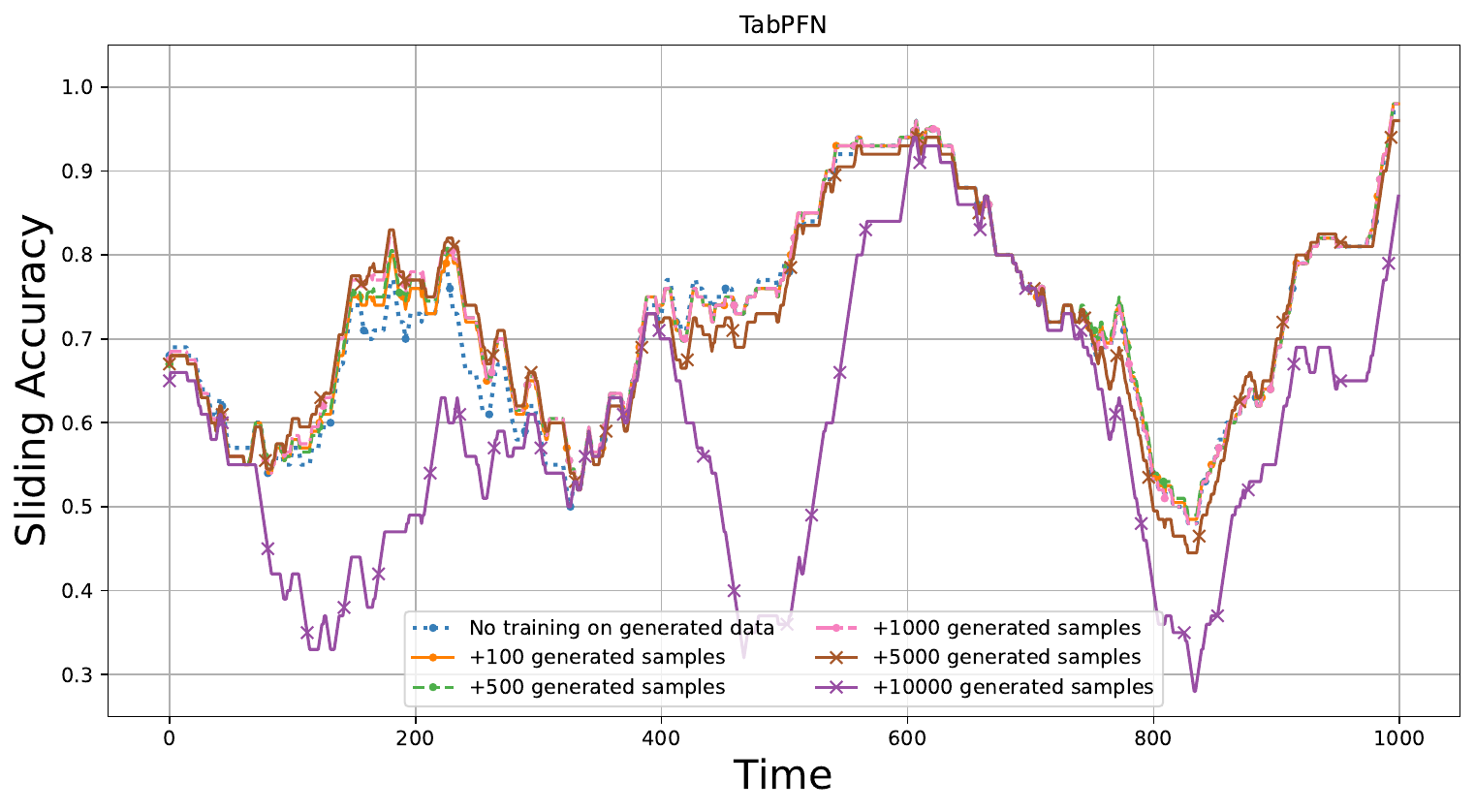}}
    
    \caption{The impact of samples generated by CaDrift on IncA-DES, ARF, and TabPFN (delayed feedback).}
    \label{fig:preq-acc-elec2-downstream-delay}
\end{figure}

\begin{figure}[htb]
    \centering
    \subfloat[IncA-DES.]{\label{subfig:incades-elec-diff}
    \includegraphics[width=0.45\linewidth]{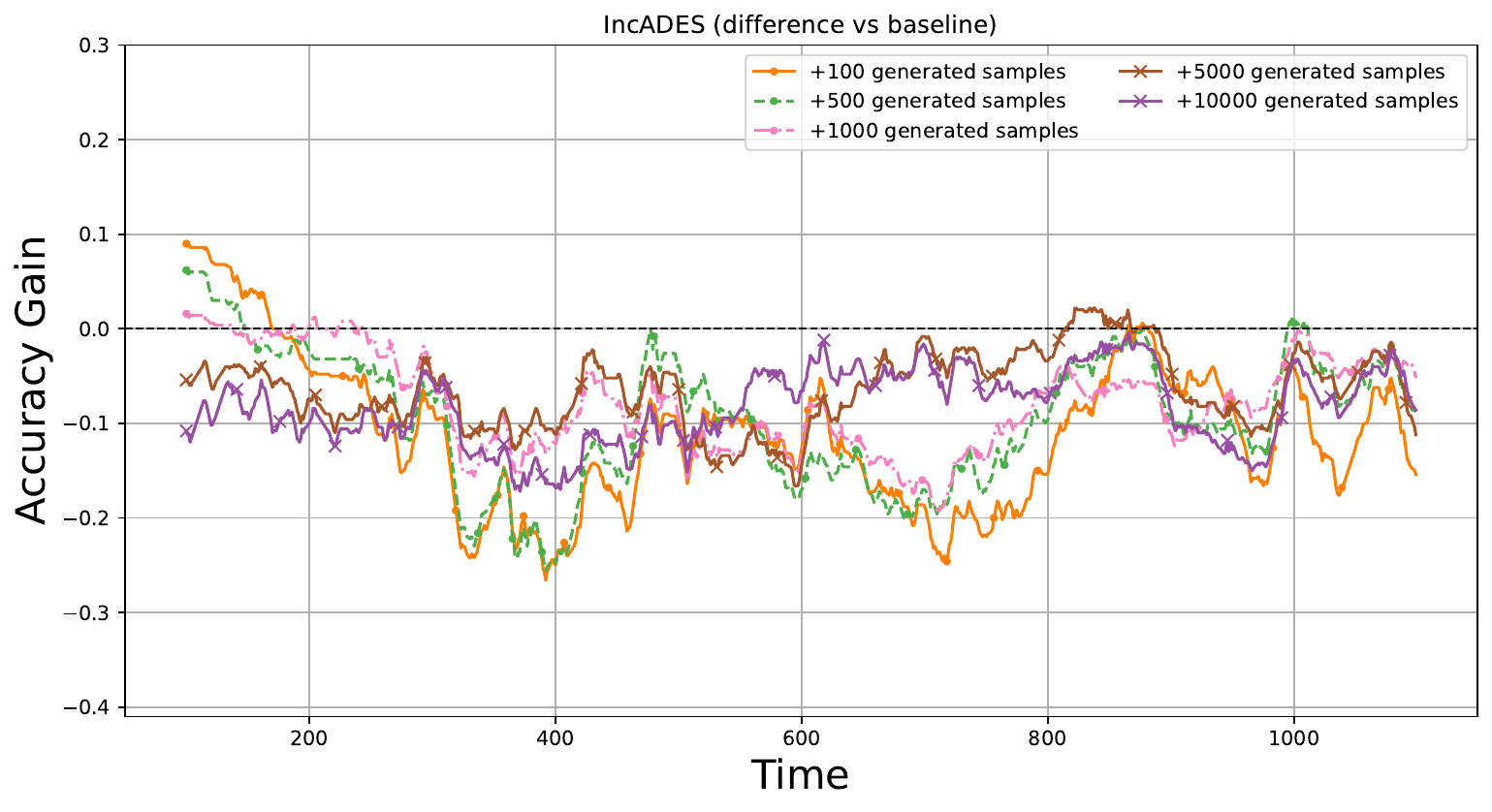}}    
    \subfloat[ARF.]{\label{subfig:arf-elec-diff}
    \includegraphics[width=0.45\linewidth]{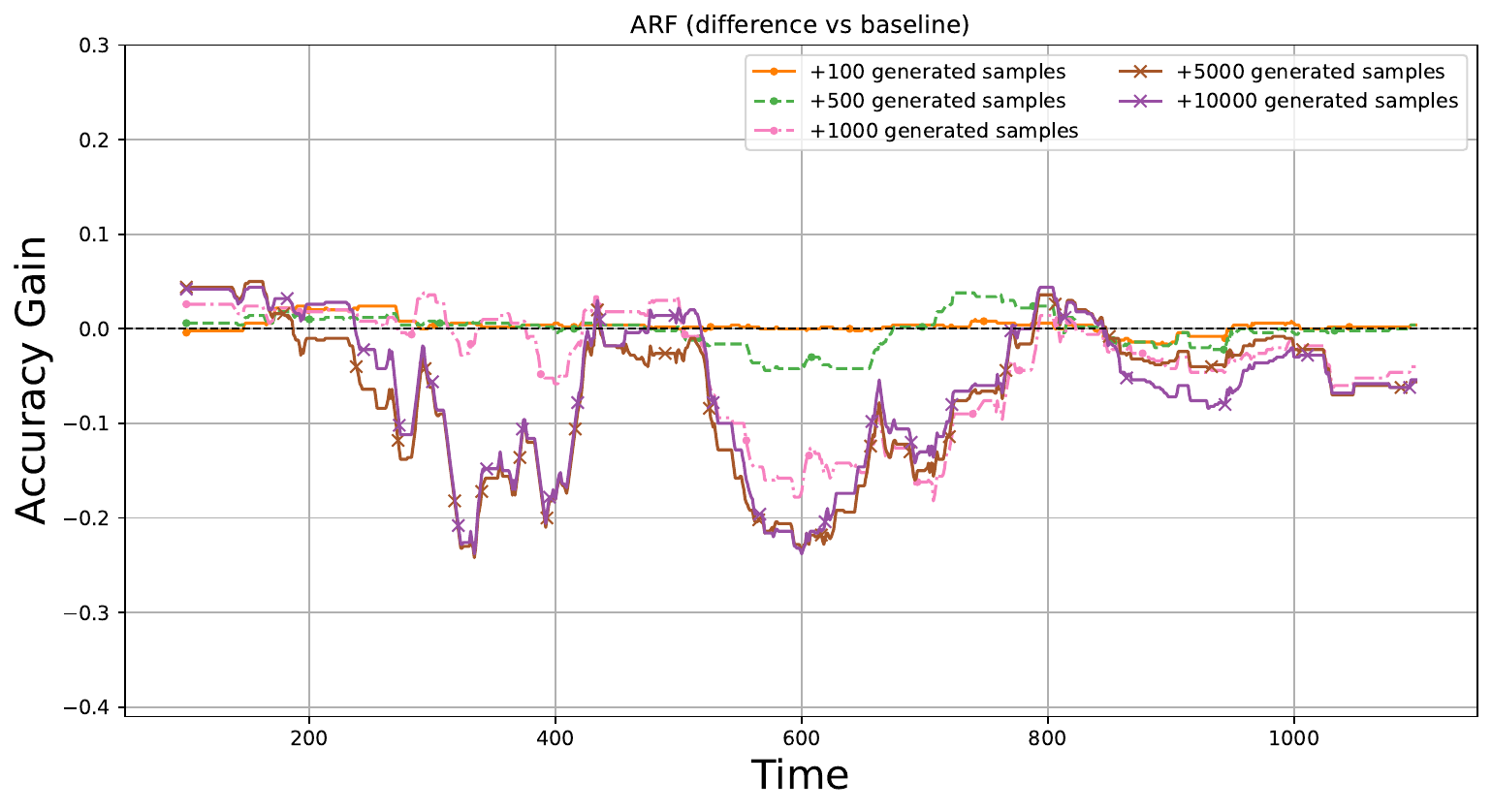}}

    \subfloat[TabPFN.]{\label{subfig:tabpfn-elec-diff}
    \includegraphics[width=0.45\linewidth]{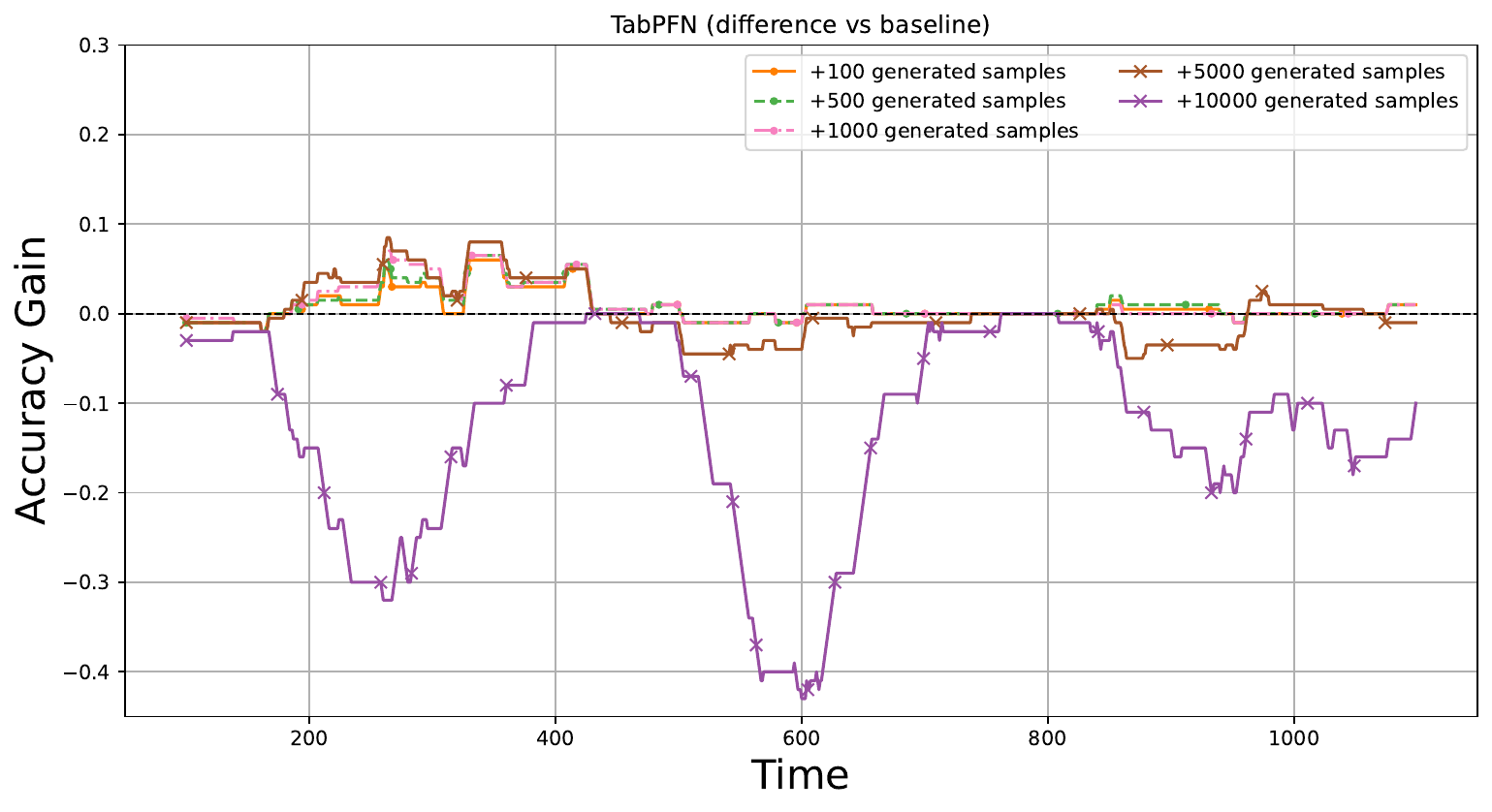}}
    
    \caption{The difference in prequential accuracy of learners when using CaDrift for data augmentation vs. when no data augmentation is used under delayed feedback.}
    \label{fig:preq-acc-elec2-diff}
\end{figure}

In Table~\ref{tab:no-aug-vs-aug}, we present the average accuracy after data augmentation on the 500 subsequent samples (approximately 10 days in the ELEC2 dataset), and 100 subsequent samples on the delayed protocol. 
The improvement in accuracy achieved by data augmentation is evident: under immediate feedback, when $n=10,000$, we observe increases of 3.1 percentage points in accuracy for \ac{ARF}, 3.3 percentage points for IncA-DES when $n=500$ and $n=10,000$, and of 5.5 percentage points for TabPFN when $n=10,000$.
Interestingly, TabPFNv2.5's initial performance when the context window consists only of synthetic samples $n=10,000$ is greater than when using only the original ELEC2 ($n=0$), and the TabPFN trained with augmented synthetic samples seems to surpass the baseline in most of the time stamp (Figure~\ref{subfig:tabpfn-elec}).
When we apply delayed feedback, the average accuracy on the 100 subsequent samples (2 days in the ELEC2 dataset) increases by 4.4 percentage points for \ac{ARF} when $n=5,000$, and 9.0 percentage points for IncA-DES when $n=100$.
TabPFN, however, presents no immediate gain in accuracy under this setup.
If the supervision of recent data samples is not provided, overfitting to a previous concept may degrade the model's performance more rapidly.
Our analysis suggests that online learners can use synthetic data as an initial knowledge base under delayed feedback, and subsequently adapt through incremental updates.
In-context learners rely more directly on the composition of their context window, making context relevance particularly important.

\begin{table}[htb]
    \centering
    \caption{Average accuracy (\%) with and without data augmentation. We report results on the next 500 samples for immediate supervision and on the next 100 samples for delayed supervision. The baseline corresponds to no data augmentation ($n=0$). Values in parentheses correspond to the difference in accuracy compared to the baseline.}
    \begin{tabular}{l c c c c c c}
    \toprule
         Model & \multicolumn{6}{c}{$n$} \\
         \cmidrule(lr){2-7}
         & 0 & 100 & 500 & 1{,}000 & 5{,}000 & 10{,}000 \\
         \midrule
         \textit{test-then-train} \\
         ARF & 80.2 & 80.3\textcolor{green}{\small (+0.1)} & 81.4\textcolor{green}{\small (+1.2)} & 82.6\textcolor{green}{\small (+2.4)} & 83.1\textcolor{green}{\small (+2.9)} & \textbf{83.3}\textcolor{green}{\small (+3.1)} \\
         IncA-DES & 78.4 & 81.5\textcolor{green}{\small (+3.1)} & \textbf{81.7}\textcolor{green}{\small (+3.3)} & 80.9\textcolor{green}{\small (+1.5)} & 81.6\textcolor{green}{\small (+3.2)} & \textbf{81.7}\textcolor{green}{\small (+3.3)} \\
         TabPFNv2.5 & 80.1 & 79.7\textcolor{red}{\small (-0.4)} & 79.8\textcolor{red}{\small (-0.3)} & 79.8\textcolor{red}{\small (-0.3)} & 81.4\textcolor{green}{\small (+1.3)} & \textbf{85.8}\textcolor{green}{\small (+5.5)} \\
         \midrule
         \textit{delay} \\
         ARF & 70.2 & 69.8\textcolor{red}{\small (-0.4)} & 70.8\textcolor{green}{\small (+0.6)} & 72.8\textcolor{green}{\small (+2.6)} & \textbf{74.6}\textcolor{green}{\small (+4.4)} & 74.4\textcolor{green}{\small (+3.6)} \\
         IncA-DES & 71.8 & 80.8\textcolor{green}{\small (+9.0)} & 78.0\textcolor{green}{\small (+6.2)} & 73.4\textcolor{green}{\small (+2.6)} & \textbf{66.4}\textcolor{red}{\small (-5.4)} & 61.0\textcolor{red}{\small (-10.8)} \\
         TabPFNv2.5 & \textbf{68.3} & 67.3\textcolor{red}{\small (-1.0)} & 67.3\textcolor{red}{\small (-1.0)} & 67.6\textcolor{red}{\small (-0.7)} & 67.3\textcolor{red}{\small (-1.3)} & 65.3\textcolor{red}{\small (-3.0)} \\
         \bottomrule
    \end{tabular}
    \label{tab:no-aug-vs-aug}
\end{table}

It is important to note that causal discovery methods such as the CD-NOD algorithm \citep{zhang2017cdnod} used in our experiments rely on assumptions (e.g., causal sufficiency and model specification) and may not recover the true underlying causal structure in all cases. 
As a result, the learned graph should be interpreted as an approximation of the data-generating process. 
Nevertheless, this approach allows us to construct structurally grounded generators that preserve the main dependencies present in the data.
In some applications, human interventions can be done to mitigate potential inaccuracies in the recovered \ac{DAG} structure -- domain expertise is used to validate edges and orient directions that remain ambiguous.

\subsection{Discussion}

\paragraph{Interventions on different nodes produce distinct effects.}
Our results show that causal drift events induce distinct, interpretable effects on both the data distributions and model performance.
Moreover, the combination of marginal and conditional measures proved essential for disentangling different drift types, as several scenarios were indistinguishable when considered from a single perspective.
\emph{Confounder drift} demonstrates that accuracy degradation does not necessarily arise from changes in $P(y\mid \mathbf{x})$, but can instead be driven by shifts in spurious associations between the target and observed variables.

\paragraph{Causal descendants create predictive shortcuts for models.} Furthermore, the presence of direct causal descendants can mask the impact of \emph{target drift} on accuracy. 
In such cases, models exploit these descendants as predictive shortcuts rather than learning the underlying causal mechanism. 
While this can yield strong performance under stationary conditions, it leads to performance degradation when the relationship between the target and its descendants changes (as in Figure \ref{subfig:mmd-kl-acc-structural}).
As the presence of such causal descendants strongly separates the target variable, the \ac{HT} tends to rely on them when constructing its splits.
Consequently, it does not recover the underlying causal relationship that generated the data, but instead exploits more predictive spurious dependencies that are easier to capture.

\paragraph{Causal data augmentation enhances subsequent performance.} In the case study in Section \ref{subsec:case-study-elec2}, models trained on samples generated by \ac{CaDrift} using a graph inferred from the ELEC2 dataset improved the accuracy of subsequent data samples in the stream.
This improvement is particularly evident in the early stages of learning, especially a short-term gain under delayed feedback.
This suggests that \ac{CaDrift} can be used for post-drift data augmentation.
By generating samples that respect the underlying causal structure, \ac{CaDrift} provides informative training instances that accelerate adaptation to new concepts.
However, this augmentation carries the risk of overfitting the model to a previous concept.
We leave further exploration of post-drift data augmentation through \ac{SCM}-based generation for future work.

\paragraph{On the identifiability of causal drift events.} An important consideration comprises the identifiability of causal drift events from observational data.
Although our experiments demonstrate that different causal interventions can produce distinct distributional and predictive signatures, these do not necessarily imply that the underlying causal mechanism that has changed can be identified from observations alone.
For instance, causal sufficiency, faithfulness, and the independent causal mechanism (ICM) assumptions may be required to identify relevant drifting causal structures.
In real-world environments, however, these assumptions are hardly satisfied, especially when confounders are unobserved.
Additional sources of information, such as instrumental variables \cite{angrist1996identification} or interventional data, are promising tools when pursuing causal explanations of distributional changes.
Developing strategies that leverage such information to localize drift in data streams is an interesting prospect for future work.

Overall, our results highlight that the causal structure of the data plays a central role in determining both the detectability of drift and its impact on stream learners. 
In particular, confounding and causal descendants can obscure or amplify the observable effects of drift, influencing both statistical measures and predictive behavior. 
This underlying complexity underscores the necessity of the proposed causal taxonomy. 
Without a principled framework to classify drift based on distinct structural interventions, the community lacks the vocabulary to distinguish between superficial symptoms -- such as shifts in spurious correlations -- and actual changes in the data-generating mechanisms. 
Ultimately, these findings reinforce that robust stream learning requires not merely detecting \emph{when} distributions change, but systematically diagnosing \emph{why} they change -- a paradigm shift for which our taxonomy provides the essential foundation.

\section{Conclusion}

In this work, we introduced a causal taxonomy of concept drift, moving beyond purely probabilistic characterizations of distribution shift. 
By grounding our framework in \acp{SCM}, we categorize drift events according to their causal origins, including exogenous, endogenous, target, confounder, and structural drift.

To operationalize this perspective, we proposed \ac{CaDrift}, an \ac{SCM}-based data stream generator that simulates controlled, mechanism-level drift events while incorporating realistic temporal dynamics, such as autoregressive noise and seasonality.

Our empirical analysis shows that drift types with different causal origins induce distinct patterns of change in marginal and conditional distributions, leading to qualitatively different impacts on predictive performance.
For instance, target drift can degrade accuracy without altering covariate distributions, whereas confounder drift changes observed associations while preserving underlying causal effects.
Distinguishing between these causal signatures provides a rigorous foundation for assessing model reliability beyond simple performance metrics. 
This insight enables researchers to move beyond ``black-box'' drift detection and develop more interpretable diagnostic tools that explain why a model is failing, a critical requirement for deploying machine learning systems in evolving domains.


By integrating causal discovery algorithms to synthesize real-world data streams, \ac{CaDrift} bridges the gap between causal theory and practical evaluation in non-stationary environments, even under limited or delayed feedback. 
This work provides a foundation for the development of causally-aware approaches that can more effectively reason about and adapt to the underlying mechanisms of change in data streams.


\begin{thebibliography}{60}
\providecommand{\natexlab}[1]{#1}
\providecommand{\url}[1]{\texttt{#1}}
\expandafter\ifx\csname urlstyle\endcsname\relax
  \providecommand{\doi}[1]{doi: #1}\else
  \providecommand{\doi}{doi: \begingroup \urlstyle{rm}\Url}\fi

\bibitem[Agrahari \& Singh(2022)Agrahari and Singh]{agrahari2022}
Supriya Agrahari and Anil~Kumar Singh.
\newblock Concept drift detection in data stream mining : A literature review.
\newblock \emph{Journal of King Saud University - Computer and Information Sciences}, 34\penalty0 (10, Part B):\penalty0 9523--9540, 2022.
\newblock ISSN 1319-1578.

\bibitem[Agrahari \& Singh(2024)Agrahari and Singh]{Agrahari2024}
Supriya Agrahari and Anil~Kumar Singh.
\newblock Comparison based analysis of window approach for concept drift detection and adaptation.
\newblock \emph{Applied Intelligence}, 55\penalty0 (1):\penalty0 39, Nov 2024.
\newblock ISSN 1573-7497.

\bibitem[Angrist et~al.(1996)Angrist, Imbens, and Rubin]{angrist1996identification}
Joshua~D Angrist, Guido~W Imbens, and Donald~B Rubin.
\newblock Identification of causal effects using instrumental variables.
\newblock \emph{Journal of the American statistical Association}, 91\penalty0 (434):\penalty0 444--455, 1996.

\bibitem[Arjovsky et~al.(2020)Arjovsky, Bottou, Gulrajani, and Lopez-Paz]{arjovsky2020invariantriskminimization}
Martin Arjovsky, Léon Bottou, Ishaan Gulrajani, and David Lopez-Paz.
\newblock Invariant risk minimization, 2020.

\bibitem[Baier et~al.(2022)Baier, Schlör, Schöffer, and Kühl]{baier2021}
Lucas Baier, Tim Schlör, Jakob Schöffer, and Niklas Kühl.
\newblock Detecting concept drift with neural network model uncertainty, 2022.

\bibitem[Barboza et~al.(2025)Barboza, de~Almeida, de~Souza~Britto, Sabourin, and Cruz]{barboza2025}
Eduardo~V.L. Barboza, Paulo R.~Lisboa de~Almeida, Alceu de~Souza~Britto, Robert Sabourin, and Rafael~M.O. Cruz.
\newblock Inca-des: An incremental and adaptive dynamic ensemble selection approach using online k-d tree neighborhood search for data streams with concept drift.
\newblock \emph{Information Fusion}, 123:\penalty0 103272, 2025.
\newblock ISSN 1566-2535.

\bibitem[Bengio et~al.(2019)Bengio, Deleu, Rahaman, Ke, Lachapelle, Bilaniuk, Goyal, and Pal]{bengio2019metatransfer}
Yoshua Bengio, Tristan Deleu, Nasim Rahaman, Rosemary Ke, Sébastien Lachapelle, Olexa Bilaniuk, Anirudh Goyal, and Christopher Pal.
\newblock A meta-transfer objective for learning to disentangle causal mechanisms, 2019.

\bibitem[Bifet \& Gavald{\`a}(2007)Bifet and Gavald{\`a}]{Bifet2007}
Albert Bifet and Ricard Gavald{\`a}.
\newblock Learning from time-changing data with adaptive windowing.
\newblock In \emph{SDM}, pp.\  443--448, 2007.

\bibitem[Bifet et~al.(2009)Bifet, Holmes, Pfahringer, Kirkby, and Gavaldà]{bifet2009a}
Albert Bifet, Geoffrey Holmes, Bernhard Pfahringer, Richard Kirkby, and Ricard Gavaldà.
\newblock New ensemble methods for evolving data streams.
\newblock \emph{Proceedings of the ACM SIGKDD International Conference on Knowledge Discovery and Data Mining}, pp.\  139–148, 06 2009.

\bibitem[Box et~al.(2015)Box, Jenkins, Reinsel, and Ljung]{box2015time}
George~EP Box, Gwilym~M Jenkins, Gregory~C Reinsel, and Greta~M Ljung.
\newblock \emph{Time series analysis: forecasting and control}.
\newblock John Wiley \& Sons, 2015.

\bibitem[Chakrabarty \& Biswas(2018)Chakrabarty and Biswas]{Chakrabarty2018}
Navoneel Chakrabarty and Sanket Biswas.
\newblock A statistical approach to adult census income level prediction.
\newblock In \emph{2018 International Conference on Advances in Computing, Communication Control and Networking (ICACCCN)}, pp.\  207--212, 2018.

\bibitem[Chen et~al.(2024)Chen, Bello, Locatello, Aragam, and Ravikumar]{Chen_Neurips2024}
Tianyu Chen, Kevin Bello, Francesco Locatello, Bryon Aragam, and Pradeep Ravikumar.
\newblock Identifying general mechanism shifts in linear causal representations.
\newblock In A.~Globerson, L.~Mackey, D.~Belgrave, A.~Fan, U.~Paquet, J.~Tomczak, and C.~Zhang (eds.), \emph{Advances in Neural Information Processing Systems}, volume~37, pp.\  42405--42429. Curran Associates, Inc., 2024.

\bibitem[Domingos \& Hulten(2000)Domingos and Hulten]{domingos2000}
Pedro Domingos and Geoff Hulten.
\newblock Mining high-speed data streams.
\newblock \emph{Association for Computing Machinery}, pp.\  71–80, 2000.

\bibitem[Eastwood et~al.(2022)Eastwood, Robey, Singh, von K\"{u}gelgen, Hassani, Pappas, and Sch\"{o}lkopf]{eastwood2022QRM}
Cian Eastwood, Alexander Robey, Shashank Singh, Julius von K\"{u}gelgen, Hamed Hassani, George~J. Pappas, and Bernhard Sch\"{o}lkopf.
\newblock Probable domain generalization via quantile risk minimization.
\newblock In S.~Koyejo, S.~Mohamed, A.~Agarwal, D.~Belgrave, K.~Cho, and A.~Oh (eds.), \emph{Advances in Neural Information Processing Systems}, volume~35, pp.\  17340--17358. Curran Associates, Inc., 2022.

\bibitem[English \& Gaur(2010)English and Gaur]{English2010}
B.~Keith English and Aditya~H. Gaur.
\newblock \emph{The Use and Abuse of Antibiotics and the Development of Antibiotic Resistance}, pp.\  73--82.
\newblock Springer New York, New York, NY, 2010.
\newblock ISBN 978-1-4419-0981-7.

\bibitem[Fujiwara et~al.(2023)Fujiwara, Koyama, Kiritoshi, Okawachi, Izumitani, and Shimizu]{fujiwara23ajitlingam}
Daigo Fujiwara, Kazuki Koyama, Keisuke Kiritoshi, Tomomi Okawachi, Tomonori Izumitani, and Shohei Shimizu.
\newblock Causal discovery for non-stationary non-linear time series data using just-in-time modeling.
\newblock In Mihaela van~der Schaar, Cheng Zhang, and Dominik Janzing (eds.), \emph{Proceedings of the Second Conference on Causal Learning and Reasoning}, volume 213 of \emph{Proceedings of Machine Learning Research}, pp.\  880--894. PMLR, 11--14 Apr 2023.

\bibitem[Gama et~al.(2004)Gama, Medas, Castillo, and Rodrigues]{gama2004}
Jo{\~a}o Gama, Pedro Medas, Gladys Castillo, and Pedro Rodrigues.
\newblock Learning with drift detection.
\newblock In Ana L.~C. Bazzan and Sofiane Labidi (eds.), \emph{Advances in Artificial Intelligence -- SBIA 2004}, pp.\  286--295, Berlin, Heidelberg, 2004. Springer Berlin Heidelberg.
\newblock ISBN 978-3-540-28645-5.

\bibitem[Gama et~al.(2014)Gama, {\v{Z}}liobait{\.e}, Bifet, Pechenizkiy, and Bouchachia]{gama2014survey}
Jo{\~a}o Gama, Indr{\.e} {\v{Z}}liobait{\.e}, Albert Bifet, Mykola Pechenizkiy, and Abdelhamid Bouchachia.
\newblock A survey on concept drift adaptation.
\newblock \emph{ACM computing surveys (CSUR)}, 46\penalty0 (4):\penalty0 1--37, 2014.

\bibitem[Gao et~al.(2023)Gao, Addanki, Yu, Rossi, and Kocaoglu]{gao2023pcmciomega}
Shanyun Gao, Raghavendra Addanki, Tong Yu, Ryan Rossi, and Murat Kocaoglu.
\newblock Causal discovery in semi-stationary time series.
\newblock In A.~Oh, T.~Naumann, A.~Globerson, K.~Saenko, M.~Hardt, and S.~Levine (eds.), \emph{Advances in Neural Information Processing Systems}, volume~36, pp.\  46624--46657. Curran Associates, Inc., 2023.
\newblock URL \url{https://proceedings.neurips.cc/paper_files/paper/2023/file/91f9fb16b5679115a777ade51af87e48-Paper-Conference.pdf}.

\bibitem[Gao et~al.(2025)Gao, Addanki, Yu, Rossi, and Kocaoglu]{gao2025causaldiscoverydrivenchangepoint}
Shanyun Gao, Raghavendra Addanki, Tong Yu, Ryan~A. Rossi, and Murat Kocaoglu.
\newblock Causal discovery-driven change point detection in time series, 2025.

\bibitem[Gomes et~al.(2017)Gomes, Bifet, Read, Barddal, Enembreck, Pfahringer, Holmes, and Abdessalem]{gomes2017}
Heitor~Murilo Gomes, Albert Bifet, Jesse Read, Jean~Paul Barddal, Fabrício Enembreck, Bernhard Pfahringer, Geoff Holmes, and Talel Abdessalem.
\newblock Adaptive random forests for evolving data stream classification.
\newblock \emph{Machine Learning}, 106:\penalty0 1--27, 10 2017.

\bibitem[Gower-Winter et~al.(2026)Gower-Winter, Groen, and Krempl]{gowerwinter2026windowdilemmaconceptdrift}
Brandon Gower-Winter, Misja Groen, and Georg Krempl.
\newblock The window dilemma: Why concept drift detection is ill-posed, 2026.

\bibitem[Gretton et~al.(2012)Gretton, Borgwardt, Rasch, Sch{{\"o}}lkopf, and Smola]{gretton12ammd}
Arthur Gretton, Karsten~M. Borgwardt, Malte~J. Rasch, Bernhard Sch{{\"o}}lkopf, and Alexander Smola.
\newblock A kernel two-sample test.
\newblock \emph{Journal of Machine Learning Research}, 13\penalty0 (25):\penalty0 723--773, 2012.

\bibitem[Grinsztajn et~al.(2026)Grinsztajn, Flöge, Key, Birkel, Jund, Roof, Jäger, Safaric, Alessi, Hayler, Manium, Yu, Jablonski, Hoo, Garg, Robertson, Bühler, Moroshan, Purucker, Cornu, Wehrhahn, Bonetto, Schölkopf, Gambhir, Hollmann, and Hutter]{grinsztajn2026tabpfn25advancingstateart}
Léo Grinsztajn, Klemens Flöge, Oscar Key, Felix Birkel, Philipp Jund, Brendan Roof, Benjamin Jäger, Dominik Safaric, Simone Alessi, Adrian Hayler, Mihir Manium, Rosen Yu, Felix Jablonski, Shi~Bin Hoo, Anurag Garg, Jake Robertson, Magnus Bühler, Vladyslav Moroshan, Lennart Purucker, Clara Cornu, Lilly~Charlotte Wehrhahn, Alessandro Bonetto, Bernhard Schölkopf, Sauraj Gambhir, Noah Hollmann, and Frank Hutter.
\newblock Tabpfn-2.5: Advancing the state of the art in tabular foundation models, 2026.

\bibitem[G\"{u}nther et~al.(2023)G\"{u}nther, Ninad, and Runge]{gunther23ajpcmci}
Wiebke G\"{u}nther, Urmi Ninad, and Jakob Runge.
\newblock Causal discovery for time series from multiple datasets with latent contexts.
\newblock In Robin~J. Evans and Ilya Shpitser (eds.), \emph{Proceedings of the Thirty-Ninth Conference on Uncertainty in Artificial Intelligence}, volume 216 of \emph{Proceedings of Machine Learning Research}, pp.\  766--776. PMLR, 31 Jul--04 Aug 2023.

\bibitem[Harries et~al.(1999)Harries, Wales, et~al.]{harries1999}
Michael Harries, New~South Wales, et~al.
\newblock Splice-2 comparative evaluation: Electricity pricing.
\newblock \emph{University of New South Wales, School of Computer Science and Engineering}, 1999.

\bibitem[Hernandez~Aros et~al.(2024)Hernandez~Aros, Bustamante~Molano, Gutierrez-Portela, Moreno~Hernandez, and Rodr{\'i}guez~Barrero]{Hernandez2024}
Ludivia Hernandez~Aros, Luisa~Ximena Bustamante~Molano, Fernando Gutierrez-Portela, John~Johver Moreno~Hernandez, and Mario~Samuel Rodr{\'i}guez~Barrero.
\newblock Financial fraud detection through the application of machine learning techniques: a literature review.
\newblock \emph{Humanities and Social Sciences Communications}, 11\penalty0 (1):\penalty0 1130, Sep 2024.
\newblock ISSN 2662-9992.

\bibitem[Hinder et~al.(2024)Hinder, Vaquet, and Hammer]{hinder2024a}
Fabian Hinder, Valerie Vaquet, and Barbara Hammer.
\newblock One or two things we know about concept drift—a survey on monitoring in evolving environments. part a: detecting concept drift.
\newblock \emph{Frontiers in Artificial Intelligence}, 7:\penalty0 1330257, 2024.
\newblock ISSN 2624-8212.

\bibitem[Hollmann et~al.(2025)Hollmann, M{\"u}ller, Purucker, Krishnakumar, K{\"o}rfer, Hoo, Schirrmeister, and Hutter]{Hollmann2025}
Noah Hollmann, Samuel M{\"u}ller, Lennart Purucker, Arjun Krishnakumar, Max K{\"o}rfer, Shi~Bin Hoo, Robin~Tibor Schirrmeister, and Frank Hutter.
\newblock Accurate predictions on small data with a tabular foundation model.
\newblock \emph{Nature}, 637\penalty0 (8045):\penalty0 319--326, Jan 2025.
\newblock ISSN 1476-4687.

\bibitem[Hulten et~al.(2001)Hulten, Spencer, and Domingos]{hulten2001}
Geoff Hulten, Laurie Spencer, and Pedro Domingos.
\newblock Mining time-changing data streams.
\newblock In \emph{Proceedings of the Seventh ACM SIGKDD International Conference on Knowledge Discovery and Data Mining}, KDD '01, pp.\  97–106, New York, NY, USA, 2001. Association for Computing Machinery.
\newblock ISBN 158113391X.

\bibitem[Hyndman \& Athanasopoulos(2018)Hyndman and Athanasopoulos]{hyndman2018forecasting}
Rob~J Hyndman and George Athanasopoulos.
\newblock \emph{Forecasting: principles and practice}.
\newblock OTexts, 2018.

\bibitem[Komnick et~al.(2025)Komnick, Lammers, Hammer, Vaquet, and Hinder]{komnick2025causalexplanationconceptdrift}
David Komnick, Kathrin Lammers, Barbara Hammer, Valerie Vaquet, and Fabian Hinder.
\newblock Causal explanation of concept drift -- a truly actionable approach, 2025.

\bibitem[Komorniczak(2025)]{komorniczak2025}
Joanna Komorniczak.
\newblock Synthetic non-stationary data streams for recognition of the unknown, 2025.

\bibitem[Krueger et~al.(2021)Krueger, Caballero, Jacobsen, Zhang, Binas, Zhang, Priol, and Courville]{krueger21a_rex}
David Krueger, Ethan Caballero, Joern-Henrik Jacobsen, Amy Zhang, Jonathan Binas, Dinghuai Zhang, Remi~Le Priol, and Aaron Courville.
\newblock Out-of-distribution generalization via risk extrapolation (rex).
\newblock In Marina Meila and Tong Zhang (eds.), \emph{Proceedings of the 38th International Conference on Machine Learning}, volume 139 of \emph{Proceedings of Machine Learning Research}, pp.\  5815--5826. PMLR, 18--24 Jul 2021.

\bibitem[Kurian \& Allali(2024)Kurian and Allali]{Kurian2024}
Jeomoan~Francis Kurian and Mohamed Allali.
\newblock Detecting drifts in data streams using kullback-leibler (kl) divergence measure for data engineering applications.
\newblock \emph{Journal of Data, Information and Management}, 6\penalty0 (3):\penalty0 207--216, Sep 2024.
\newblock ISSN 2524-6364.

\bibitem[Lee \& Lee(2024)Lee and Lee]{lee2024}
Sanghyuk Lee and Eunmi Lee.
\newblock Score function design for decision making using conditional kullback-leibler divergence.
\newblock In \emph{2024 IEEE International Conference on Artificial Intelligence in Engineering and Technology (IICAIET)}, pp.\  216--221, 2024.

\bibitem[Liu \& Kuang(2023)Liu and Kuang]{liu23lin}
Chenxi Liu and Kun Kuang.
\newblock Causal structure learning for latent intervened non-stationary data.
\newblock In Andreas Krause, Emma Brunskill, Kyunghyun Cho, Barbara Engelhardt, Sivan Sabato, and Jonathan Scarlett (eds.), \emph{Proceedings of the 40th International Conference on Machine Learning}, volume 202 of \emph{Proceedings of Machine Learning Research}, pp.\  21756--21777. PMLR, 23--29 Jul 2023.

\bibitem[Ljung \& Box(1978)Ljung and Box]{ljung1978measure}
Greta~M Ljung and George~EP Box.
\newblock On a measure of lack of fit in time series models.
\newblock \emph{Biometrika}, 65\penalty0 (2):\penalty0 297--303, 1978.

\bibitem[Losing et~al.(2016)Losing, Hammer, and Wersing]{losing2016}
Viktor Losing, Barbara Hammer, and Heiko Wersing.
\newblock Knn classifier with self adjusting memory for heterogeneous concept drift.
\newblock In \emph{2016 IEEE 16th International Conference on Data Mining (ICDM)}, pp.\  291--300, 2016.

\bibitem[Lourenço et~al.(2025)Lourenço, Gama, Xing, and Marreiros]{lourenço2025}
Afonso Lourenço, João Gama, Eric~P. Xing, and Goreti Marreiros.
\newblock In-context learning of evolving data streams with tabular foundational models, 2025.

\bibitem[Lu et~al.(2019)Lu, Liu, Dong, Gu, Gama, and Zhang]{lu2019}
Jie Lu, Anjin Liu, Fan Dong, Feng Gu, João Gama, and Guangquan Zhang.
\newblock Learning under concept drift: A review.
\newblock \emph{IEEE Transactions on Knowledge and Data Engineering}, 31\penalty0 (12):\penalty0 2346--2363, 2019.

\bibitem[Lu et~al.(2025)Lu, Lu, Liu, and Zhang]{Lu_Lu_Liu_Zhang_2025}
Pengqian Lu, Jie Lu, Anjin Liu, and Guangquan Zhang.
\newblock Early concept drift detection via prediction uncertainty.
\newblock \emph{Proceedings of the AAAI Conference on Artificial Intelligence}, 39\penalty0 (18):\penalty0 19124--19132, Apr. 2025.

\bibitem[Musbah et~al.(2019)Musbah, El-Hawary, and Aly]{musbah2019fft}
Hmeda Musbah, Mo~El-Hawary, and Hamed Aly.
\newblock Identifying seasonality in time series by applying fast fourier transform.
\newblock In \emph{2019 IEEE Electrical Power and Energy Conference (EPEC)}, pp.\  1--4, 2019.

\bibitem[Paim \& Enembreck(2025)Paim and Enembreck]{paim2025}
Aldo~M. Paim and Fabrício Enembreck.
\newblock Adaptive random tree ensemble for evolving data stream classification.
\newblock \emph{Knowledge-Based Systems}, 309:\penalty0 112830, 2025.
\newblock ISSN 0950-7051.

\bibitem[Pearl(2009)]{pearl2009causality}
Judea Pearl.
\newblock \emph{Causality}.
\newblock Cambridge university press, 2009.

\bibitem[Pesaranghader et~al.(2018)Pesaranghader, Viktor, and Paquet]{pesaranghader2018mcdiarmid}
Ali Pesaranghader, Herna Viktor, and Eric Paquet.
\newblock Mcdiarmid drift detection methods for evolving data streams, 2018.

\bibitem[Peters et~al.(2016)Peters, Bühlmann, and Meinshausen]{peters2016icp}
Jonas Peters, Peter Bühlmann, and Nicolai Meinshausen.
\newblock Causal inference by using invariant prediction: Identification and confidence intervals.
\newblock \emph{Journal of the Royal Statistical Society Series B: Statistical Methodology}, 78\penalty0 (5):\penalty0 947--1012, 11 2016.
\newblock ISSN 1369-7412.

\bibitem[Peters et~al.(2017)Peters, Janzing, and Sch{\"o}lkopf]{peters2017elements}
Jonas Peters, Dominik Janzing, and Bernhard Sch{\"o}lkopf.
\newblock \emph{Elements of causal inference: foundations and learning algorithms}.
\newblock The MIT Press, 2017.

\bibitem[Roberts(1959)]{roberts1959}
S.~W. Roberts.
\newblock Control chart tests based on geometric moving averages.
\newblock \emph{Technometrics}, 1\penalty0 (3):\penalty0 239--250, 1959.
\newblock ISSN 00401706.

\bibitem[Rosenfeld et~al.(2021)Rosenfeld, Ravikumar, and Risteski]{rosenfeld2021risksinvariantriskminimization}
Elan Rosenfeld, Pradeep Ravikumar, and Andrej Risteski.
\newblock The risks of invariant risk minimization, 2021.

\bibitem[Runge et~al.(2019{\natexlab{a}})Runge, Bathiany, Bollt, Camps-Valls, Coumou, Deyle, Glymour, Kretschmer, Mahecha, Mu{\~{n}}oz-Mar{\'i}, van Nes, Peters, Quax, Reichstein, Scheffer, Sch{\"o}lkopf, Spirtes, Sugihara, Sun, Zhang, and Zscheischler]{runge2019causeme}
Jakob Runge, Sebastian Bathiany, Erik Bollt, Gustau Camps-Valls, Dim Coumou, Ethan Deyle, Clark Glymour, Marlene Kretschmer, Miguel~D. Mahecha, Jordi Mu{\~{n}}oz-Mar{\'i}, Egbert~H. van Nes, Jonas Peters, Rick Quax, Markus Reichstein, Marten Scheffer, Bernhard Sch{\"o}lkopf, Peter Spirtes, George Sugihara, Jie Sun, Kun Zhang, and Jakob Zscheischler.
\newblock Inferring causation from time series in earth system sciences.
\newblock \emph{Nature Communications}, 10\penalty0 (1):\penalty0 2553, Jun 2019{\natexlab{a}}.
\newblock ISSN 2041-1723.

\bibitem[Runge et~al.(2019{\natexlab{b}})Runge, Nowack, Kretschmer, Flaxman, and Sejdinovic]{runge2019pcmci}
Jakob Runge, Peer Nowack, Marlene Kretschmer, Seth Flaxman, and Dino Sejdinovic.
\newblock Detecting and quantifying causal associations in large nonlinear time series datasets.
\newblock \emph{Science Advances}, 5\penalty0 (11):\penalty0 eaau4996, 2019{\natexlab{b}}.

\bibitem[Schölkopf et~al.(2021)Schölkopf, Locatello, Bauer, Ke, Kalchbrenner, Goyal, and Bengio]{scholkopf2021towardCRL}
Bernhard Schölkopf, Francesco Locatello, Stefan Bauer, Nan~Rosemary Ke, Nal Kalchbrenner, Anirudh Goyal, and Yoshua Bengio.
\newblock Toward causal representation learning.
\newblock \emph{Proceedings of the IEEE}, 109\penalty0 (5):\penalty0 612--634, 2021.

\bibitem[Sharma \& Kiciman(2020)Sharma and Kiciman]{sharma2020dowhy}
Amit Sharma and Emre Kiciman.
\newblock Dowhy: An end-to-end library for causal inference.
\newblock \emph{arXiv preprint arXiv:2011.04216}, 2020.

\bibitem[Street \& Kim(2001)Street and Kim]{street2001}
W.~Nick Street and YongSeog Kim.
\newblock A streaming ensemble algorithm (sea) for large-scale classification.
\newblock In \emph{Proceedings of the Seventh ACM SIGKDD International Conference on Knowledge Discovery and Data Mining}, KDD '01, pp.\  377–382, New York, NY, USA, 2001. Association for Computing Machinery.
\newblock ISBN 158113391X.

\bibitem[Xie et~al.(2026)Xie, Feofanov, Zhang, Palpanas, and Redko]{xie2026cauker}
Shifeng Xie, Vasilii Feofanov, Jianfeng Zhang, Themis Palpanas, and Ievgen Redko.
\newblock Cauker: Classification time series foundation models can be pretrained on synthetic data.
\newblock In \emph{The Fourteenth International Conference on Learning Representations}, 2026.

\bibitem[Yang et~al.(2025)Yang, Cheng, Luo, Zhou, and Zhang]{yang2025detecting}
Lingkai Yang, Jian Cheng, Yi~Luo, Tianbai Zhou, and Xiaoyu Zhang.
\newblock Detecting and rationalizing concept drift: A feature-level approach for understanding cause--effect relationships in dynamic environments.
\newblock \emph{Expert Systems with Applications}, 260:\penalty0 125365, 2025.

\bibitem[Yao et~al.(2025)Yao, Rancati, Cadei, Fumero, and Locatello]{yao2025unifyingcrl}
Dingling Yao, Dario Rancati, Riccardo Cadei, Marco Fumero, and Francesco Locatello.
\newblock Unifying causal representation learning with the invariance principle, 2025.

\bibitem[Yogi et~al.(2024)Yogi, V, K~M, Sujithra, Prasad, and Midhun]{yogi2024}
Kottala~Sri Yogi, Dankan~Gowda V, Mouna K~M, L.R. Sujithra, KDV Prasad, and P~Midhun.
\newblock Scalability and performance evaluation of machine learning techniques in high-volume social media data analysis.
\newblock In \emph{2024 11th International Conference on Reliability, Infocom Technologies and Optimization (Trends and Future Directions) (ICRITO)}, pp.\  1--6, 2024.

\bibitem[Zhang et~al.(2017)Zhang, Huang, Zhang, Glymour, and Sch\"{o}lkopf]{zhang2017cdnod}
Kun Zhang, Biwei Huang, Jiji Zhang, Clark Glymour, and Bernhard Sch\"{o}lkopf.
\newblock Causal discovery from nonstationary/heterogeneous data: skeleton estimation and orientation determination.
\newblock In \emph{Proceedings of the 26th International Joint Conference on Artificial Intelligence}, IJCAI'17, pp.\  1347–1353. AAAI Press, 2017.
\newblock ISBN 9780999241103.

\end{thebibliography}

\bibliographystyle{tmlr}

\appendix
\appendix
\onecolumn

\section{Deferred Proofs from Section 3}
\label{app:proofs}

\begin{proposition}[Propagation of Endogenous Drift]\label{prop:endogenous-app}
    Consider an SCM $\mathcal{M}$ and suppose an \emph{endogenous drift} at node $X_i$, such that the structural mechanism $f_i$ changes while all other mechanisms remain invariant. Assume in particular that the target mechanism $f_y$ is unchanged. Then: 
    \begin{enumerate}
        \item The structural conditional distribution $P(y \mid pa_y)$ remains invariant.
        \item For any ancestor $a \in \text{ancestors}(X_i)$, the conditional distribution $P(y \mid a)$ may change across concepts due to the altered mediation through $X_i$.
        \item Let $X_j$ be a descendant of $X_i$ such that $X_j$ \emph{d-separates} $X_i$ from $y$ -- that is, every path between $X_i$ and $y$ is blocked by conditioning on $X_j$. 
        Then $P(y \mid X_j) = P'(y \mid X_j)$. 
        Conversely, if $X_j$ does not d-separate $X_i$ from $y$, then $P(y\mid X_j)$ may change across concepts.
    \end{enumerate}
\end{proposition}

\textbf{Proof.} Let $P$ and $P'$ denote the joint distributions induced by the SCM before and after an endogenous drift at node $X_i$. 
Since the structural mechanism of the target node remains unchanged, $f_y = f'_y$.
Therefore, for any fixed value of the complete parent set $pa_y$, the conditional distribution induced by the target mechanism remains unchanged when conditioned to the full parent set:

\begin{equation}
    P(y\mid pa_y) = P'(y\mid pa_y).
\end{equation}

Hence, the conditional distribution of the target given all of its direct causes remains invariant (statement 1).

\textbf{Statement 2.} Let $a\in \text{ancestors}(X_i)$. 
By marginalizing over $X_i$,

\begin{equation}
\label{eq:marginalization-xi}
    P(y\mid a) = \sum_{X_i}P(y\mid X_i,a)P(X_i\mid a)
\end{equation}

Since $a$ is an ancestor of $X_i$, the changed mechanism $f_i$ can alter $P(X_i\mid a)$.
Moreover, the term $P(y\mid X_i,a)$ may not be invariant either: if $X_i$ has another ancestor $b\neq a$ that is not conditioned on and has a directed path to $y$, then conditioning on $X_i$ opens a dependency between them, and this dependency can propagate to $y$ through $b$.
In such cases, $P(y\mid X_i,a)$ may shift across concepts.
Substituting into Equation \eqref{eq:marginalization-xi}, the product -- and hence $P(y\mid a)$ may change across concepts, proving statement 2.

\textbf{Statement 3.} Let $X_j$ be a descendant of $X_i$ that d-separates $X_i$ from $y$. 
Since the graph structure is unchanged by the drift, this separation holds in both $\mathcal{M}$ and $\mathcal{M}'$. 
By the global Markov property (Pearl, 2009),

\begin{equation*}
    y\independent X_i\mid X_j,
\end{equation*}

\noindent in both $\mathcal{M}$ and $\mathcal{M}'$.
Marginalizing over $X_i$,

\begin{equation}
\label{eq:marg-xj}
    P(y\mid X_j) = \sum_{X_i} P(y\mid X_i, X_j)\, P(X_i\mid X_j). 
\end{equation}


By definition of conditional independence, conditioning on $X_j$ renders $X_i$ uninformative about $y$, i.e. $P(y\mid X_i,X_j) = P(y\mid X_j)$ for every value of $X_i$:



\begin{equation}
    P(y\mid X_j) = P(y\mid X_j)\sum_{X_i} P(X_i\mid X_j).
\end{equation}


We are left with $\sum_{X_i} P(X_i\mid X_j)$, which equals $1$: for a fixed value of $X_j$, $P(\cdot\mid X_j)$ is a valid probability distribution over $X_i$, and summing (or integrating, in the continuous case) any distribution over its full support gives $1$. 
Hence,

\begin{equation}
    P(y\mid X_j) = P(y\mid X_j)\cdot 1 = P(y\mid X_j),
\end{equation}

confirming the equality holds trivially once $X_i$ has dropped out: the result does not depend on $P(X_i\mid X_j)$, and therefore not on $f_i$, so $P(y\mid X_j) = P'(y\mid X_j)$.

Conversely, suppose some path between $X_i$ and $y$ remains active given $X_j$ -- either a directed path bypassing $X_j$, or a backdoor path through a common cause of $X_i$ and $y$. Then
\begin{equation*}
    y \notindependent X_i \mid X_j,
\end{equation*}
so $P(y\mid X_i, X_j)$ generically varies with $X_i$.

Separately, since $X_j$ is a descendant of $X_i$, $X_i$ influences $X_j$ along a directed path through mechanisms that remain unchanged by the drift.
A change in $f_i$ therefore generically propagates through this path and alters the joint distribution of $(X_i, X_j)$, and hence $P(X_i\mid X_j)$.

Substituting into Equation~\eqref{eq:marg-xj}, the altered weights $P(X_i\mid X_j)$ now reweight a summand that is not constant in $X_i$, so $P(y\mid X_j)$ may change across concepts.
This concludes the pr



\begin{proposition}[Effects of Confounder Drift]\label{prop:confounder-drift-general}
    Consider an SCM where an exogenous confounder $C$ and a variable $X$ are direct parents of the target $y$, forming a backdoor path $X \leftarrow C \rightarrow y$. 
    Let $Z = pa_y \setminus \{X, C\}$ denote the set of all other parents of $y$. 
    Suppose a drift occurs such that the marginal distribution of the confounder changes ($P(C)\neq P'(C)$), while all other structural mechanisms in the SCM remain invariant. 
    Hence:
    \begin{enumerate}
        \item The observable conditional distribution $P(y \mid X)$ may change across concepts due to the altered spurious association.
        \item The conditional interventional distribution, denoted by $P(y\mid \text{do}(X), c)$, remains strictly invariant.
    \end{enumerate}
\end{proposition}

\textbf{Proof.} Let $P$ and $P'$ denote the joint distributions before and after the drift.
By the law of total probability, the observable conditional distribution of $y$ given $X$ requires marginalizing over both the confounder $C$ and the other parents $Z$:

\begin{equation}\label{eq}
P(y \mid X) = \sum_{c, z} P(y \mid X, c, z) P(c, z \mid X).
\end{equation}

By assumption, the structural mechanism of the target ($f_y$) remains invariant.
Therefore, the conditional distribution of $y$ given its complete parent set remains unchanged:

\begin{equation}
P(y \mid X, c, z) = P'(y \mid X, c, z).
\end{equation}

However, the joint distribution of the parents conditioned on $X$ may change across concepts. 
By Bayes' theorem,

\begin{equation}
P(C,Z\mid X) = \frac{P(X\mid C,Z)P(C,Z)}{P(X)}.
\end{equation}

Since the marginal distribution of the confounder changes, $P(C)\neq P'(C)$, the joint distribution $P(C,Z)$, and consequently the posterior distribution $P(C,Z\mid X)$, may also change. 
Thus, although the conditional distribution of $y$ given its complete parent set remains invariant, the weights assigned to the different values of $(C,Z)$ when marginalizing over the unobserved variables may change.
Therefore, $P(y\mid X)$ need not be invariant, proving the first statement.





\textbf{Statement 2}. For the second statement, we apply the rules of do-calculus (Pearl, 2009). 
Since $\{X,C,Z\}=pa_y$ and the intervention $\text{do}(X)$ replaces only $X$'s own equation while leaving $f_y$ untouched, we have that
\begin{equation}
    P(y\mid \text{do}(X), c, z) = P(y \mid x, c, z).
\end{equation}

Substituting into the law of total probability then yields:

\begin{equation}\label{eq:do-calc-general}
    P(y\mid \text{do}(X),c) = \sum_{z} P(y\mid x, c, z) P(z \mid \text{do}(X), c).
\end{equation}

As established, the structural mechanism $P(y\mid x, c, z)$ (left term in the summation) is invariant.
Since all structural mechanisms are assumed invariant across concepts, it follows that:

\begin{equation}
    P(z \mid \text{do}(X), c) = P'(z \mid \text{do}(X), c).
\end{equation}

The distribution of $Z$ is therefore entirely determined by the remaining structural mechanisms of the SCM, all of which are invariant by assumption. 
Since both terms inside the summation in Equation~\eqref{eq:do-calc-general} are invariant, it follows that the conditional interventional distribution $P(y \mid \text{do}(X), c)$ is also invariant. 
Once $C=c$ is fixed, the change in the marginal distribution $P(C)$ no longer affects the conditional interventional distribution, assuming the relevant mechanisms remain unchanged.
\qed

\section{CaDrift Pseudocode}
\label{app:pseudocode}

\ac{CaDrift}'s pseudocode can be found in Algorithm~\ref{alg:cadrift}.

\begin{algorithm}[htbp]
\caption{CaDrift: Time-Dependent SCM Data Stream Generation}
\label{alg:cadrift}
\begin{algorithmic}[1]
\State \textbf{Input:} Causal graph $\mathcal{G} = (\mathbf{V}, \mathbf{E})$, structural functions $\{f_i\}_{i=1}^d$, number of time steps $N$, AR parameter $\rho$, EWMA parameter $\alpha$, seasonal parameters $\{(A_i, T_i, \phi_i)\}$, drift schedule $\mathcal{D}$.
\State \textbf{Output:} Non-stationary data stream $\mathcal{S} = \{(x^{(t)}, y^{(t)})\}_{t=1}^N$
\State Initialize $z_i^{(0)} = 0$ and $U_i^{(0)} = 0$ for all $i \in \{1, \dots, d\}$
\State $\mathcal{S} \gets \emptyset$
\For{$t = 1$ \textbf{to} $N$}
    \If{$t \in \mathcal{D}$}
        \State \text{Apply soft intervention to SCM} \Comment{Induce causal drift}
    \EndIf
    \For{\text{each node } $X_i \in \mathbf{V}$ \text{ in topological order}}
        \State Sample Gaussian noise $\epsilon_i^{(t)} \sim \mathcal{N}(0, \sigma^2)$
        \State $U_i^{(t)} \gets \rho U_i^{(t-1)} + \epsilon_i^{(t)}$ \Comment{Autoregressive noise}
        \State $s_i^{(t)} \gets A_i \sin\left(\frac{2\pi t}{T_i}+\phi_i\right)$ \Comment{Seasonal component}
        
        \State $X_i^{(t)} \gets f_i(pa_i^{(t)}, U_i^{(t)}) + s_i^{(t)}$ \Comment{Base structural assignment}
        
        \If{$pa_i^{(t)} = \emptyset$} \Comment{If $X_i$ is a root node}
            \State $z_i^{(t)} \gets (1-\alpha)z_i^{(t-1)} + \alpha X_i^{(t)}$ \Comment{EWMA smoothing}
            \State $X_i^{(t)} \gets z_i^{(t)}$
        \EndIf
    \EndFor
    \State Extract feature vector $\mathbf{x}^{(t)}$ and target variable $y^{(t)}$ from $\mathbf{V}^{(t)}$
    \State $\mathcal{S} \gets \mathcal{S} \cup \{(\mathbf{x}^{(t)}, y^{(t)})\}$
\EndFor
\State \textbf{Return} $\mathcal{S}$
\end{algorithmic}
\end{algorithm}

\section{Additional Details on Sampling DAGs}
\label{app:detail-dag}

The datasets in Section \ref{subsec:exp-dist} are sampled using the \acp{DAG} in Figure \ref{fig:graph-mmd}.
In these \acp{DAG}, the mapping functions of inner nodes are linear functions.
Here, we present the linear mapping equations for each variable on those \acp{DAG}.
We omit the exogenous noise terms $U_i^{(t)}$ in the equations.
For \ac{DAG} \#1, these are the mapping functions:

\begin{multicols}{2}
\begin{itemize}
    \item $X_1 \sim \mathcal{N}(\mu_1,\sigma_1^2)$
    \item $X_2 \sim \mathcal{U}(a_2,b_2)$
    \item $X_3 :=w_{31} X_1 + w_{32} X_2 + w_{33} X_6$
    \item $X_4:=w_{41}X_3$
    \item $X_5 := w_{51}X_3$
    \item $X_6 \sim \mathcal{N}(\mu_6,\sigma_6^2)$
    \item $y:=f_y(X_4,X_5,X_6)$
\end{itemize}
\end{multicols}

For \ac{DAG} \#2, we have:

\begin{multicols}{2}
\begin{itemize}
    \item $X_1 \sim \mathcal{N}(\mu_1,\sigma_1^2)$
    \item $X_2 \sim \mathcal{U}(a_2,b_2)$
    \item $X_3 \sim \mathcal{N}(\mu_3,\sigma_3^2)$
    \item $X_4 :=w_{41} X_1 + w_{42} X_2 + w_{43} X_6$
    \item $X_5:=w_{51}X_1 + w_{52}X_4$
    \item $X_6 := w_{61}X_3 + w_{62}X_4$
    \item $X_7 := w_{71}X_5 + w_{72}X_6$
    \item $X_8 := w_{81}X_5 + w_{82}X_6$
    \item $X_9 \sim \mathcal{N}(\mu_9,\sigma_9^2)$
    \item $X_{10} \sim \mathcal{N}(\mu_{10},\sigma_{10}^2)$
    \item $y := f_y(X_7,X_8,X_9,X_{10})$
\end{itemize}
\end{multicols}

And, finally, for \ac{DAG} \#3:

\begin{multicols}{2}
\begin{itemize}
    \item $X_1 \sim \mathcal{N}(\mu,\sigma^2)$
    \item $X_2 \sim \mathcal{U}(a_2,b_2)$
    \item $X_3 :=w_{31} X_1 + w_{32} X_2$
    \item $X_4 :=w_{41} X_1 + w_{42} X_3 + w_{43} X_6$
    \item $X_5:=w_{51}X_3+ w_{52}X_2$
    \item $X_6 \sim \mathcal{N}(\mu,\sigma^2)$
    \item $X_7 := w_{71}y + w_{72}X_{10}$
    \item $X_8 := w_{81}X_6 + w_{82}y$
    \item $X_9 := w_{91}X_7 + w_{92}X_8$
    \item $X_{10} \sim \mathcal{N}(\mu,\sigma^2)$
    \item $y := f_y(X_4,X_5,X_6,X_{10})$
\end{itemize}
\end{multicols}

All of the linear function weights $w_{ij}$ are initialized randomly by following a uniform distribution $\mathcal{U}(-1,1)$. The parameter $\mu_i$ of the Gaussian distributions sampling exogenous variables is randomly set to a value in the range $[-1,1]$, while the $\sigma$ parameter is set to a value in the range $[0.5,1.5]$.
The hyperparameters $a$ and $b$ in the Uniform distributions, also sampling exogenous variables, are sampled in the ranges $[-2,0]$ and $[0, 2]$, respectively.

\emph{Endogenous drift} is simulated by randomly modifying the weights of the linear functions, while \emph{exogenous drift} and \emph{confounder drift} are simulated by changing the parameters of the Gaussian distributions $\mathcal{N}(\mu,\sigma^2)$. 

The target variable $y$ is defined by a prototype-based mapping $f_y$, where each class is associated with a prototype vector in the space of the parent variables of $y$. 
Given an input sample $x$, the class is assigned according to the nearest prototype under the Euclidean distance:
\[
f_y(\mathbf{x}) = \arg\min_{k \in \{1, \dots, K\}} \lVert \mathbf{x} - \mathbf{c}_k \rVert_2.
\]

To initialize the prototypes, we compute the empirical mean $\mu_{pa_y}$ and standard deviation $\sigma_{pa_y}$ of the parent variables of $y$. 
Each prototype $\mathbf{c}_k$ is then sampled independently from a Gaussian distribution:
\[
\mathbf{c}_k \sim \mathcal{N}\left(
\mu_{pa_y}, \;
\beta^2 \, \mathrm{diag}\big(\sigma_{pa_y}^2\big)
\right),
\]
where $\beta$ is a scaling factor controlling the dispersion of the prototypes, which has been set to $\beta=1$. 
This results in class centers that are aligned with the distribution of the causal parents while allowing variability across classes.

\emph{Target drift} is simulated by moving the prototypes $\{\mathbf{c}_k\}_{k=1}^K$, thereby modifying the decision boundaries induced by $f_y$.

Lastly, \emph{structural drift} is simulated by changing the parents of a node, which might include either removing or adding nodes to the mapping function $f_i$.
When adding a new node as a parent, its weight on the linear function is also randomly initialized through a uniform distribution.

\subsection{Nonlinear functions}

One of our experiments in Section \ref{subsec:exp-dist} comprises datasets generated using SCMs with nonlinear mapping functions.
We report in Table \ref{tab:nonlinear-funcs} which functions were used on each node and DAG.

\begin{table}[htb]
    \centering
    \caption{Nonlinear functions used in each DAG and node.}
    \begin{tabular}{l c c c}
    \toprule
        Node & DAG \#1 & DAG \#2 & DAG \# 3 \\
        \midrule
        $X_1$ & Normal & Normal & Normal \\
        $X_2$ & Uniform & Uniform & Uniform \\
        $X_3$ & Polynomial & Normal & Polynomial \\
        $X_4$ & Sigmoid & Polynomial & Sigmoid \\
        $X_5$ & Sigmoid & Sigmoid & Polynomial \\
        $X_6$ & Normal & Polynomial & Normal \\
        $X_7$ & -- & Sigmoid & Sigmoid \\
        $X_8$ & -- & Polynomial & Polynomial \\
        $X_9$ & -- & Normal & Sigmoid \\
        $X_{10}$ & -- & Normal & Normal \\
        \bottomrule
    \end{tabular}
    \label{tab:nonlinear-funcs}
\end{table}

\subsection{Drift events -- Atomic Drift}
\label{app:drift-events}
 
In this section, we describe the drift events in each \ac{DAG}, divided by
drift type, for the experiments performed in Section~\ref{subsec:exp-dist}.
On all \acp{DAG} there are 9 drift events, one every 2{,}000 samples,
totaling 20{,}000 data samples and 10 concepts per dataset.
 
\subsubsection{DAG \#1}
 
The exogenous nodes are $X_1$ and $X_2$; since $X_1$ is a confounder,
only $X_2$ is a non-confounder exogenous node, so \emph{exogenous drift}
events are simulated exclusively by modifying the Gaussian distribution
sampling $X_2$ (changing its parameter $\mu$).
 
\paragraph{Endogenous drift.} The drifted node(s) at each event:
 
\begin{center}
\begin{tabular}{@{}lccccccccc@{}}
\toprule
Event & 1 & 2 & 3 & 4 & 5 & 6 & 7 & 8 & 9 \\
\midrule
Node(s) & $X_3$ & $X_3$ & $X_4$ & $X_5$ & $X_3,X_4$ & $X_4,X_5$ & $X_5$ & $X_3$ & $X_4$ \\
\bottomrule
\end{tabular}
\end{center}
 
\paragraph{Confounder drift.} DAG~\#1 has two confounder nodes: $X_1$
(observable) and $X_6$ (latent).
 
\begin{center}
\begin{tabular}{@{}lccccccccc@{}}
\toprule
Event & 1 & 2 & 3 & 4 & 5 & 6 & 7 & 8 & 9 \\
\midrule
Node(s) & $X_1$ & $X_6$ & $X_1$ & $X_6$ & $X_1$ & $X_1,X_6$ & $X_1,X_6$ & $X_6$ & $X_1$ \\
\bottomrule
\end{tabular}
\end{center}
 
\paragraph{Structural drift.} Simulated by changing DAG edges; for the
linear function this means adding a node with a random weight. The
mechanism after each event:
 
\begin{center}
\begin{tabular}{@{}cl@{}}
\toprule
Event & Updated equation \\
\midrule
1 & $X_5 := w_{51}X_3 + w_{52}X_2$ \\
2 & $X_5 := w_{51}X_3 + w_{52}X_1$ \\
3 & $X_4 := w_{41}X_3 + w_{42}X_2$ \\
4 & $X_3 := 0 X_1 + w_{32}X_2 + w_{33}X_6$ \quad (removes edge $X_1 \to X_3$) \\
5 & $X_5 := w_{51}X_3 + w_{52}X_6$ \\
6 & $X_4 := w_{41}X_3 + w_{42}X_6$ \\
7 & $y := f_y(X_4, X_5, X_6, X_3)$ \\
8 & $y := f_y(X_4, X_5, X_6, X_2)$ \\
9 & $X_4 := w_{41}X_3 + w_{42}X_2$ \quad and \quad $X_3 := w_{31}X_1 + w_{32}X_2 + \tilde{w}_{33}X_6$ \\
\bottomrule
\end{tabular}
\end{center}
 
\emph{Target drift} is simulated by randomly shifting the centroids
(applies uniformly across all three \acp{DAG}, so is not tabulated
per-DAG). For the nonlinear functions, \emph{structural drift} affects
the same nodes as in the corresponding linear-mapper SCMs above.
 
\subsubsection{DAG \#2}
 
The exogenous (non-confounder) nodes are $X_1$, $X_2$, $X_3$.
 
\paragraph{Exogenous drift.}
\begin{center}
\begin{tabular}{@{}lccccccccc@{}}
\toprule
Event & 1 & 2 & 3 & 4 & 5 & 6 & 7 & 8 & 9 \\
\midrule
Node(s) & $X_1$ & $X_2$ & $X_3$ & $X_1$ & $X_2$ & $X_3$ & $X_1$ & $X_2$ & $X_3$ \\
\bottomrule
\end{tabular}
\end{center}
 
\paragraph{Endogenous drift.}
\begin{center}
\begin{tabular}{@{}lccccccccc@{}}
\toprule
Event & 1 & 2 & 3 & 4 & 5 & 6 & 7 & 8 & 9 \\
\midrule
Node(s) & $X_4$ & $X_6$ & $X_8$ & $X_7$ & $X_4$ & $X_5$ & $X_6$ & $X_8$ & $X_7$ \\
\bottomrule
\end{tabular}
\end{center}
 
\paragraph{Confounder drift.} Both confounders, $X_9$ and $X_{10}$, are
observable.
\begin{center}
\begin{tabular}{@{}lccccccccc@{}}
\toprule
Event & 1 & 2 & 3 & 4 & 5 & 6 & 7 & 8 & 9 \\
\midrule
Node(s) & $X_9$ & $X_{10}$ & $X_9$ & $X_{10}$ & $X_9$ & $X_{10}$ & $X_9$ & $X_{10}$ & $X_9$ \\
\bottomrule
\end{tabular}
\end{center}
 
\paragraph{Structural drift.}
\begin{center}
\begin{tabular}{@{}cl@{}}
\toprule
Event & Updated equation \\
\midrule
1 & $X_8 := w_{81}X_5 + w_{82}X_6 + w_{83}X_4$ \\
2 & $X_8 := w_{81}X_5 + w_{82}X_6 + w_{83}X_1$ \\
3 & $X_6 := w_{61}X_3 + w_{62}X_4 + w_{63}X_1$ \\
4 & $y := f_y(X_7, X_8, X_9, X_{10}, X_6)$ \\
5 & $X_7 := w_{71}X_5 + w_{72}X_6 + w_{73}X_8$ \\
6 & $X_8 := w_{81}X_5 + w_{82}X_6 + w_{83}X_4$ \\
7 & $X_8 := w_{81}X_5 + w_{82}X_6 + w_{83}X_3$ \\
8 & $X_5 := w_{51}X_1 + w_{52}X_4 + w_{53}X_2$ \\
9 & $y := f_y(X_7, X_8, X_9, X_{10}, X_1)$ \\
\bottomrule
\end{tabular}
\end{center}
 
\subsubsection{DAG \#3}
 
The exogenous non-confounder nodes are $X_1$ and $X_2$.
 
\paragraph{Exogenous drift.}
\begin{center}
\begin{tabular}{@{}lccccccccc@{}}
\toprule
Event & 1 & 2 & 3 & 4 & 5 & 6 & 7 & 8 & 9 \\
\midrule
Node(s) & $X_1$ & $X_2$ & $X_1$ & $X_2$ & $X_1$ & $X_2$ & $X_1,X_2$ & $X_1$ & $X_1,X_2$ \\
\bottomrule
\end{tabular}
\end{center}
 
\paragraph{Endogenous drift.}
\begin{center}
\begin{tabular}{@{}lccccccccc@{}}
\toprule
Event & 1 & 2 & 3 & 4 & 5 & 6 & 7 & 8 & 9 \\
\midrule
Node(s) & $X_3$ & $X_5$ & $X_4$ & $X_8$ & $X_3$ & $X_4$ & $X_5$ & $X_7$ & $X_9$ \\
\bottomrule
\end{tabular}
\end{center}
 
\paragraph{Confounder drift.} $X_6$ is an observable confounder; $X_{10}$
is latent.
\begin{center}
\begin{tabular}{@{}lccccccccc@{}}
\toprule
Event & 1 & 2 & 3 & 4 & 5 & 6 & 7 & 8 & 9 \\
\midrule
Node(s) & $X_6$ & $X_{10}$ & $X_6$ & $X_{10}$ & $X_6$ & $X_{10}$ & $X_6,X_{10}$ & $X_6$ & $X_6,X_{10}$ \\
\bottomrule
\end{tabular}
\end{center}
 
\paragraph{Structural drift.}
\begin{center}
\begin{tabular}{@{}cl@{}}
\toprule
Event & Updated equation \\
\midrule
1 & $X_7 := w_{71}y + w_{72}X_{10} + w_{73}X_5$ \\
2 & $X_5 := w_{51}X_3 + w_{52}X_2 + w_{53}X_1$ \\
3 & $X_8 := w_{81}X_6 + w_{82}y + w_{83}X_4$ \\
4 & $X_7 := 0y + w_{72}X_{10}$ \quad (removes edge $y \to X_7$) \\
5 & $X_9 := 0X_7 + w_{92}X_8$ \quad (removes edge $X_7 \to X_9$) \\
6 & $y := f_y(X_4, X_5, X_6, X_{10}, X_3)$ \\
7 & $y := f_y(X_4, X_5, X_6, X_{10}, X_1)$ \\
8 & $X_7 := w_{71}y + w_{72}X_{10} + w_{73}X_6$ \\
9 & $X_9 := w_{91}X_7 + w_{92}X_8 + w_{93}X_{10}$ \\
\bottomrule
\end{tabular}
\end{center}
 
\subsection{Drift Events -- Compound Drift}
\label{app:drift-events-compound}
 
In the compound drift setting, soft interventions were applied to two
node types simultaneously (e.g., exogenous and target). Each table below
lists which node(s) were intervened on, per event, for one compound
configuration. \emph{Structural+target drift} always reuses the
structural-drift equations from Section~\ref{app:drift-events}, adding
an intervention on the target node at each event, so it is not
retabulated here.
 
\subsubsection{DAG \#1}
 
\begin{center}
\resizebox{\textwidth}{!}{%
\begin{tabular}{@{}llccccccccc@{}}
\toprule
Configuration & & 1 & 2 & 3 & 4 & 5 & 6 & 7 & 8 & 9 \\
\midrule
Exogenous+target       & & \multicolumn{9}{c}{every event: $X_2, y$} \\
Endogenous+target      & & $X_3,y$ & $X_4,y$ & $X_5,y$ & $X_3,y$ & $X_4,y$ & $X_5,y$ & $X_3,y$ & $X_4,y$ & $X_5,y$ \\
Confounder+target      & & $X_1,y$ & $X_6,y$ & $X_1,y$ & $X_6,y$ & $X_1,y$ & $X_6,y$ & $X_1,y$ & $X_6,y$ & $X_1,y$ \\
Exogenous+endogenous   & & $X_2,X_3$ & $X_2,X_4$ & $X_2,X_5$ & $X_2,X_3$ & $X_2,X_4$ & $X_2,X_5$ & $X_2,X_3$ & $X_2,X_4$ & $X_2,X_5$ \\
Confounder+endogenous  & & $X_1,X_3$ & $X_6,X_4$ & $X_1,X_5$ & $X_6,X_3$ & $X_1,X_4$ & $X_6,X_5$ & $X_1,X_3$ & $X_6,X_4$ & $X_1,X_5$ \\
\bottomrule
\end{tabular}
}
\end{center}
 
\subsubsection{DAG \#2}
 
\begin{center}
\resizebox{\textwidth}{!}{%
\begin{tabular}{@{}lccccccccc@{}}
\toprule
Configuration & 1 & 2 & 3 & 4 & 5 & 6 & 7 & 8 & 9 \\
\midrule
Exogenous+target      & $X_1,y$ & $X_2,y$ & $X_3,y$ & $X_1,y$ & $X_2,y$ & $X_3,y$ & $X_1,y$ & $X_2,y$ & $X_3,y$ \\
Endogenous+target     & $X_4,y$ & $X_5,y$ & $X_6,y$ & $X_8,y$ & $X_7,y$ & $X_4,y$ & $X_5,y$ & $X_6,y$ & $X_8,y$ \\
Confounder+target     & $X_9,y$ & $X_{10},y$ & $X_9,y$ & $X_{10},y$ & $X_9,y$ & $X_{10},y$ & $X_9,y$ & $X_{10},y$ & $X_9,y$ \\
Exogenous+endogenous  & $X_4,X_1$ & $X_5,X_2$ & $X_6,X_3$ & $X_8,X_1$ & $X_7,X_2$ & $X_4,X_3$ & $X_5,X_1$ & $X_6,X_2$ & $X_8,X_3$ \\
Confounder+endogenous & $X_4,X_9$ & $X_5,X_{10}$ & $X_6,X_9$ & $X_8,X_{10}$ & $X_7,X_9$ & $X_4,X_{10}$ & $X_5,X_9$ & $X_6,X_{10}$ & $X_8,X_9$ \\
\bottomrule
\end{tabular}
}
\end{center}
 
\subsubsection{DAG \#3}
 
\begin{center}
\resizebox{\textwidth}{!}{%
\begin{tabular}{@{}lccccccccc@{}}
\toprule
Configuration & 1 & 2 & 3 & 4 & 5 & 6 & 7 & 8 & 9 \\
\midrule
Exogenous+target      & $X_1,y$ & $X_2,y$ & $X_1,y$ & $X_2,y$ & $X_1,y$ & $X_2,y$ & $X_2,y$ & $X_1,y$ & $X_2,y$ \\
Endogenous+target     & $X_3,y$ & $X_4,y$ & $X_5,y$ & $X_7,y$ & $X_8,y$ & $X_3,y$ & $X_4,X_5,y$ & $X_7,y$ & $X_8,y$ \\
Confounder+target     & $X_{10},y$ & $X_6,y$ & $X_{10},y$ & $X_6,y$ & $X_{10},y$ & $X_6,y$ & $X_{10},y$ & $X_6,y$ & $X_{10},y$ \\
Exogenous+endogenous  & $X_1,X_3$ & $X_2,X_4$ & $X_1,X_5$ & $X_2,X_7$ & $X_1,X_8$ & $X_2,X_3$ & $X_2,X_4$ & $X_1,X_5$ & $X_2,X_7$ \\
Confounder+endogenous & $X_3,X_{10}$ & $X_4,X_6$ & $X_5,X_{10}$ & $X_7,X_6$ & $X_8,X_{10}$ & $X_3,X_6$ & $X_4,X_5,X_{10}$ & $X_7,X_6$ & $X_8,X_{10}$ \\
\bottomrule
\end{tabular}
}
\end{center}

\section{Generator hyperparameters}
\label{app:hyperparameters}

Table~\ref{tab:hyperparam-sec-61} reports the hyperparameter configuration of \ac{CaDrift} used to generate the datasets for the experiments in Section~\ref{subsec:exp-dist}.
All non-stationarity parameters were set to zero, as these experiments aim to isolate the effects of causal drift events on the data distribution and predictive performance.

We employ a number of prototypes equal to ten times the number of classes to induce a non-linear mapping between the parents of $y$ and the class labels.
This choice increases the expressiveness of the data-generating process and allows different concepts to occupy distinct regions of the feature space.

\begin{table}[htb]
    \centering
    \caption{CaDrift hyperparameters to generate data samples for the experiments in Section 6.1.}
    \begin{tabular}{l c}
    \toprule
        Hyperparameter & Value \\
        \midrule
        \# classes & 5 \\
        \# prototypes & 50 \\
        $\rho$ & 0 \\
        $\alpha$ & 0 \\
        $A_i$ & 0 \\
        $T_i$ & 0 \\
        $\phi$ & 0 \\
        \bottomrule
    \end{tabular}
    \label{tab:hyperparam-sec-61}
\end{table}

Table \ref{tab:mlp-hyperparam} reports the hyperparameters of the neural network utilized to approximate cause--effect relationships on the dataset synthesized from the ELEC2 dataset.

\begin{table}[htb]
    \centering
    \caption{Neural network hyperparameters to map cause--effect relationships on the ELEC2 dataset.}
    \begin{tabular}{l c}
    \toprule
        Hyperparameter & Value \\
        \midrule
        Learning Rate & 0.001 \\
        Hidden layers & 1 \\
        \# neurons hidden layer & 10 \\
        Optimizer & $adam$ \\
        $max\_iter$ & 100 \\
        hidden layer $activation\_function$ & ReLU \\
        \bottomrule
    \end{tabular}
    \label{tab:mlp-hyperparam}
\end{table}

\section{The impact of causal drift -- Time-dependent mappers}
\label{app:causal-drift-additional-exp}

The plots of MMD, KL divergence, and prequential accuracy performance when the components to induce time dependence (autoregressive noise, EWMA, and seasonality) are enabled are shown in Figure \ref{fig:mmd-kl-acc-full-nonstationary}.
We notice some noise in MMD and KL divergence measured across windows belonging to the same concept when compared to the stationary version, but there is still a noticeable discrepancy when measuring different concepts (i.e., SCMs before and after interventions).

\begin{figure}[hp]
    \centering
    \subfloat[Exogenous drift.]{\label{subfig:mmd-kl-acc-exogenous-nonstationary}
        \includegraphics[width=0.8\linewidth]{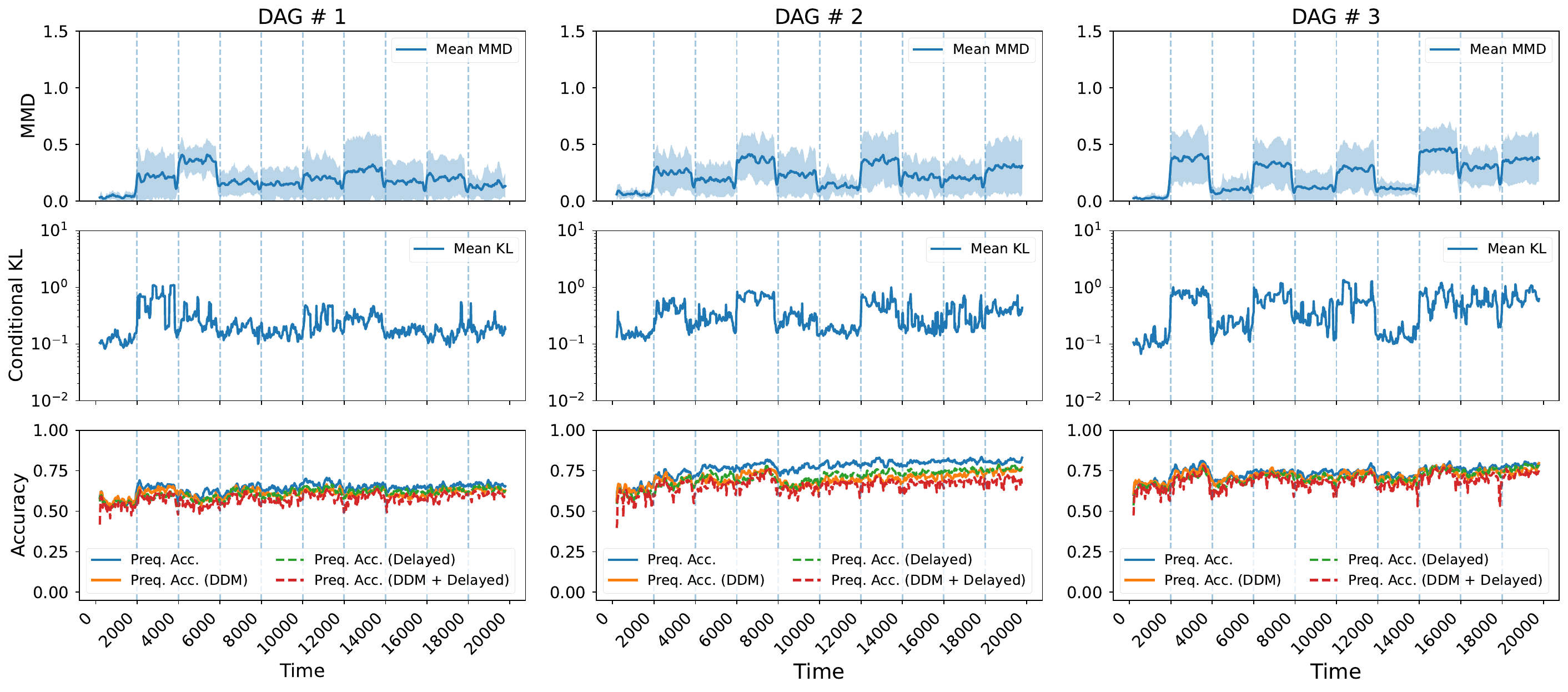}
    }    
    
    \subfloat[Endogenous drift.]{\label{subfig:mmd-kl-acc-endogenous-nonstationary}
        \includegraphics[width=0.8\linewidth]{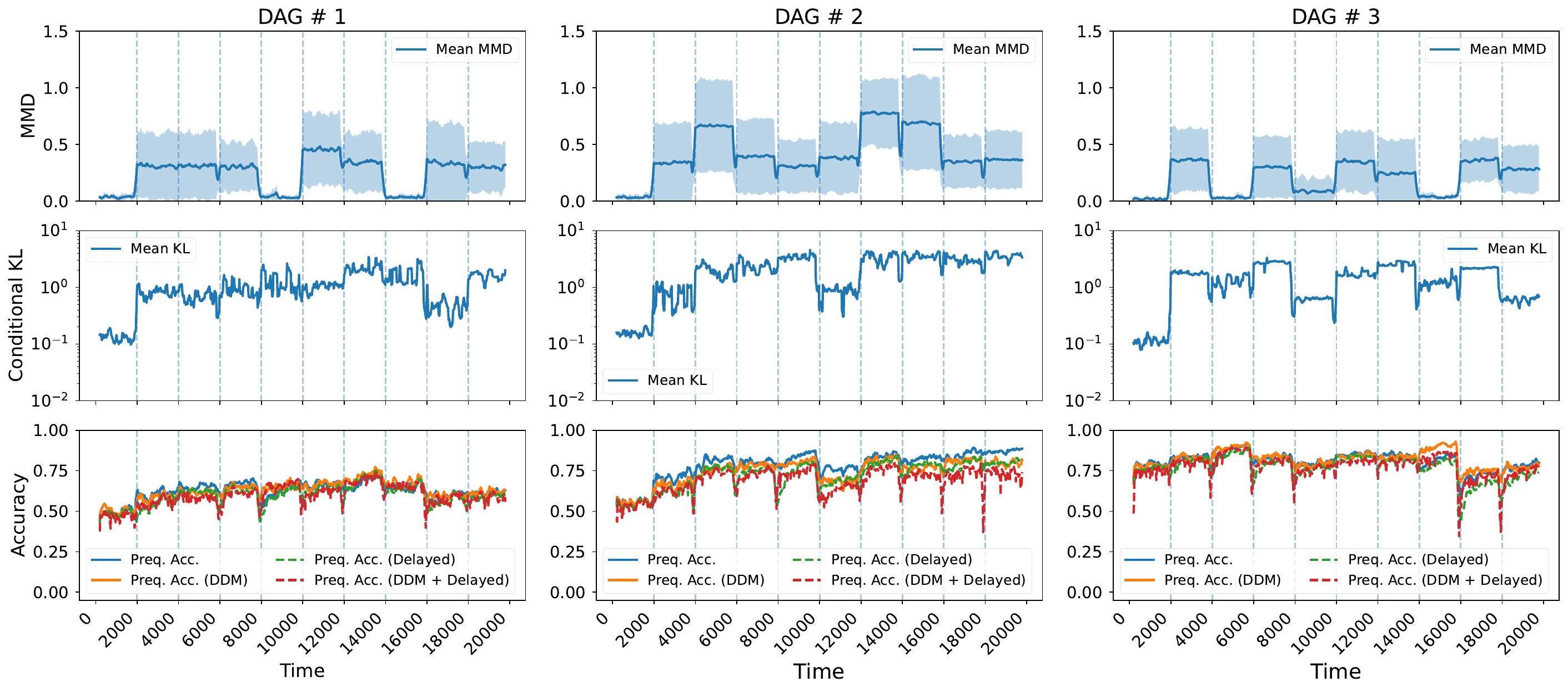}
    }
    
    \subfloat[Confounder drift.]{\label{subfig:mmd-kl-acc-confounder-nonstationary}
        \includegraphics[width=0.8\linewidth]{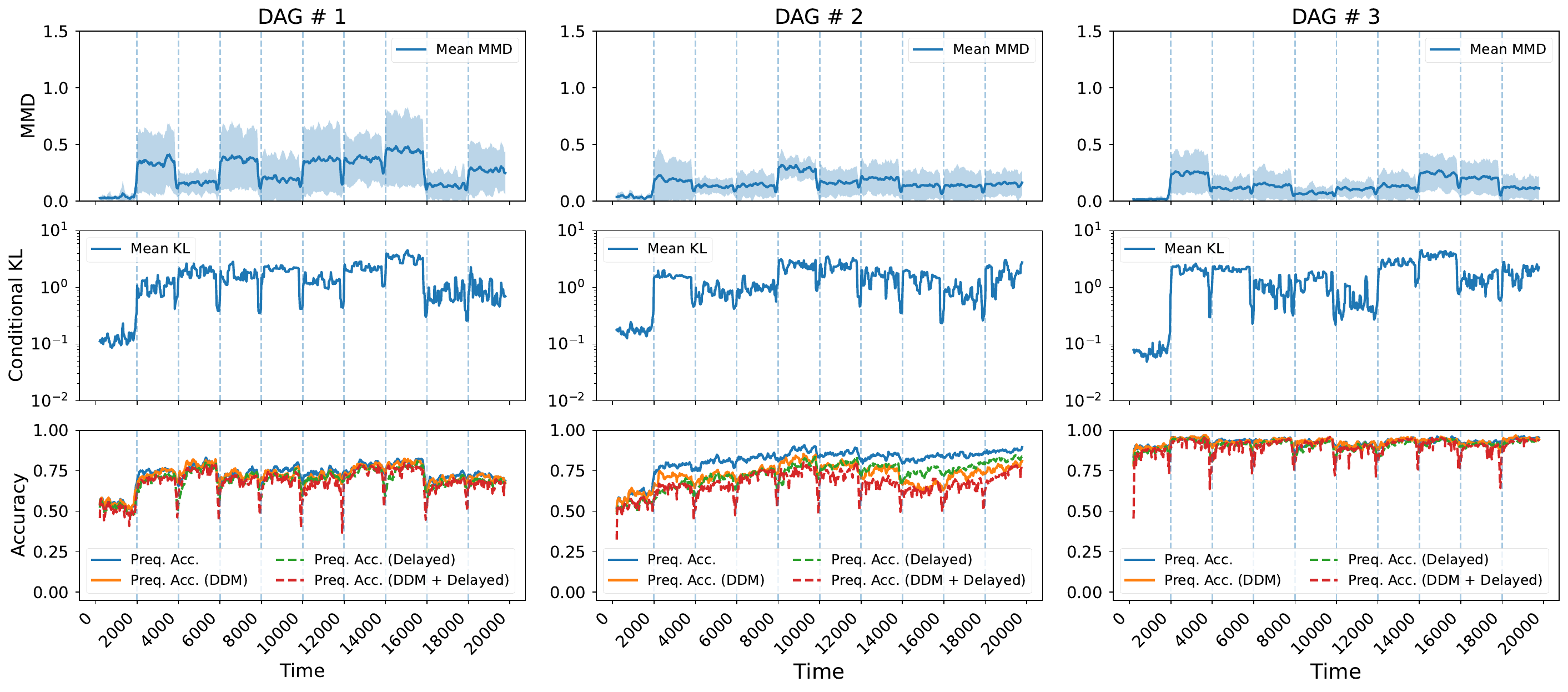}
    }
    
    \caption{MMD, KL divergence, and prequential accuracy over time in each DAG on causal drift types -- Time-dependent.}
    \label{fig:mmd-kl-acc-full-nonstationary}
\end{figure}

\begin{figure}[ht]
\ContinuedFloat
    \centering
    
    \subfloat[Target drift.]{\label{subfig:mmd-kl-acc-target-nonstationary}
        \includegraphics[width=0.8\linewidth]{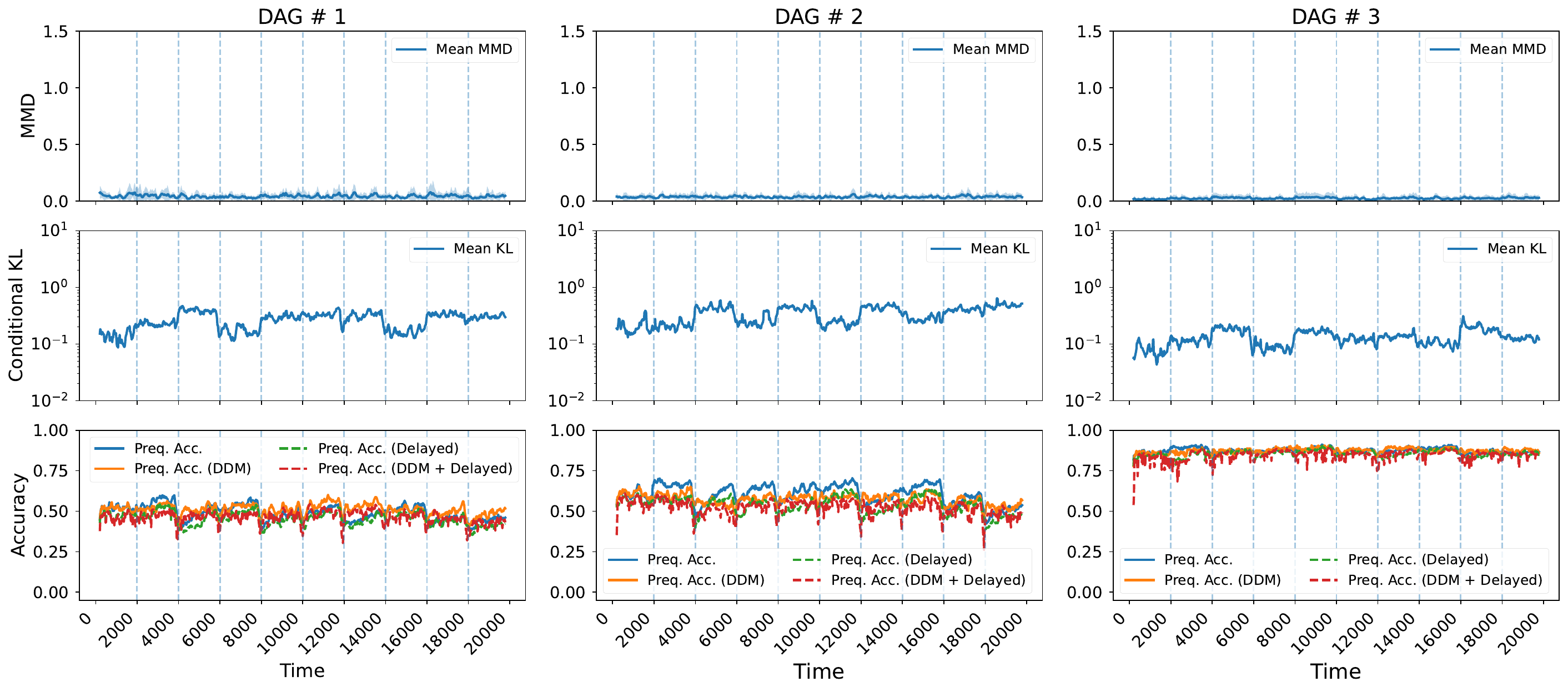}
    }

    \subfloat[Structural drift.]{\label{subfig:mmd-kl-acc-structural-nonstationary}
        \includegraphics[width=0.8\linewidth]{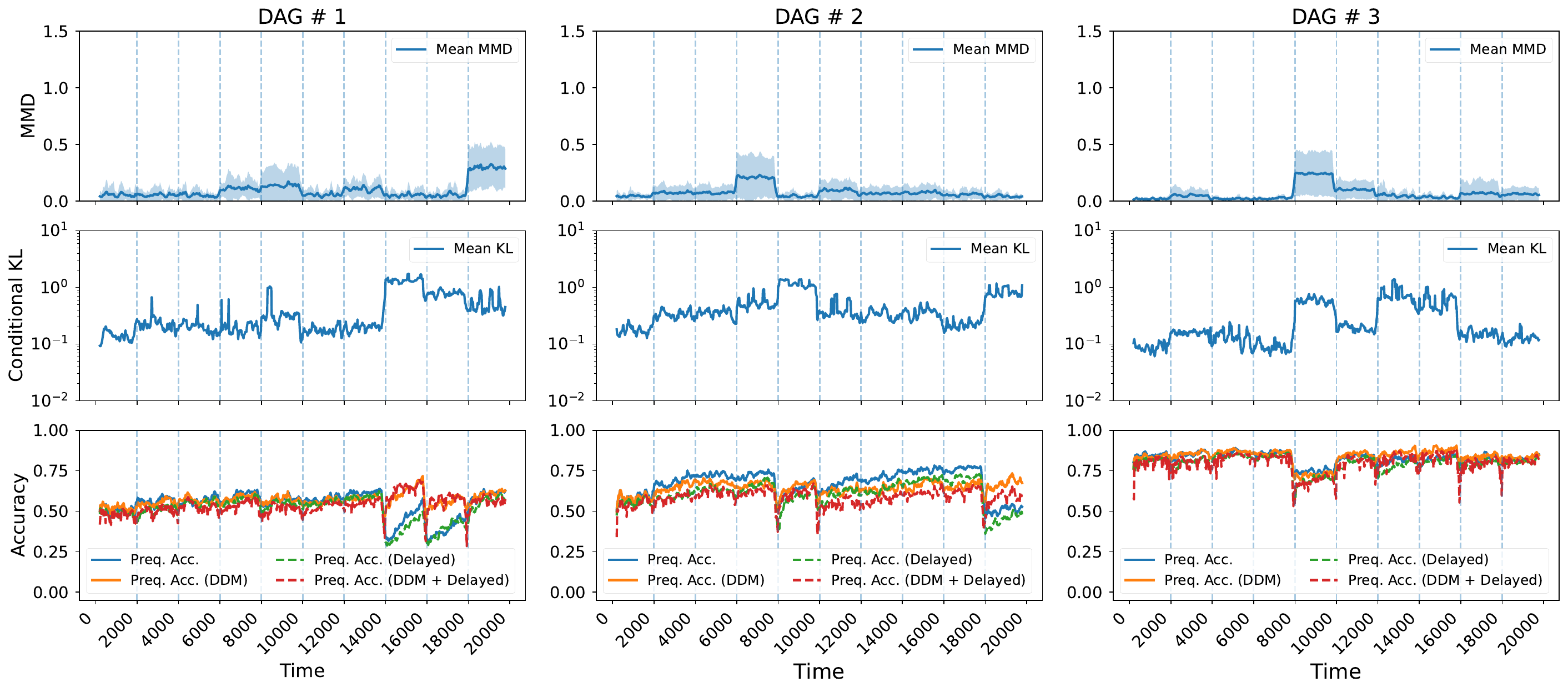}
    }
    
    \caption{MMD, KL divergence, and prequential accuracy over time in each DAG on causal drift -- Time-dependent  (continued).}
    \label{fig:mmd-kl-acc-full-nonstationary-2}
\end{figure}

\section{Class distributions and MMD across causal drift events}
\label{app:class-dist-mmd}

In Figure~\ref{fig:drift-class-dist}, we show the class distributions across batches for datasets generated by the \ac{SCM} framework with induced causal drift events, in which each color represents a different class. These are sampled from \ac{DAG} \#1, and the drift events are the same as described in Appendix \ref{app:drift-events}. We plot the features $X_1$ and $X_4$, two direct causes of the target $y$. Each color in the plots represents a class.

Under \emph{exogenous drift} affecting the node $X_2$ (Figure \ref{subfig:exogenous-drift}), we can notice the region in which $X_4$ is being sampled, moving, as the node $X_2$, which is being drifted, is one of its ancestors. The class distributions do not change in this type of drift.

In contrast, \emph{endogenous drift}, in Figure \ref{subfig:endogenous-drift}, leads to clear changes in the class distribution between concepts, reflecting modifications in the underlying mechanisms that map features to the target. 
As a result, decision rules learned under previous concepts generally become invalid and must be adapted.

To induce \emph{confounder drift}, we observe both the emergence of new regions in the feature space, as observed on \emph{exogenous drift}, as well as changes in the class distribution across concepts, as observed in Figure \ref{subfig:confounder-drift}. 
The impact on the class distribution becomes more pronounced under \emph{target drift} (Figure~\ref{subfig:target-drift}), where changes in the target-generating mechanism directly alter the class boundaries, making differences between concepts apparent. To simulate these, we randomly reposition the prototypes that map the causal parents of $y$.

Finally, under \emph{structural drift} in Figure \ref{subfig:structural-drift}, it is observed that class distribution also changes.
The impact of \emph{structural drift} is closely related to the size of the causal graph. 
These plots use a \ac{DAG} with six nodes, excluding the target node; consequently, even small structural modifications can produce large effects throughout the causal chain, leading to significant distributional changes.

\begin{figure*}
    \centering

    \subfloat[Exogenous drift.]{\label{subfig:exogenous-drift}
        \includegraphics[width=0.9\linewidth]{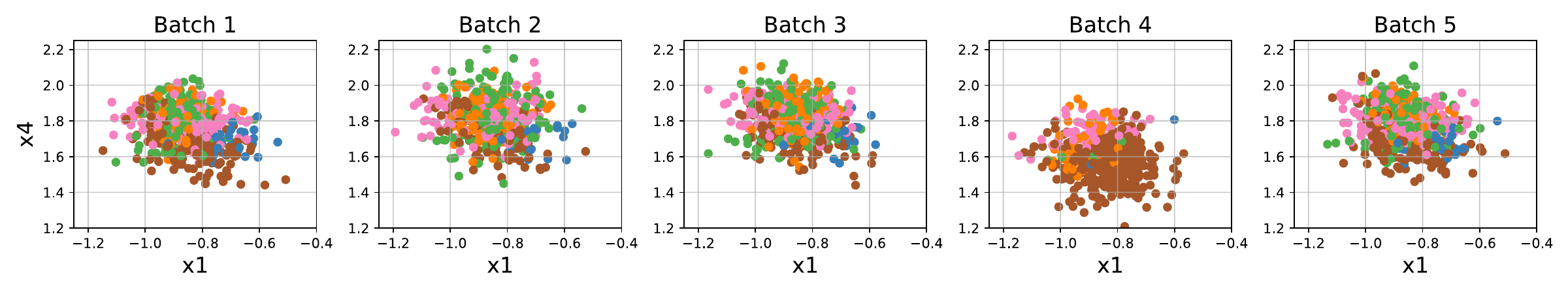}
    }
    
    \subfloat[Endogenous drift.]{\label{subfig:endogenous-drift}
        \includegraphics[width=0.9\linewidth]{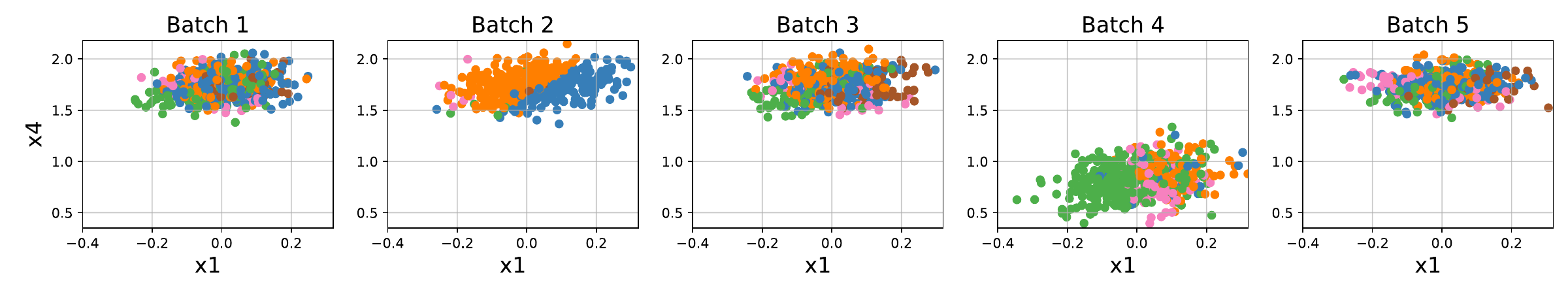}
    }

    \subfloat[Confounder drift.]{\label{subfig:confounder-drift}
        \includegraphics[width=0.9\linewidth]{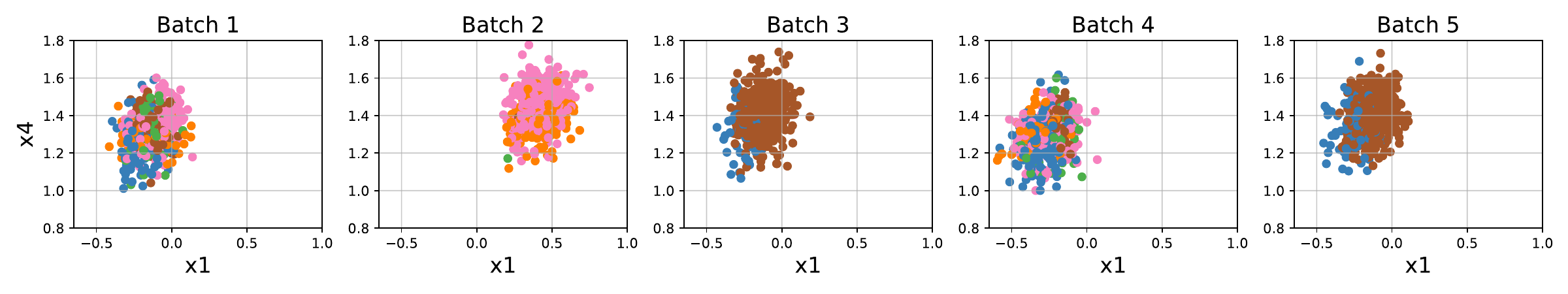}
    }

    \subfloat[Target drift.]{\label{subfig:target-drift}
        \includegraphics[width=0.9\linewidth]{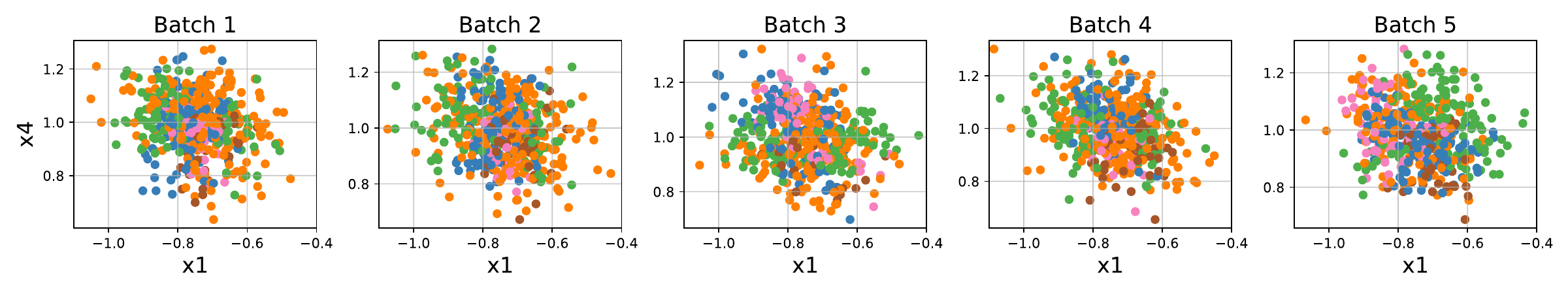}
    }

    \subfloat[Structural drift.]{\label{subfig:structural-drift}
        \includegraphics[width=0.9\linewidth]{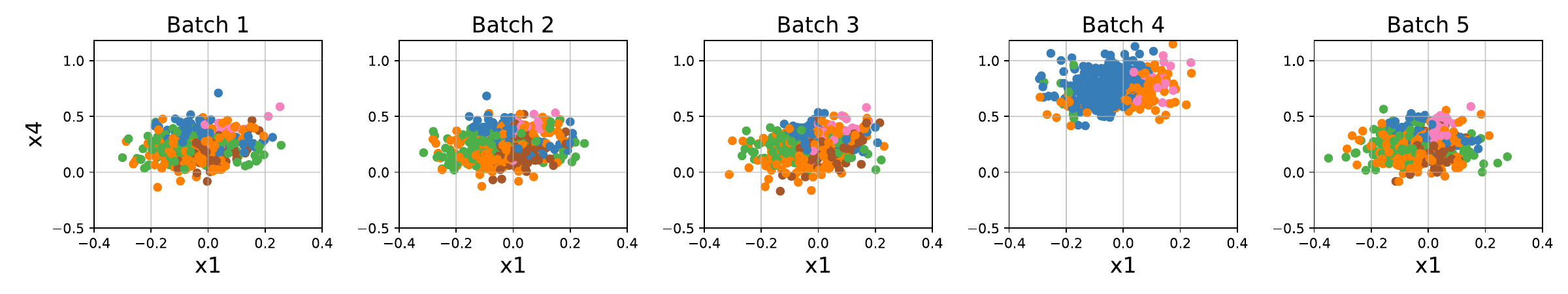}
    }
    
    \caption{Class distribution across concepts under causal concept drift events. Each color represents a different class.}
    \label{fig:drift-class-dist}
\end{figure*}

\section{Impact of incremental endogenous drift on specific nodes} \label{app:incremental-endogenous-analysis}

In Figure~\ref{fig:mmd-over-time-incremental-nodes}, we show the plots when we induce incremental updates on endogenous nodes of \ac{DAG} \#2 over time.
The results are averaged over 10 datasets, with randomly initialized weights for the linear mappers. 

We observe that the observable magnitude of drift depends not only on the structural location of the intervened node but also on the distributional regime of its ancestors. 
In particular, upstream variables with small means and small linear coefficients tend to attenuate perturbations as they propagate through successive structural equations. 
As a consequence, incremental changes introduced at early nodes in the graph may result in mild observable divergence at the level of $P(\mathbf{x})$ and $P(y \mid \mathbf{x})$. 

Interventions on all nodes produce measurable marginal \ac{MMD} and conditional KL divergence. 
However, the magnitude and qualitative impact on predictive performance vary with the causal position of the drifted node.

Interestingly, long-term incremental drift on nodes leads to increases in predictive accuracy, occasionally approaching near-perfect classification when intervening on deeper nodes. 
This behavior arises when incremental parameter updates push the target's parents into regions of the feature space dominated by a single-class prototype, thereby reducing class overlap. 
Because the drift is applied smoothly at each time step, the classifier adapts to a progressively simplified decision surface before more substantial overlap is reintroduced. 
Thus, one should be careful when applying interventions to structural mechanisms in an \ac{SCM}.

These results highlight that the effect of incremental endogenous drift is mediated by both causal distance to the target and by the interaction between structural coefficients and class-prototype geometry. 
Consequently, drift impact cannot be characterized solely by topological position in the graph. 
It also depends on how structural perturbations reshape the effective class-conditional distributions.

\begin{figure}[htb]
    \centering
    \subfloat[Incremental Endogenous drift -- Node $X_4$.]{\label{subfig:incremental-endogenous-x4-mmd-time}
        \includegraphics[width=0.3\linewidth]{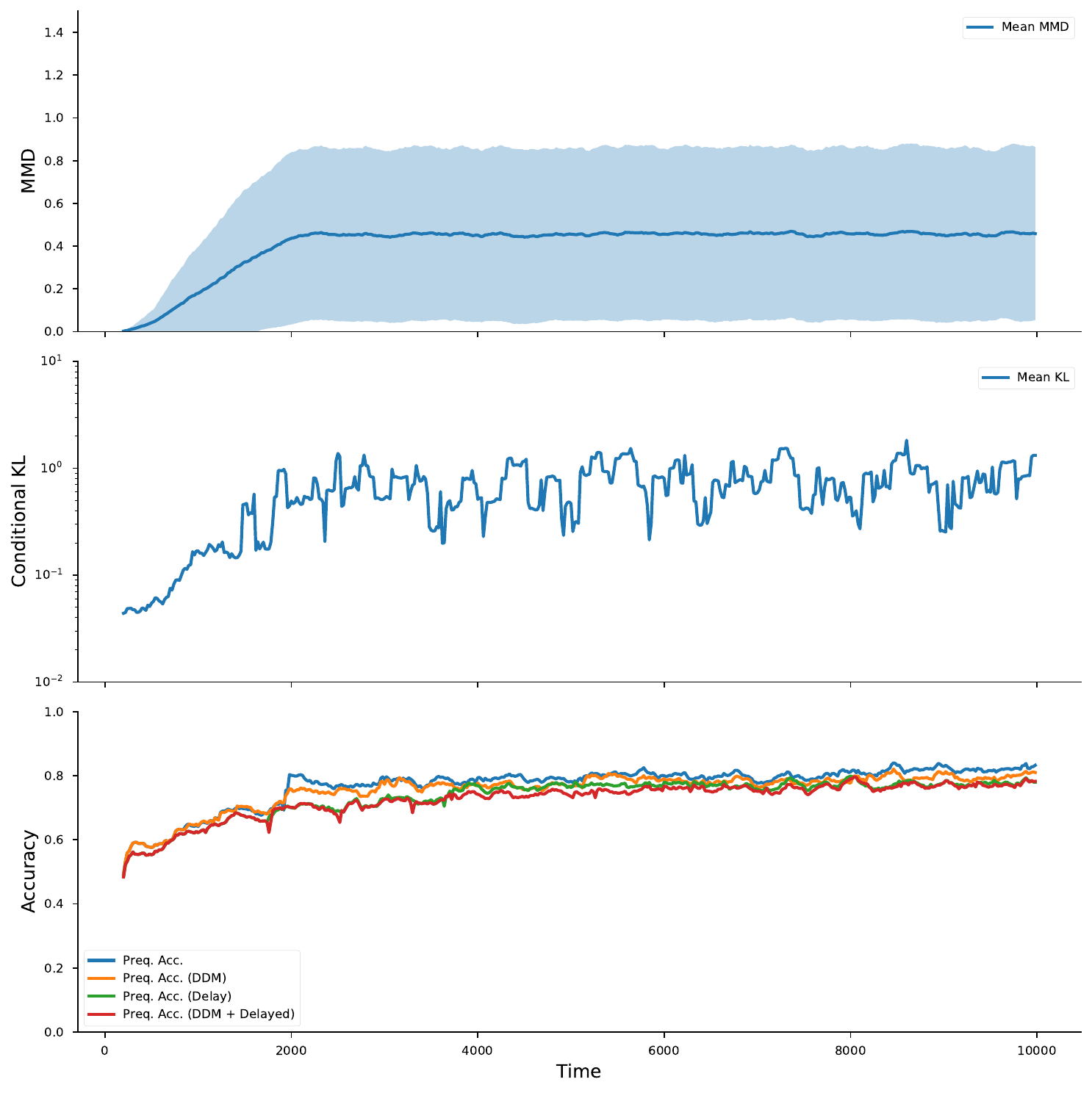}
    } \hspace{3mm}
    \subfloat[Incremental Endogenous drift -- Node $X_5$.]{\label{subfig:incremental-endogenous-x5-mmd-time}
        \includegraphics[width=0.3\linewidth]{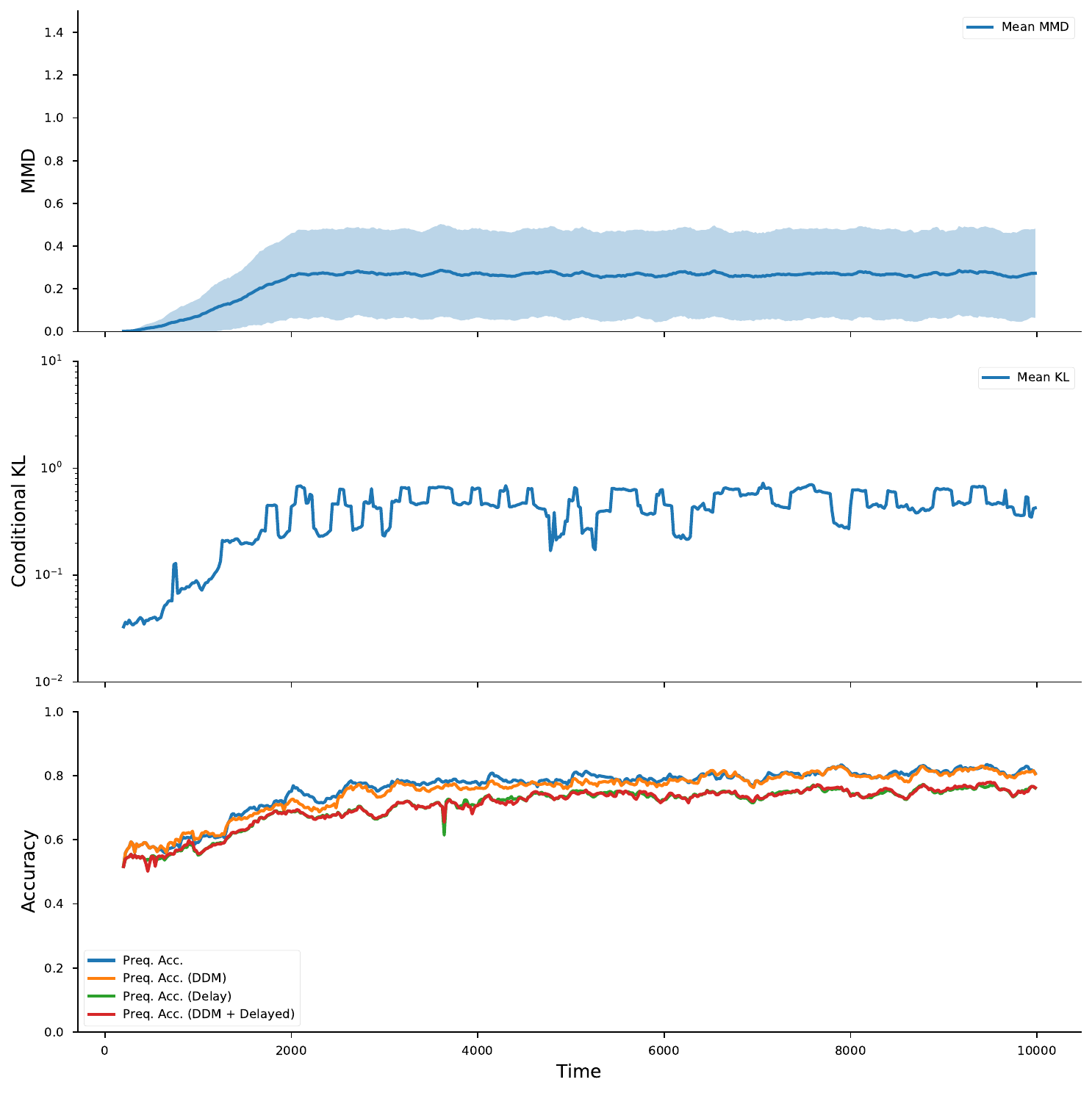}
    } \hspace{3mm}
    \subfloat[Incremental Endogenous drift -- Node $X_6$.]{\label{subfig:incremental-endogenous-x6-mmd-time}
        \includegraphics[width=0.3\linewidth]{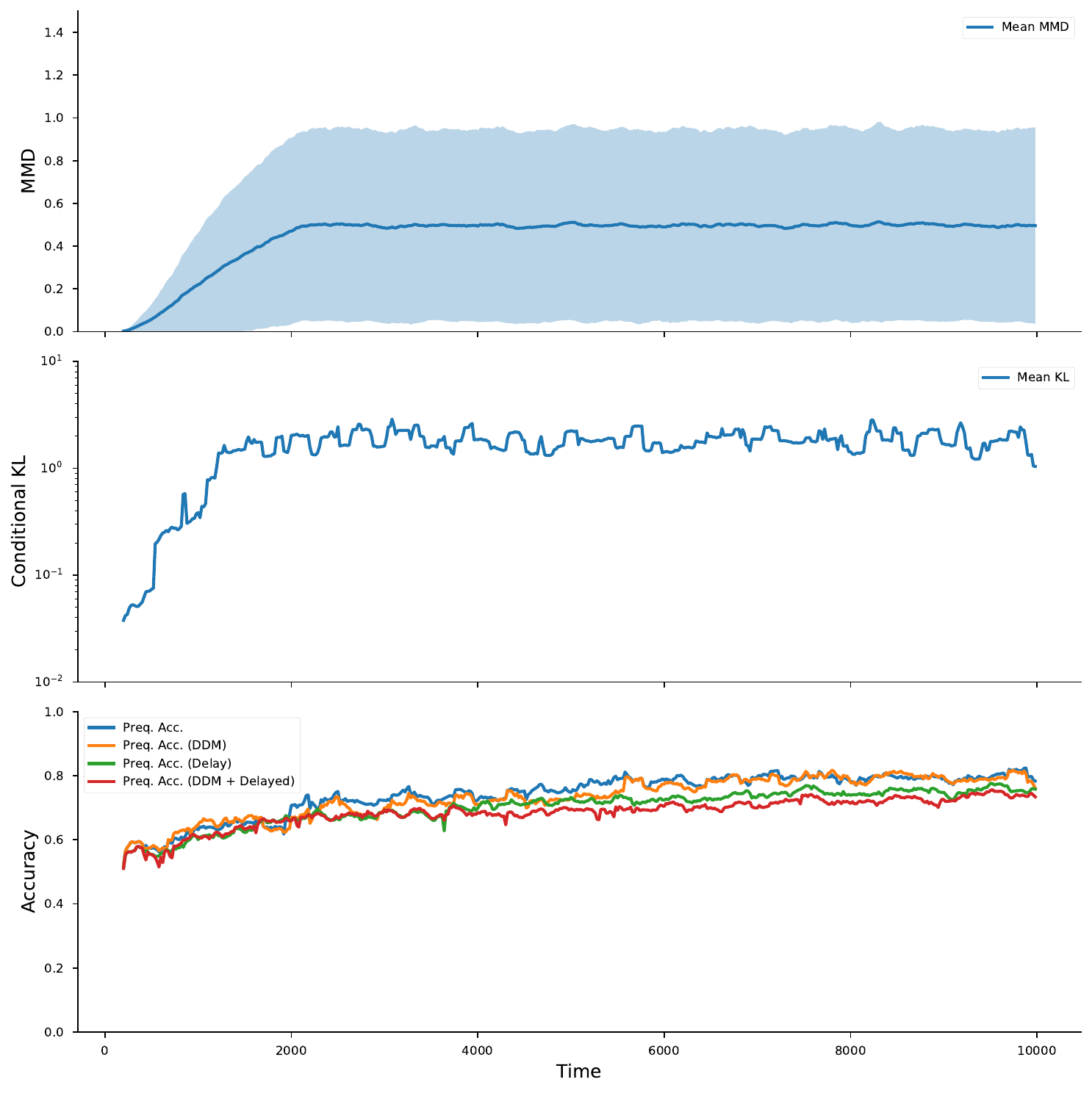}
    }

    \subfloat[Incremental Endogenous drift -- Node $X_7$.]{\label{subfig:incremental-endogenous-x7-mmd-time}
        \includegraphics[width=0.3\linewidth]{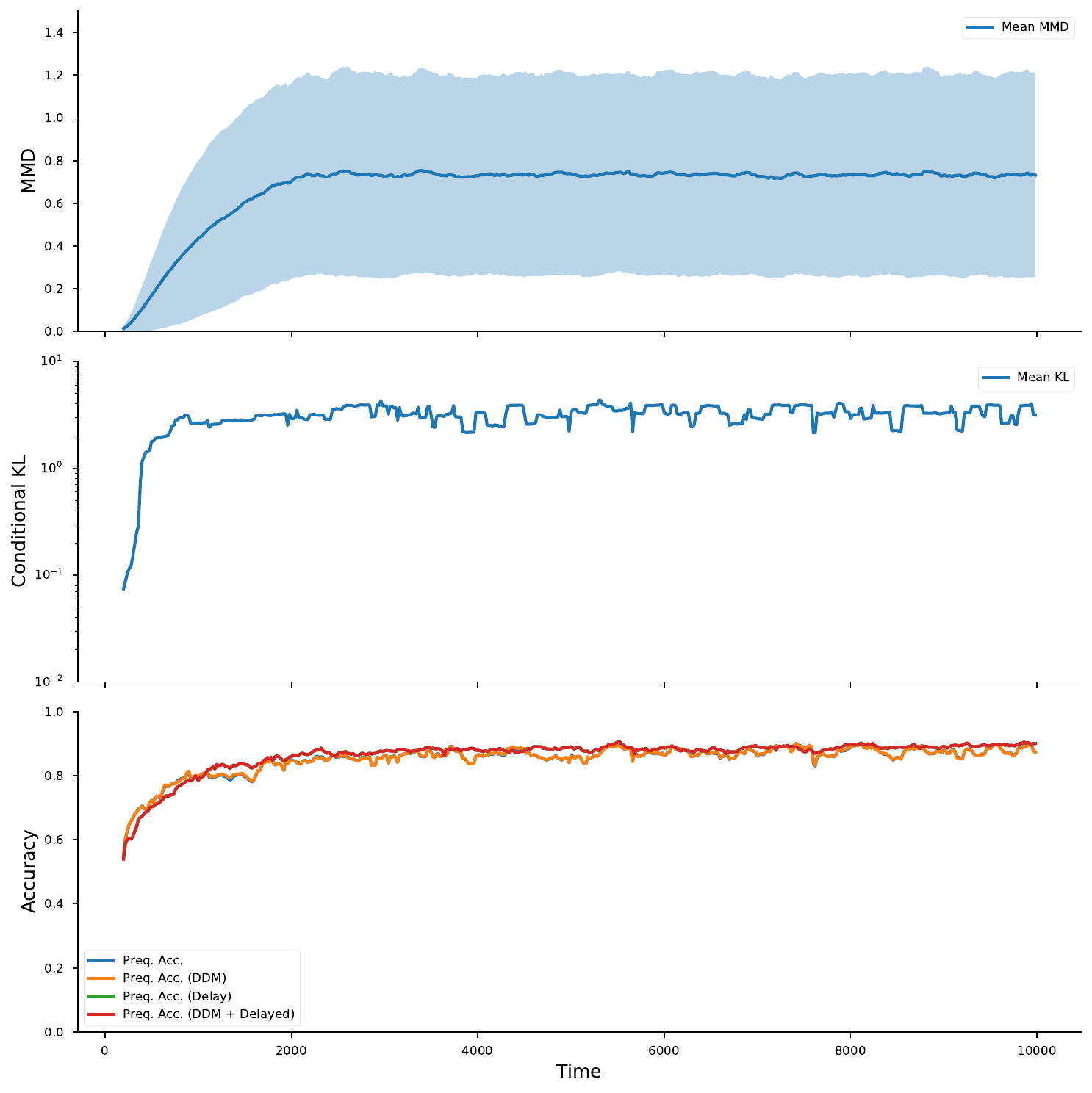}
    } \hspace{3mm}
    \subfloat[Incremental Endogenous drift -- Node $X_8$.]{\label{subfig:incremental-endogenous-x8-mmd-time}
        \includegraphics[width=0.3\linewidth]{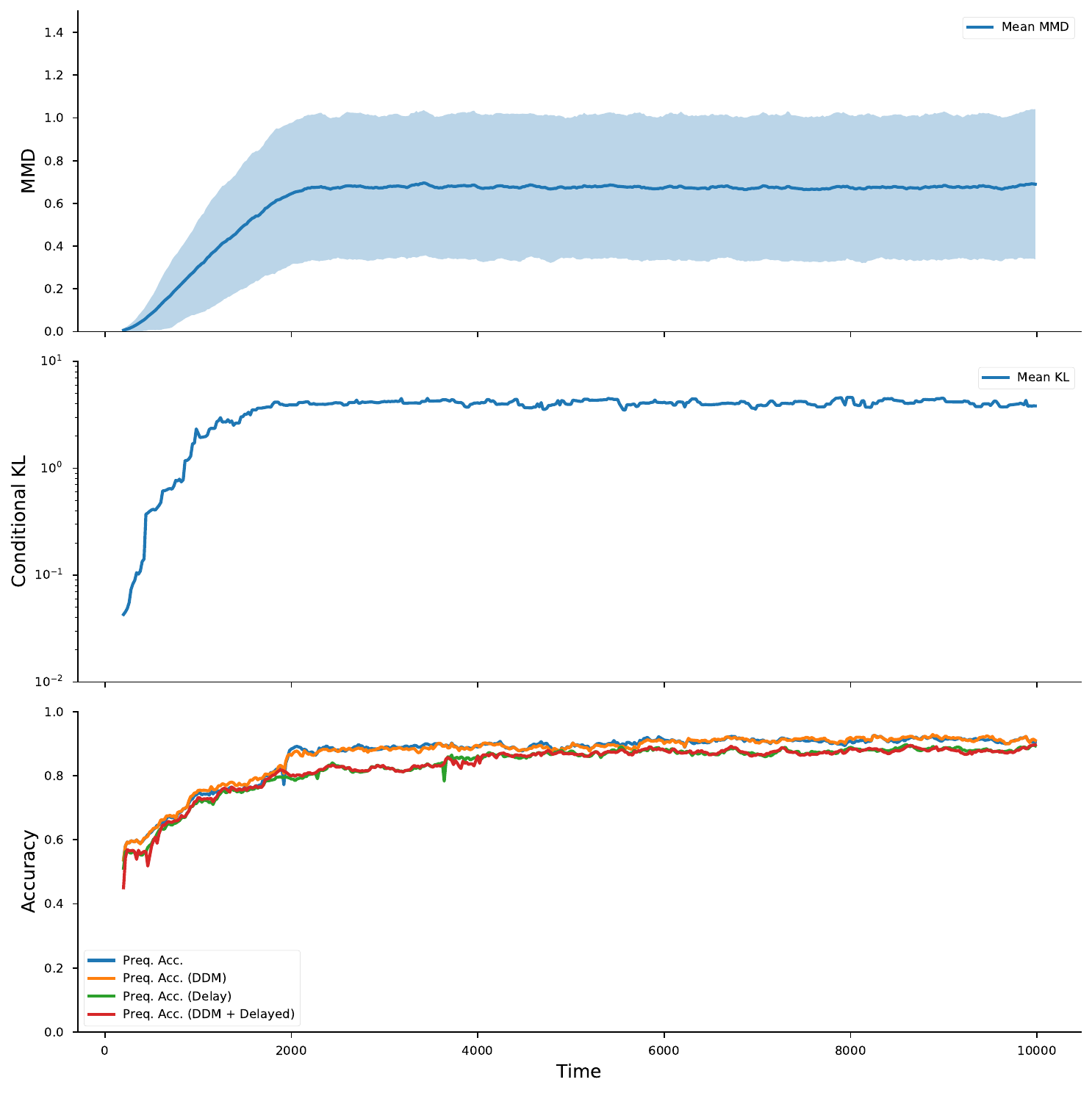}
    }
    
    \caption{MMD, KL divergence, and prequential accuracy over time across concepts of incremental causal drift events on specific nodes (DAG \#2).}
    \label{fig:mmd-over-time-incremental-nodes}
\end{figure}

\section{ACF Plots for datasets generated by DAGs \#1 and \#3}

The ACF plots for \ac{DAG} \#1 are shown in Figures~\ref{fig:acf-dag1} (without seasonality) and \ref{fig:acf-seasonality-dag1} (with seasonality).
For \ac{DAG} \#3, the ACF plots are in Figures~\ref{fig:acf-dag3} (without seasonality) and \ref{fig:acf-seasonality-dag3} (with seasonality).
The observations are the same as in the main text.

\begin{figure}
    \centering
    \includegraphics[width=0.3\linewidth]{figures/legend_rho.pdf}
    \includegraphics[width=.8\linewidth]{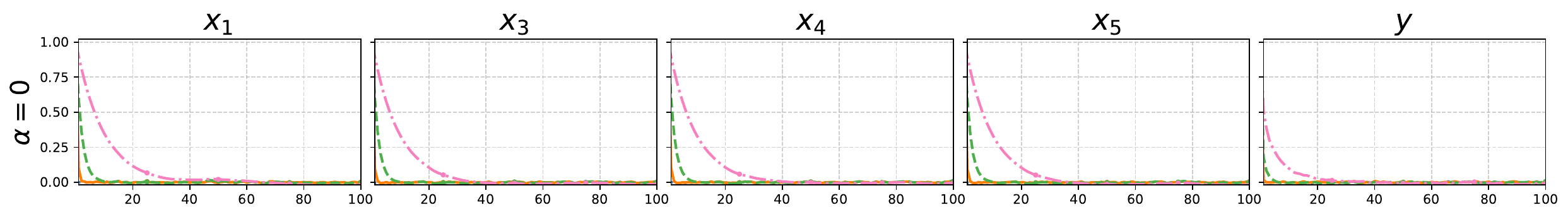}
    \includegraphics[width=.8\linewidth]{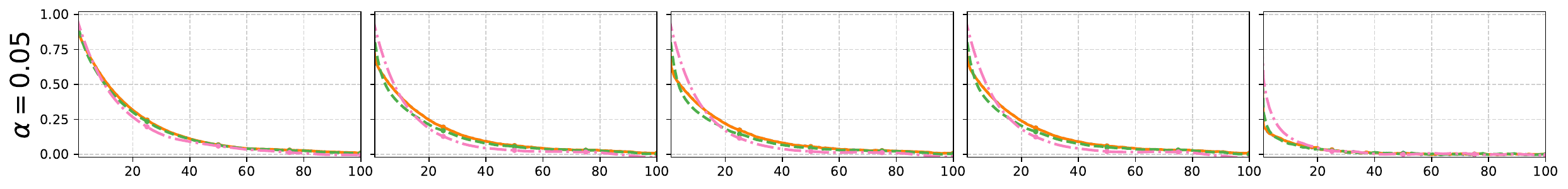}
    \includegraphics[width=.8\linewidth]{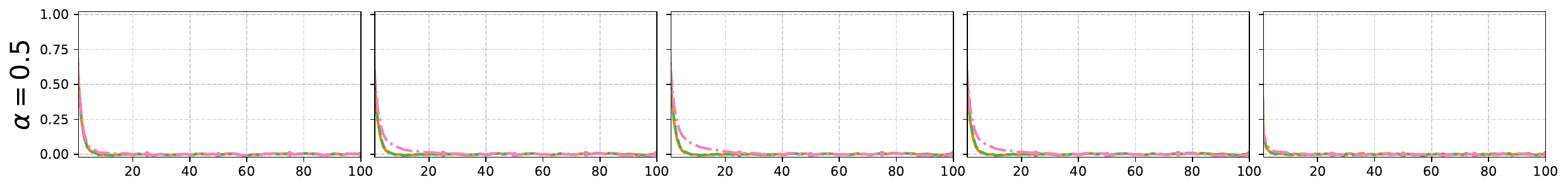}
    \includegraphics[width=.8\linewidth]{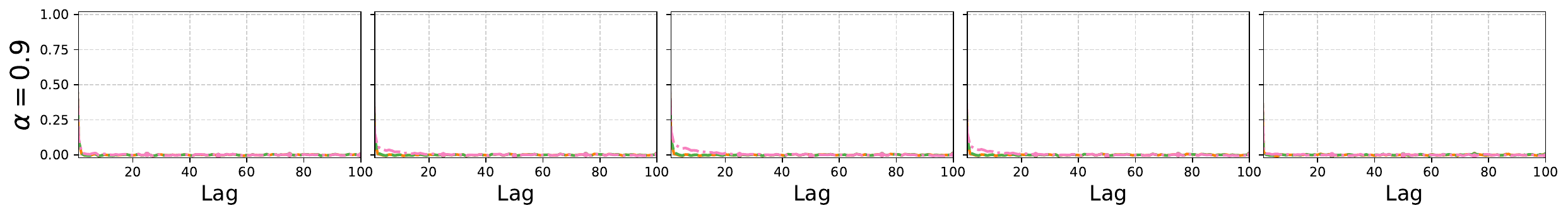}
    
    \caption{ACF plot on features with seasonality in DAG \#1 without seasonality.}
    \label{fig:acf-dag1}
\end{figure}

\begin{figure}
    \centering
    \includegraphics[width=0.3\linewidth]{figures/legend_rho.pdf}
    \includegraphics[width=.8\linewidth]{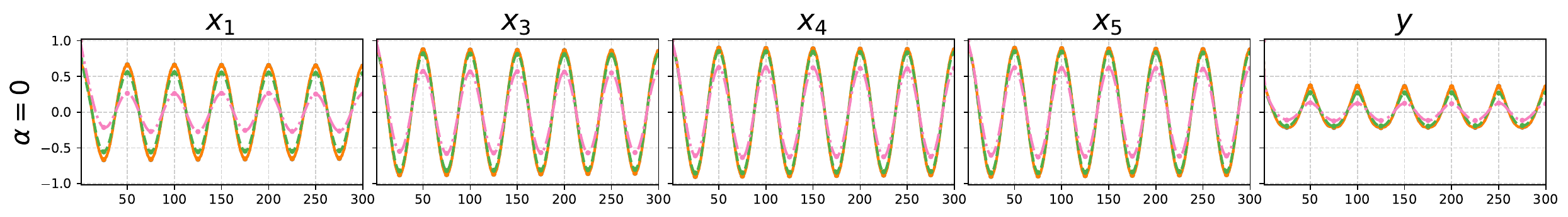}
    \includegraphics[width=.8\linewidth]{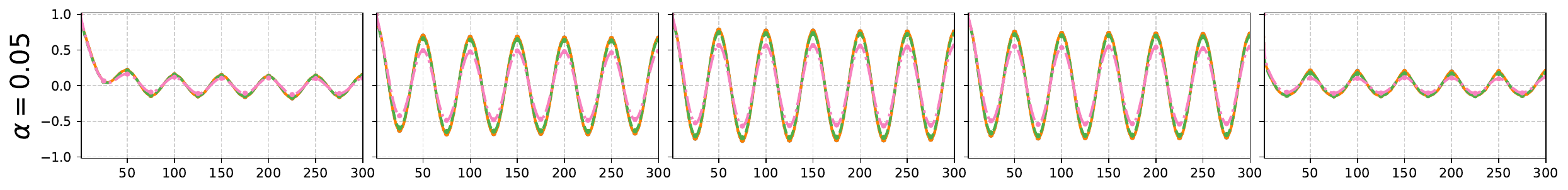}
    \includegraphics[width=.8\linewidth]{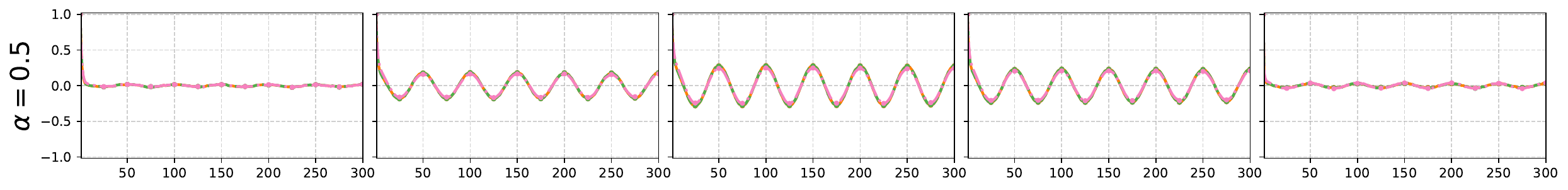}
    \includegraphics[width=.8\linewidth]{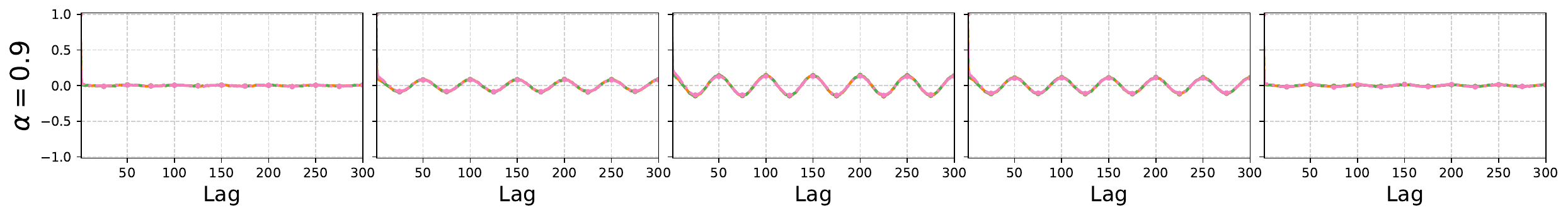}
    
    \caption{ACF plot on features with seasonality in DAG \#1 with seasonality.}
    \label{fig:acf-seasonality-dag1}
\end{figure}

\begin{figure}
    \centering
    \includegraphics[width=0.3\linewidth]{figures/legend_rho.pdf}
    \includegraphics[width=.8\linewidth]{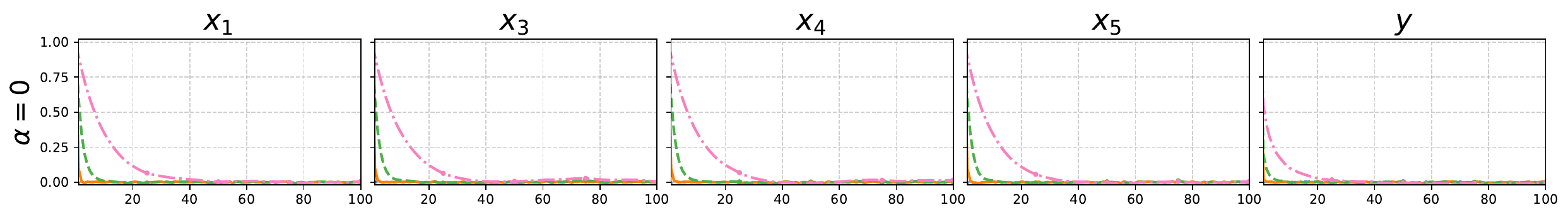}
    \includegraphics[width=.8\linewidth]{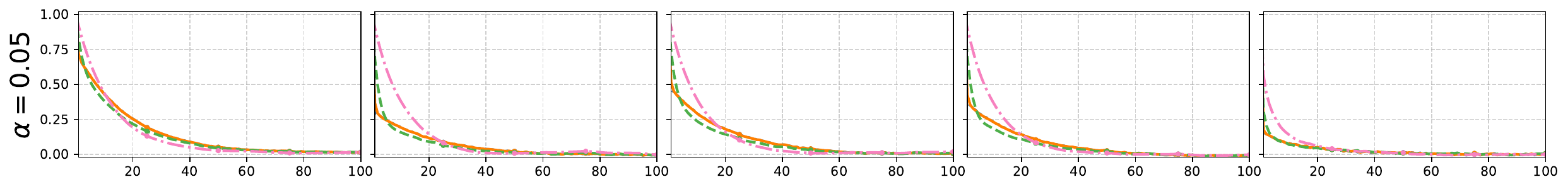}
    \includegraphics[width=.8\linewidth]{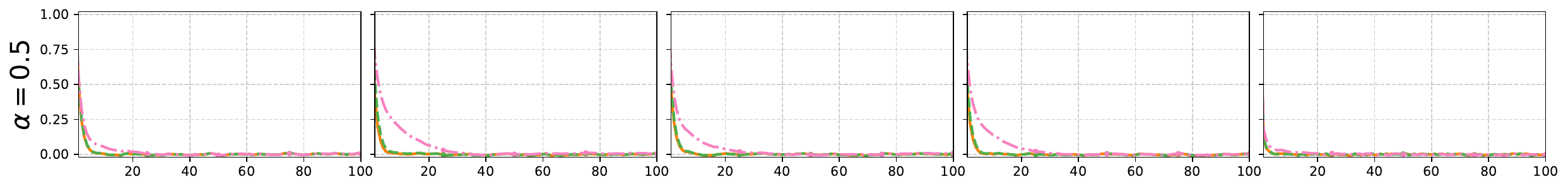}
    \includegraphics[width=.8\linewidth]{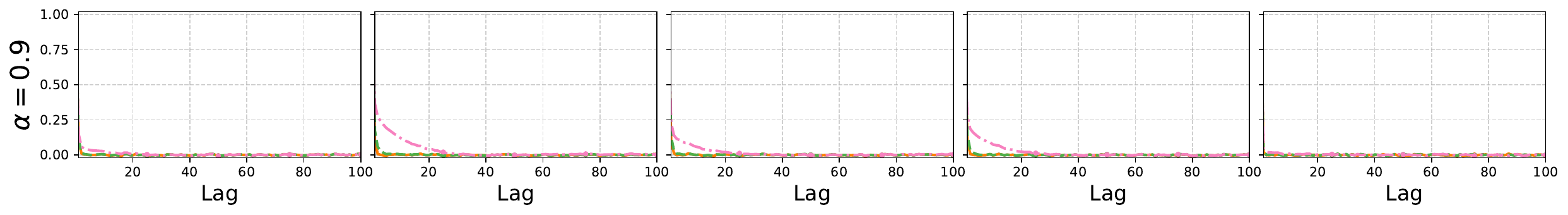}
    
    \caption{ACF plot on features with seasonality in DAG \#3 without seasonality.}
    \label{fig:acf-dag3}
\end{figure}

\begin{figure}
    \centering
    \includegraphics[width=0.3\linewidth]{figures/legend_rho.pdf}
    \includegraphics[width=.8\linewidth]{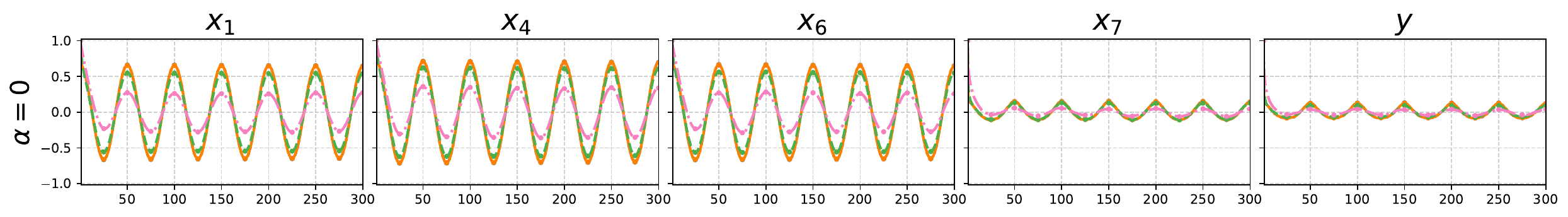}
    \includegraphics[width=.8\linewidth]{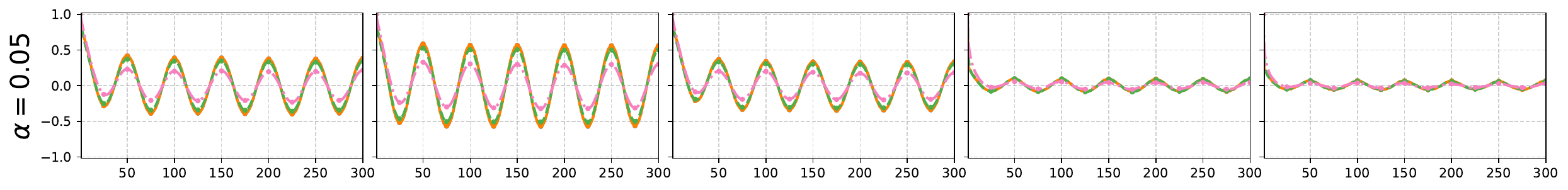}
    \includegraphics[width=.8\linewidth]{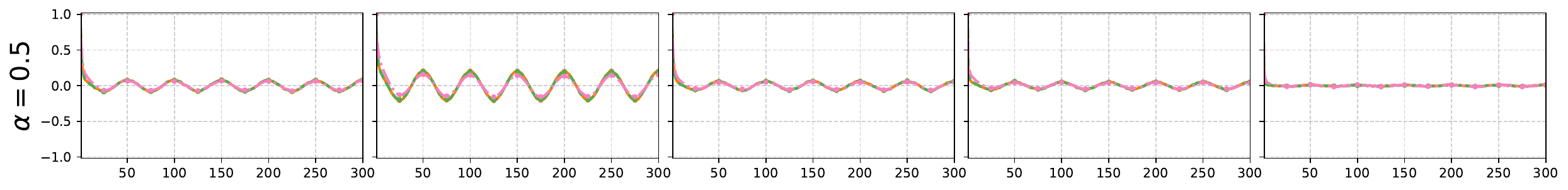}
    \includegraphics[width=.8\linewidth]{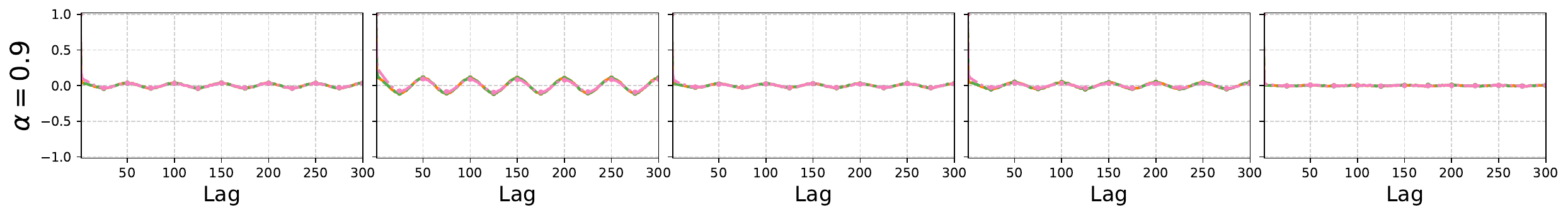}
    
    \caption{ACF plot on features with seasonality in DAG \#3 with seasonality.}
    \label{fig:acf-seasonality-dag3}
\end{figure}

\section{Ljung-box test}
\label{app:lb-test}

We report the Ljung–Box (LB) test results for data generated by \ac{DAG} \#2 in Table~\ref{tab:ljung-box-synth}.
When no component inducing temporal dependence is active, we fail to reject the null hypothesis of no autocorrelation across all features.
In contrast, when AR, EWMA, or seasonal components are included, the test consistently rejects the null hypothesis ($p<0.001$), indicating the presence of temporal dependence in the generated samples.

Because the temporal components are introduced through the same generation mechanisms across all \ac{DAG} configurations, the qualitative effect of inducing serial correlation is consistent for \acp{DAG} \#1 and \#3. 
Therefore, for brevity, we report the LB analysis only for \ac{DAG} \#2.

\begin{table}[htb]
    \centering
    \caption{Ljung--Box test (lag $h = 10$) p-values for synthetic data generated by \ac{CaDrift} using \ac{DAG} \#2 under different temporal dependence configurations.}
    \small
    \begin{tabular}{l c c c c c}
    \toprule
        & None (i.i.d.) & EWMA ($\alpha=0.05$) & AR ($\rho=0.6$) & EWMA+AR & Seasonality \\
    \midrule
        $X_1$  & 0.83 & \textbf{$<0.001$} & \textbf{$<0.001$} & \textbf{$<0.001$} & \textbf{$<0.001$} \\
        $X_2$  & 0.54 & \textbf{$<0.001$} & \textbf{$<0.001$} & \textbf{$<0.001$} & \textbf{$<0.001$} \\
        $X_3$  & 0.56 & \textbf{$<0.001$} & \textbf{$<0.001$} & \textbf{$<0.001$} & \textbf{$<0.001$} \\
        $X_4$  & 0.45 & \textbf{$<0.001$} & \textbf{$<0.001$} & \textbf{$<0.001$} & \textbf{$<0.001$} \\
        $X_5$  & 0.56 & \textbf{$<0.001$} & \textbf{$<0.001$} & \textbf{$<0.001$} & \textbf{$<0.001$} \\
        $X_6$  & 0.83 & \textbf{$<0.001$} & \textbf{$<0.001$} & \textbf{$<0.001$} & \textbf{$<0.001$} \\
        $X_7$  & 0.79 & \textbf{$<0.001$} & \textbf{$<0.001$} & \textbf{$<0.001$} & \textbf{$<0.001$} \\
        $X_8$  & 0.85 & \textbf{$<0.001$} & \textbf{$<0.001$} & \textbf{$<0.001$} & \textbf{$<0.001$} \\
        $X_9$  & 0.15 & \textbf{$<0.001$} & \textbf{$<0.001$} & \textbf{$<0.001$} & \textbf{$<0.001$} \\
        $X_{10}$ & 0.85 & \textbf{$<0.001$} & \textbf{$<0.001$} & \textbf{$<0.001$} & \textbf{$<0.001$} \\
        $y$    & 0.63 & \textbf{$<0.001$} & \textbf{$<0.001$} & \textbf{$<0.001$} & \textbf{$<0.001$} \\
    \bottomrule
    \end{tabular}
    \label{tab:ljung-box-synth}
\end{table}

\section{Additional experimental details: causal discovery on the ELEC2 dataset}
\label{app:detail-exp-elec2}

The experiments in Section 6.3 follow a test-then-train protocol with immediate label availability \citep{gama2014survey}, and one with a 100-sample delay in label availability.
The results are averaged over 5 runs. 
This appendix provides additional detail regarding root mappers, ACF plots, and learners' hyperparameters.
Before causal discovery, we set the constraint that the nodes $day$ and $period$ are root nodes.

The $day$ feature is a categorical variable taking values in $[1,7]$, where each value corresponds to a day of the week. 
Since each sample in the ELEC2 dataset is collected every 30 minutes, one day corresponds to 48 consecutive samples. To preserve this temporal structure in the synthetic stream, we define the root node as a deterministic periodic mapping from time:

\[
X_{day}^{(t)} := \left( \left\lfloor \frac{t}{48} \right\rfloor \bmod 7 \right) + 1.
\]

This mapping generates a weekly categorical cycle in which each day label remains constant for 48 samples (i.e., 24 hours), after which it advances to the next day and repeats after seven days. 
The $period$ feature represents a day period, ranging in $[1, 48]$, and has been normalized to $[0, 1]$. To simulate this feature, we use the following equation:

\[
    X_{period}^{(t)} := \frac{t\bmod{p}}{p-1},
\]

\noindent where $p$ is set to 48, referring to the period in which the feature changes.

These are the only two features for which we needed to explicitly define custom mappers to faithfully reproduce the temporal dynamics at the root nodes of the original dataset. 
Inner nodes and the target variable are mapped through learned small neural networks, fitted on the first 50\% of the data, following the cause--effect relationships present in Figure \ref{fig:elec-dag-pc}.

Among the original features, the FFT identified a clear seasonal period only for the $period$ node, with a period $T_{period}$ of 48.
Since no consistent seasonality was detected in the remaining nodes, no explicit seasonal components were added to their corresponding learned mappers. 
The only exception is the autoregressive noise term, for which the parameter $\rho$ was set to $0.6$ in order to induce moderate temporal dependence across consecutive samples.
With this setup, we can synthesize the data samples. 

For prequential accuracy, we use a sliding window of 500 samples under immediate feedback, providing a smoother visualization of the long-term evolution of predictive performance over the stream.
Under delayed feedback, the effects of augmentation are concentrated in the early stages following adaptation. 
Therefore, we use a smaller sliding window of 100 samples to better capture short-term performance changes that would otherwise be smoothed out by larger windows.

\subsection{Additional experiments: Synthesized vs. original ELEC2}
\label{app:acf-elec2-synth-real}

In Figure \ref{fig:acf-real-synth-elec}, we show the \ac{ACF} plots for all features of the datasets generated by \ac{CaDrift} and compare them to the real-world datasets. 
We notice that the seasonality on the features $day$ and $period$ is very similar.
However, many features exhibit seasonality in the real-world dataset that the FFT did not detect. 
Nevertheless, we can still generate synthetic time-dependent samples that follow the learned cause--effect relationships from the source data.

\begin{figure}
    \centering

    \includegraphics[width=0.2\linewidth]{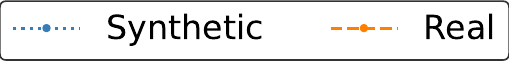}
    
    \includegraphics[width=0.9\linewidth]{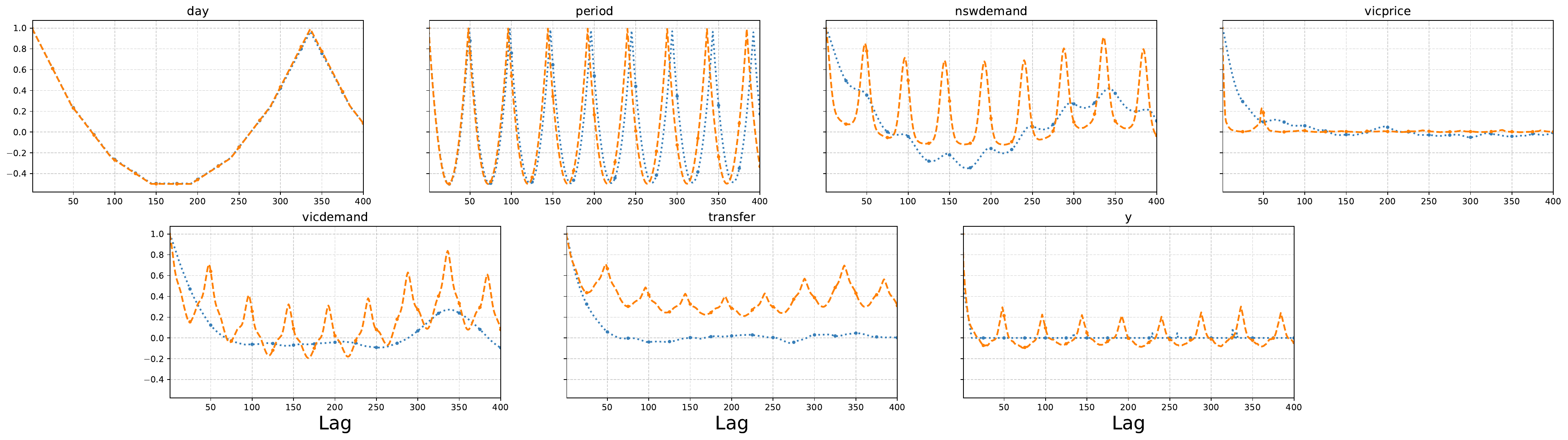}
    \caption{Synthetic vs real-world ACF plots for ELEC2. The y-axis refers to the autocorrelation, and the x-axis to the lag.}
    \label{fig:acf-real-synth-elec}
\end{figure}

\subsection{MMD: synthesized vs. real-world ELEC2}

In Figure~\ref{fig:mmd-real-generated-elec2}, we present the evolution of the MMD over time for the real ELEC2 stream and the stream synthesized by CaDrift.
We use two complementary MMD plots to assess the fidelity of the synthesized stream in terms of marginal distribution.
First, the ``Real vs Real'' curve corresponds to the MMD computed between windows from the original ELEC2 dataset and a reference window containing the first 200 samples, capturing the natural distributional evolution of the real stream over time.
Second, the ``Real vs Generated'' curve corresponds to the MMD computed between aligned windows from the real and synthesized streams at the same temporal positions.
This design allows us to evaluate not only the similarity between real and generated samples, but also whether the synthesized stream reproduces the temporal distributional dynamics of the original data.

We observe that the synthesized stream follows the same seasonal MMD patterns present in the original ELEC2 dataset, with peaks occurring at similar temporal regions.
Moreover, the divergence between aligned real and synthesized windows remains relatively stable over time, suggesting that CaDrift reproduces key temporal and distributional characteristics of the original stream, although the seasonal dynamics are not perfectly matched.

\begin{figure}[htb]
    \centering
    \includegraphics[width=0.5\linewidth]{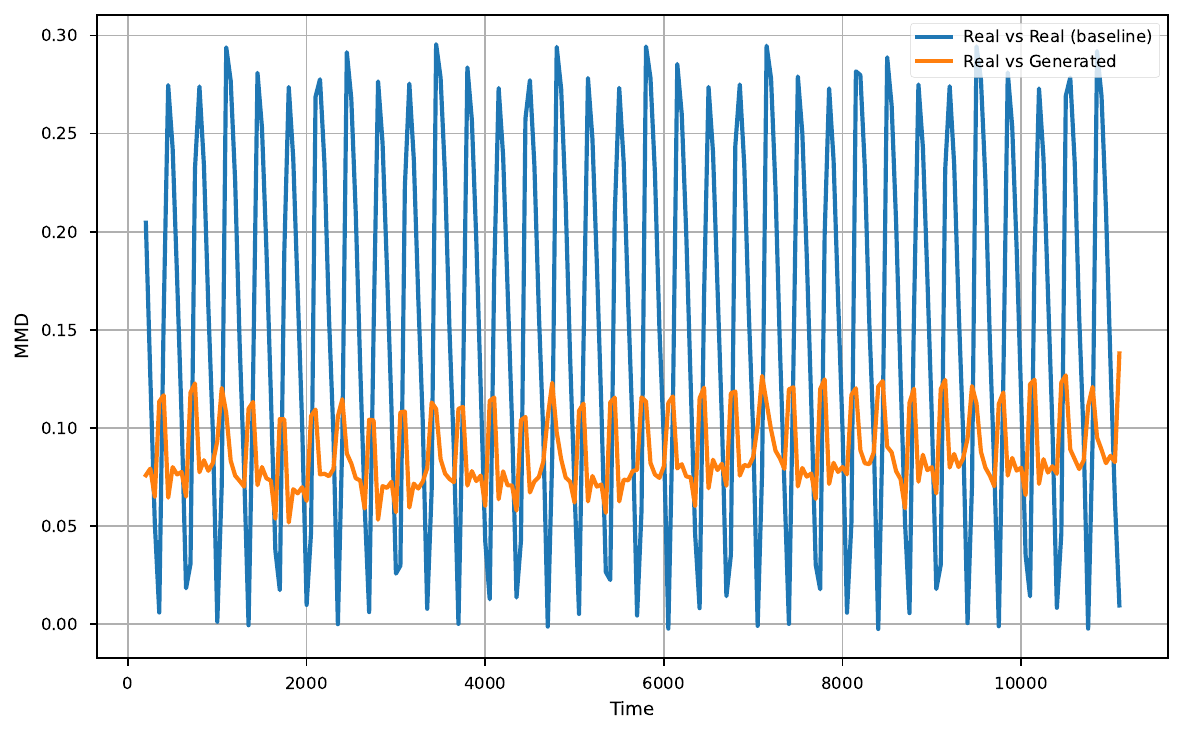}
    \caption{MMD over time: real vs. generated ELEC2.}
    \label{fig:mmd-real-generated-elec2}
\end{figure}

\subsection{Online Learners Hyperparameters}
\label{app:learners-hyperparameters}

Table~\ref{tab:learners-hyperparameters} reports the hyperparameters of the online learners used in the experiments in Section~\ref{subsec:case-study-elec2}.
For TabPFN, we use a $context\_window$ of 10,000 samples.
This implies that, when $n=5,000$, 50\% of the context window consists of synthetic data generated by \ac{CaDrift}.
When $n=10,000$, the whole initial context window consists of synthetic data.

\begin{table}[htb]
    \centering
    \caption{Learners' hyperparameters. All ensemble methods use an HT as a base classifier.}
    \begin{tabular}{l p{10cm}}
    \toprule
        Method & Hyperparameters \\
        \midrule
        ARF & $change\_detector:ADWIN$, $ensemble\_size=100$ \\ \hline
        IncA-DES & $change\_detector:RDDM$, $pool\_size=100$, $F=200$, $k=5$ \\ \hline
        TabPFN$_{v2.5}$$^{\text{Stream}}$ & $context\_window=10,000$, $short\_term\_window=7,500$, $long\_term\_window=2,500$ \\
        \bottomrule
    \end{tabular}
    \label{tab:learners-hyperparameters}
\end{table}

\end{document}